\documentclass{article}

\usepackage[final]{neurips_2022}

\usepackage[utf8]{inputenc} %
\usepackage[T1]{fontenc}    %
\usepackage{hyperref}       %
\usepackage{url}            %
\usepackage{booktabs}       %
\usepackage{amsfonts}       %
\usepackage{nicefrac}       %
\usepackage{microtype}      %
\usepackage{xcolor}         %
\usepackage{amsthm, amsmath, amssymb}
\usepackage{mathtools}
\usepackage{subfig}
\usepackage{graphicx}
\PassOptionsToPackage{options}{natbib}
\usepackage{caption}
\newtheorem{theorem}{Theorem}[section]
\newtheorem{proposition}[theorem]{Proposition}

\theoremstyle{definition}
\newtheorem{definition}{Definition}[section]
\usepackage{wrapfig}
\usepackage[normalem]{ulem}

\newcommand{\f}{f^{\mathcal{I}}_{y-c}}
\newcommand{\an}{\alpha_c^{\mathcal{I}, \mathcal{A}}}
\newcommand{\A}{\mathcal{A}}

\title{How Sampling Impacts the Robustness of Stochastic Neural Networks}

\author{%
  Sina Däubener and Asja Fischer\\
  Department of Computer Science\\
  Ruhr University Bochum, Germany \\
  \texttt{\{sina.daeubener, asja.fischer\}@rub.de} \\
}

\begin{document}

\maketitle

\begin{abstract}
Stochastic neural networks (SNNs) are random functions 
whose predictions are gained by averaging over multiple realizations. 
Consequently, a gradient-based adversarial example is calculated based on one set of samples and its classification on another set. 
In this paper, we derive a sufficient condition for such a stochastic prediction 
to be robust against a given sample-based attack. 
This allows us to identify the factors that lead to an increased robustness of SNNs and gives theoretical explanations for: 
(i) the well known observation, that increasing the amount of samples drawn for the estimation of adversarial examples increases the attack's strength,
(ii) why increasing the number of samples during an attack can not fully reduce the effect of stochasticity, 
(iii) why the sample size during inference does not influence the robustness, and
(iv) why a higher gradient variance and a shorter expected value of the gradient relates to a higher robustness. 
Our theoretical findings give a unified view on the mechanisms underlying previously proposed approaches for increasing attack strengths or model robustness 
and are verified by an extensive empirical analysis.
\end{abstract}

\section{Introduction}
Since the discovery of adversarial examples~\citep{biggio_adv_att, intruding_Szegedy},
a significant amount 
of research was dedicated to hinder
attacks~\citep[e.g.][]{ madry2018towards,papernot_distillation, trades}, to enhance attack strategies~\citep[e.g.][]{ obfuscated_grad, survey_adv_vision, carlini_bypassing, obscurity_robustness} or to derive ways to certify model robustness~\citep[e.g.][]{cohen2019certified, lecuyer2019certified}. 
Robustness guarantees often
specify an $\epsilon$-ball around input points in which perturbations do not lead to a label change \citep{hein_formal_gurantees, Croce2020Provable}. 
The maximal possible radius of such  an $\epsilon$-ball corresponds to the distance of the input point to the nearest decision boundary, which on the other hand is equal to the length of the smallest perturbation vector that leads to a misclassification (c.f. figure~\ref{fig:decision_boundaries}a)).
Such a robustness analysis assumes 
that the decision boundaries are fixed and
that the attacker is able to estimate (at least approximately) this minimal perturbation vector, which is a reasonable assumption for deterministic networks but usually does not hold for stochastic neural networks (SNNs).

Stochastic neural networks, and stochastic classifiers more generally, are random functions and predictions are given by the expected value %
of the random function for the given input. In practice, this expectation is usually
not tractable and hence it is approximated by averaging over multiple realizations of the random function. This approximation leads to the challenging setting where predictions, decision boundaries, and gradients become random variables themselves. 
Hence, under an adversarial attack, the decision boundaries used for calculating the %
adversarial  example
and those used when predicting the label %
of
the resulting %
adversarial  example differ. 
This means that the attacker can not estimate the optimal perturbation direction 
i.e., the direction to the closest %
decision boundary %
of
the network that will be sampled during inference c.f. figure~\ref{fig:decision_boundaries} b), c).

In this paper, %
we study how robustness  of SNNs arises from this misalignment of the attack direction and the optimal perturbation direction during inference that results from the stochasticity inherent to stochastic classifiers. 

We make the following contributions:
First, we derive a %
sufficient condition for a SNN prediction relying on one set of samples to be robust against %
an attack that was calculated on a second set of samples. 
Second, we discuss how %
model properties and sample sizes impact this condition
which does not only allows us to answer the questions stated in the abstract but also to explain the success of recently proposed defense mechanism from a simple unifying geometric perspective.  
Lastly, we conduct an empirical analysis that demonstrates that the novel theoretical insights perfectly match what we observe in practice.

\begin{figure}[]
\centering
\subfloat[]{\includegraphics[width=0.29\columnwidth]{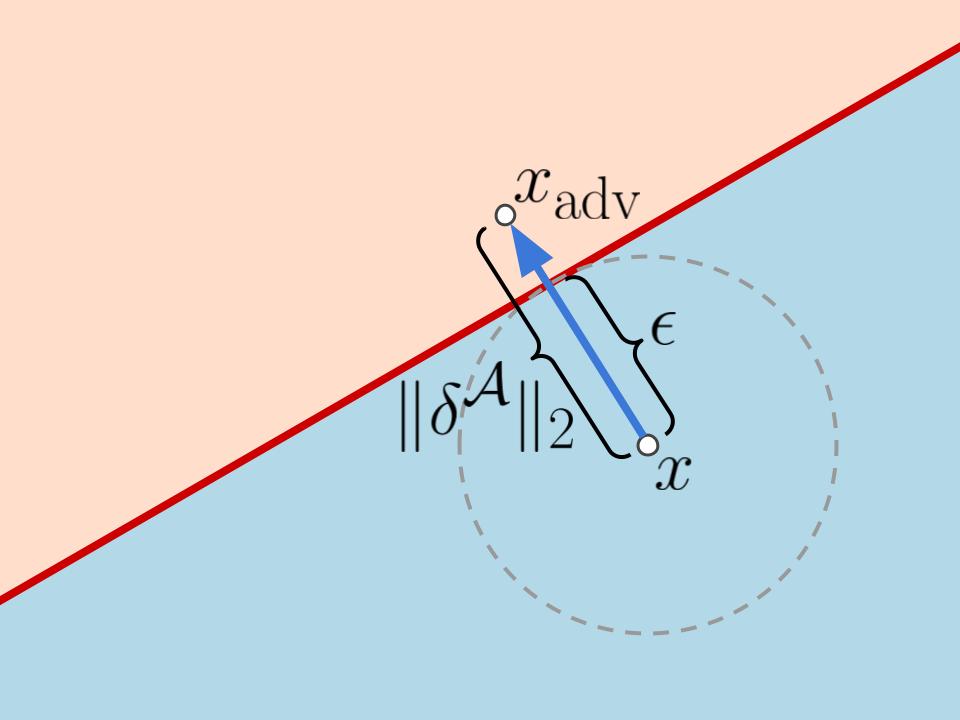} }%
\qquad
\subfloat[]{\includegraphics[width=0.29\columnwidth]{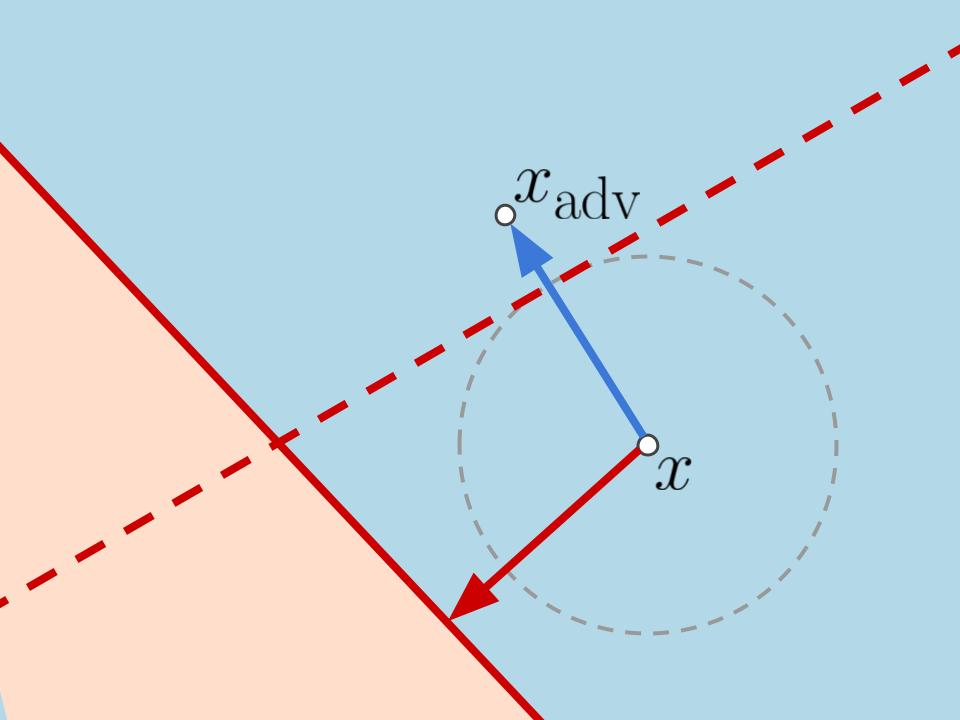} }%
\qquad
\subfloat[]{\includegraphics[width=0.29\columnwidth]{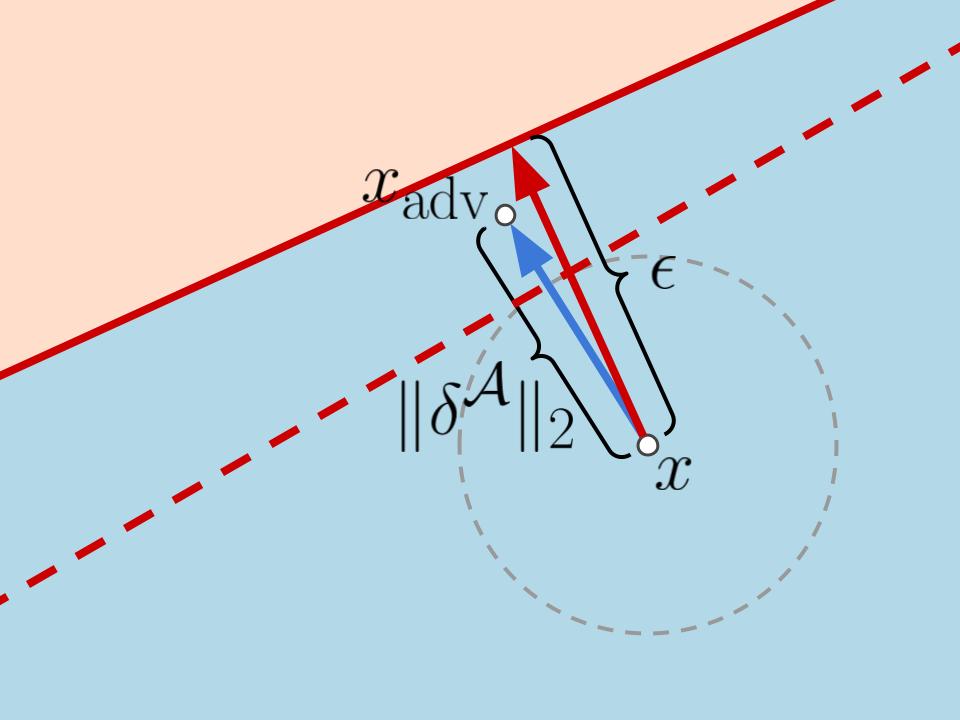} }%
\caption{
(Un-)successful attacks on a binary stochastic classifier with a  linear decision boundary.
a) an adversarial example
$x_{\text{adv}}=x+\delta^{\mathcal{A}}$ is created by shifting $x$ in the direction of 
the closest
decision boundary, indicated by the blue arrow. $\epsilon$ is the minimal needed attack length and %
the attack can
only be an adversarial example if $ \| \delta^{\mathcal{A}} \|_2 \geq \epsilon$ holds. 
b) and c) show 
the stochastic
decision boundaries during attack (dashed)  
and inference (solid). The red arrows indicate the shortest direction to the latter, respectively.
In b) $\delta^{\mathcal{A}}$  moves $x$ even further away from the decision boundary used during inference, while in c) the magnitude 
of $\delta^{\mathcal{A}}$ is too short to result in a successful attack.}
\label{fig:decision_boundaries}
\end{figure}
\section{Related work}
Several works proposed stochastic defense mechanism to increase adversarial robustness~\citep[e.g.][]{BART, xie2018mitigating}. \citet{obfuscated_grad} linked their success to gradient obfuscating during the attack and showed
that increasing the number of samples for approximating the gradient during attack  leads to a sever decrease in adversarial accuracy.

However, SNNs were still found 
to have an increased robustness even w.r.t.~stronger attacks~\citep{simpleSNN, He_2019_CVPR, weight_covariance-eustratiadis21a,learn2perturb}.
Their success 
was attributed to different effects of stochasticity, e.g.~model smoothing~\citep{liu_selfensemble, Addepalli_2021_CVPR} or
 diversification of the gradients~\citep{GradDiv, diverse_directions_bender20a}.
While a lot of work analyzed the robustness of deterministic neural networks \cite[e.g.][]{madry2018towards, croce_max_lin_regions, Croce2020Provable, cvpr_robust, yang2021ensemble_robust},
the robustness of %
SNNs
is less well understood. One line of research %
focused on the robustness of Bayesian neural networks (BNNs).
\citet{wicker2020probabilistic} certified robustness of BNNs using 
interval bound propagation techniques which they later also employed 
to derive guarantees for the robustness of BNNs 
with modified adversarial training \citep{wicker_adv_train_BNN}.
Moreover, \citet{carbone2020robustness} investigate robustness 
in the infinite-width infinite-sample limit.
Lastly, \citet{nips19_theory_robustness_randomization} derived theoretical robustness guarantees for randomized networks, %
where the randomization
is based on additive noise from an exponential family distribution.
This leads to a generalization of the robustness guarantees derived previously by~\citet{lecuyer2019certified} and~\citet{cohen2019certified}. Those certified robustness guarantees specify the radius of an $\ell_2$-ball in which the prediction does not change. In contrast, our results specify the robustness that results from the difficulty to identify the ideal attack direction and imply that even for perturbations outside this confidence ball, a gradient-based attack has a chance of not being successful.

To the best of our knowledge, no 
existing theoretical analysis of SNNs  explicitly discusses either the effect of stochasticity during inference nor the %
impact
of the sample size during attack and prediction on the robustness.

\section{Preliminaries}
\label{sec:preliminaries}
We first clarify the terminology before we state the main theoretical results of our paper.%
\paragraph{Stochastic classifiers}
\label{subsec:stoch_classifiers}
We use the term \textit{stochastic classifiers} for all classifiers which have an inherent stochasticity 
through the use of random variables in the model. Formally, we define them as follows:

\begin{definition}[Stochastic classifiers]
A stochastic classifier with $k$ classes corresponds to a function
$ f:\mathbb{R}^d\times \Omega^h  \rightarrow \mathbb{R}^k$ that maps a pair $(x, \Theta)$ to the output $f(x, \Theta)=( f_1(x,\Theta), \dots  f_k(x,\Theta))^T$, where $x \in \mathbb{R}^d$ is an input vector, $\Theta \in \Omega^h, \Theta \sim p(\Theta)$ is a random vector, and $f_c(x,\Theta)$ with $c \in \{1,\dots,k\}$ are the discriminant functions for each class. 
The prediction of a stochastic classifier for
an input $x$ %
is given by $%
\mathbb{E}_{\Theta}
[f(x,\Theta)]$ 
and the predicted class by
$
\arg \max_{c} %
\mathbb{E}_{\Theta}
[f_c(x,\Theta)] .
$
\end{definition}
This generic definition of stochastic classifiers %
covers linear models with random 
weights, but also more complicated methods like %
BNNs~\citep{neal}, infinite mixtures~\citep{daubener2020investigating}, Monte Carlo dropout networks \citep{mc_dropout_gal},
randomized smoothing as proposed by \citet{lecuyer2019certified}\footnote{Note, that,
in contrast to \citet{lecuyer2019certified} the variant of randomized smoothing proposed by
\citet{cohen2019certified} 
does not %
define
the 
prediction of the SNN to be given by the
expectation over $\Theta$ %
but %
by
$\arg \max _{c \in \mathcal{Y}} P(f(x+ \epsilon) = c)$, where $\epsilon \sim \mathcal{N}(0, \sigma^2 I)$. %
This makes our results not directly applicable to their networks.
However, %
they can probably be transferred to the decision boundaries and attacks  corresponding to their form of generating  predictions.
},
and any other class
of neural networks which use stochasticity at the input level~\citep[e.g.][]
{%
BART}
or within the network~\citep[e.g.][]{He_2019_CVPR, learn2perturb, liu_selfensemble, simpleSNN, weight_covariance-eustratiadis21a}.
In cases where
$%
\mathbb{E}_{\Theta} [f(x,\Theta)]$
is not tractable --- which is in practice  usually the case ---
the prediction of the stochastic classifier is 
approximated by its Monte Carlo (MC) estimate 
$
f^{\mathcal{S}}(x) := \frac{1}{S}\sum_{s=1}^S  f(x, \theta_s)
$, 
where the
samples in the sample set $\mathcal{S}= \{\theta_1,\dots, \theta_S\}$ 
are drawn from $p(\Theta)$.

\paragraph{Adversarial attacks on stochastic classifiers}
\label{subsec:adv_attacks}

Informally speaking, adversarial examples are inputs that %
are modified such that %
the network predicts wrong classes even though %
the changes to the inputs are not 
perceptible %
for a human. 
More precisely, let $x$ be an %
input with corresponding true label $y \in \{1,\dots,k\}$ that is classified correctly by the multi-class classifier $f(\cdot)$,  that is $\arg \max_c f_c(x)=y.$
We consider the most common attack form %
which
aims at %
misclassifying
$x$ by allowing for some predefined maximum magnitude of perturbation. That is, the attacker targets the optimization problem 
\begin{align} \label{eq:max_allowable_attack}
&\mathit{maximize} \;\;\; \mathcal{L}(f(x+\delta), y) \enspace , \; \; \mathit{s.t.} \;\;\; \|\delta\|_p \leq \eta \enspace, 
\end{align}
where $\mathcal{L}(\cdot, \cdot)$ 
is the loss function, 
$\|\cdot\|_p$ with $p \geq 1$ is the $\ell_p$-norm, 
and $\eta$ is the  \textit{perturbation strength}, i.e.~the maximal allowed magnitude of the attack (see %
figure~\ref{fig:decision_boundaries}a) for an illustration).
Common choices of loss functions include %
the cross-entropy loss and the negative margin loss $\mathcal{L}_{\text{margin}}(f(x+\delta), y) = -(f_y(x+\delta )- \max_{c\neq y} f_c(x+\delta))$.
In practice, targeting the optimization problem in eq.~\eqref{eq:max_allowable_attack} usually involves estimating $\delta$ by performing some kind of gradient-based optimization on the loss function~\citep[e.g.][]{Goodfellow_fgsm, madry2018towards}.
If $f^{\mathcal{S}}(\cdot)$ is a stochastic classifier (as introduced in the previous paragraph) approximated by its MC estimate, the loss gradient with respect to the input is stochastic as well and given by %
\begin{equation}
\label{eq:stochastic_gradient}
    \nabla_x \mathcal{L}(f^{\mathcal{S}}(x), y) =   \nabla_x \mathcal{L}\left( \frac{1}{S}\sum_{s=1}^S f(x,\theta_s), y\right) \enspace .
\end{equation}
Note, that for linear $\mathcal{L}$, as for example the margin loss for a fixed class $c$, the loss gradient of the mean prediction is equivalent to the mean of the loss gradients for single sample predictions. However, in general (e.g.~for the cross entropy loss) this is not the case.

\section{Geometrical robustness analysis}
\label{sec:theory}
After clarifying our understanding of stochastic classifiers and on how gradient-based adversarial attacks are conducted on them, we are now able to present %
a %
simple 
but general geometrical, adversarial robustness analysis for %
stochastic classifiers. It is motivated by the following observation: %
Each time a stochastic classifier is used for
a prediction, another set of realizations from the random vectors %
are drawn, resulting in another set of discriminant functions and corresponding decision boundaries. 
To put it into other words, the classifier gets a random variable itself. 
As a consequence, calculating a gradient-based adversarial attack for an input $x$ is done with respect to a drawn set of realizations
$\mathcal{A}=\{\theta^a_1, \theta^a_2, \dots, \theta^a_{S^{\mathcal{A}}} \}$, and thus $\delta$ from eq.~\eqref{eq:max_allowable_attack} is specific for a  
realization $f^{\mathcal{A}}(x)$ %
of the classifier,
which we specify by writing $\delta^{\mathcal{A}}$.
During inference, the resulting adversarial example $x_{\text{adv}} = x+ \delta^{\mathcal{A}}$ 
is then fed to another random classifier $f^{\mathcal{I}}$, which is based on a different set of realizations from the random vector %
$\mathcal{I}=\{\theta^i_1,\theta^i_2, \dots, \theta^i_{S^{\mathcal{I}}}\}$. 
From this perspective, the prediction %
model
$f^{\mathcal{I}}$ is robust against the attack if the distance from $x$ to the decision boundary (given by $f^{\mathcal{I}}_y(x)- f^{\mathcal{I}}_{c \neq y}(x) =0$ ) in the \emph{direction of $\delta^{\mathcal{A}}$} is larger than the length of $\delta^{\mathcal{A}}$, as illustrated in figures~\ref{fig:decision_boundaries}b) and c). 

\paragraph{Robustness conditions for stochastic attacks}
For a classifier with linear discriminant functions, we are able to turn the previously described observation %
into a theorem in which we 
derive a sufficient and necessary condition for the prediction %
model
to be robust against a given attack.
\begin{theorem}[Sufficient and necessary robustness condition for linear classifiers]
\label{thm:linear}
Let $f:\mathbb{R}^d\times \Omega^h  \rightarrow \mathbb{R}^k$ be a stochastic classifier with linear discriminant functions and $f^{\mathcal{A}}$ and  $f^{\mathcal{I}}$ be two MC estimates of the classifier. %
Let 
$x \in \mathbb{R}^d $ be a data point with label $y \in \{1,\dots,k\}$ and $\arg \max_c f_c^{\mathcal{A}}(x)= \arg \max_c f_c^{\mathcal{I}}(x)=y$, and let $x_{\text{adv}}=x+ \delta^{\mathcal{A}}$ be an adversarial example computed for solving the minimization problem~\eqref{eq:max_allowable_attack}
 for $f^{\mathcal{A}}$. 
It holds that $\arg \max_{c} f^{\mathcal{I}}_c(x+ \delta^{\mathcal{A}}) =y $ if and only if
\begin{equation}\label{eq:r_linear}
     \min_{c\neq y} \tilde{r}^{\mathcal{I}}_c > \|\delta^{\mathcal{A}}\|_2 \enspace,      \text{ with }
\end{equation}
\begin{equation}
      \tilde{r}^{\mathcal{I}}_c = \begin{cases}
      \infty \enspace ,   &\text{if} \;\;   \cos(\alpha_c^{\mathcal{I}, \mathcal{A}})= \frac{\langle -\nabla_x (f^{\mathcal{I}}_y(x) - f^{\mathcal{I}}_c(x)), \delta^{\mathcal{A}} \rangle}{\|\nabla_x (f^{\mathcal{I}}_y(x) - f^{\mathcal{I}}_c(x)) \|_2 \cdot \|\delta^{\mathcal{A}} \|_2}  \leq 0 \\
      \frac{ f^{\mathcal{I}}_y \left(x \right) - f^{\mathcal{I}}_c\left(x \right) }{ \| \nabla_x (f^{\mathcal{I}}_y(x) - f^{\mathcal{I}}_c(x) ) \|_2  \cdot \cos(\alpha_c^{\mathcal{I}, \mathcal{A}})} \enspace ,
        &\text{otherwise \enspace,
        }
        \nonumber
      \end{cases}
\end{equation}
where $\alpha_c^{\mathcal{I}, \mathcal{A}}$ is the angle between $-\nabla_x (f^{\mathcal{I}}_y(x) - f^{\mathcal{I}}_c(x) )$ and $\delta^{\mathcal{A}}$.

\end{theorem}

The proof which is based on Taylor expansion is given in supplement~A. %
The conditions for $\cos(\alpha_c^{\mathcal{I}, \mathcal{A}})$ have a nice geometrical interpretation:
An angle of more than 90° %
(which corresponds to a negative cosine value) 
indicates that the gradient $-\nabla_x (f^{\mathcal{I}}_y(x) - f^{\mathcal{I}}_c(x) )$ and 
the perturbation
$\delta^{\mathcal{A}}$  point into ``opposite'' directions and thus even for infinitely long moves into the direction of $\delta^{\mathcal{A}}$, the predicted label for $x$ will not change to class $c \neq y$.  
For positive cosine values,
$\tilde{r}^{\mathcal{I}}_c$ 
specifies the distance to the decision boundary in the attack direction. It looks similar to the minimal distance to the decision boundary in a deterministic 
setting which is given by $\frac{ f_y \left(x \right) - f_c\left(x \right) }{ \| \nabla_x (f_y(x) - f_c(x) ) \|_2 }$ and which %
is recovered if $\cos(\alpha_c^{\mathcal{I}, \mathcal{A}}) =1$.
A cosine value of one however can only occur if the gradient and perturbation direction are identical
i.e.,~if the margin loss is used for calculating  the attack direction and
if attack and inference model are identical. %
In practice, the latter is almost surely not the case due to
the finite sample approximation, and thus %
$\cos(\alpha_c^{\mathcal{I}, \mathcal{A}}) < 1$. 
This illustrates the robustness advantage induces by stochasticity.

The derived conditions may locally hold for classifiers which can be reasonably well approximated by a first-order Taylor approximation. To derive further guarantees, we can relax the linearity assumption by assuming discriminant functions which are $L$-smooth  as defined in the following.

\begin{definition}[$L$-smoothness~\citep{yang2021ensemble_robust}]
A differentiable function $f:\mathbb{R}^d \rightarrow \mathbb{R}^k$ is $L$-smooth, if for any $x_1, x_2\in \mathbb{R}^d$ and any output %
$c\in \{1,\dots,k\}$: 
\[
\frac{\| \nabla_{x_1} f_c(x_1)- \nabla_{x_2} f_c(x_2)\|_2}{\| x_1 - x_2\|_2} \leq L \enspace.
\]
\end{definition}
For such smooth discriminant function we can derive a sufficient (but not necessary\footnote{The necessity of this condition is not given since for non-linear decision boundaries it is possible that  behind a region of a different class there exists another region of class $y$ that an attack ends in if $\delta^{\mathcal{A}}$ is long enough.}) robustness condition specified by the following theorem, which is proven in supplement~A. %

\begin{theorem}[Sufficient condition for the robustness of a L-smooth stochastic classifier]
\label{thm:robustness_beta}
In the setting of Theorem~\ref{thm:linear}, let $f^{\mathcal{A}}$ and  $f^{\mathcal{I}}$ be $L$-smooth (instead of linear) discriminant functions. %
Then it holds that $\arg \max_c f_c^{\mathcal{I}}(x+ \delta^{\mathcal{A}}) = y$ if $\min_{c \neq y} r_c^{\mathcal{I}}> \| \delta^{\mathcal{A}}\|_2$ with 
\begin{equation*}
      r_c^{\mathcal{I}} =  \begin{cases}
      &\infty  \enspace , \enspace \text{if}  \enspace  \| \nabla_x (f^{\mathcal{I}}_y(x) - f^{\mathcal{I}}_c(x))\|_2  \cdot \cos(\alpha_c^{\mathcal{I}, \mathcal{A}}) + \frac{L}{2}  \cdot \| \delta^{\mathcal{A}} \|_2 \leq 0 \enspace ,  \\
      &\frac{f^{\mathcal{I}}_y(x) - f^{\mathcal{I}}_c(x)}{ \| \nabla_x (f^{\mathcal{I}}_y(x) - f^{\mathcal{I}}_c(x))\|_2  \cdot \cos(\alpha_c^{\mathcal{I}, \mathcal{A}}) + \frac{L}{2} \cdot \| \delta^{\mathcal{A}} \|_2  } \enspace , \enspace \text{otherwise
      } \end{cases}  \\
 \end{equation*}
with $\cos(\alpha^{\mathcal{I}, \mathcal{A}}_c)$ as in theorem~\ref{thm:linear}.
\end{theorem}
\paragraph{Factors influencing the robustness of stochastic classifiers}
Since in practice neither the parameter set $\mathcal{A}$ nor $\mathcal{I}$ is fixed, the quantity %
better
describing the practical robustness of a stochastic classifier with linear discriminant functions is
\begin{equation}\label{eq:probability}
    \mathbb{P}(\min_{c \neq y} \tilde r^{\mathcal{I}}_c >\|\delta^{\mathcal{A}}\|_2) \enspace ,
\end{equation} 
where $\tilde{r}^{\mathcal{I}}_c$ and $\delta^{\mathcal{A}}$ are random variables. 
For $L$-smooth models,
replacing $\tilde r^{\mathcal{I}}_c$ by $ r^{\mathcal{I}}_c$ leads to a lower %
bound on the  robustness.
Deriving an analytic expression for this probability is a hard problem.  
However, based on 
theorems~\ref{thm:linear} and \ref{thm:robustness_beta} %
it becomes clear
that a larger $\min_{c \neq y} \tilde{r}_c^{\mathcal{I}}$  %
relates to increasing the probability in eq.~\eqref{eq:probability} and thus to an increased robustness. 
We note
that
$\tilde{r}_c^{\mathcal{I}}$ with $c \neq y$ grows with i) larger prediction margins $f^{\mathcal{I}}_y(x) -f^{\mathcal{I}}_c(x)$, 
 ii) smaller gradient norms $\|\nabla_x (f^{\mathcal{I}}_y(x) - f^{\mathcal{I}}_c(x) )\|_2$, and 
iii) larger angles $\alpha_c^{\mathcal{I}, \mathcal{A} }$.
Larger prediction margins and smaller gradient norms were also found to positively impact the robustness of deterministic networks~\citep[e.g.][]{Ross_Doshi-Velez_2018}. In contrast, the dependency on the angle is unique to the stochastic setting. We therefore focus on the analysis of this factor in the following.

\paragraph{%
Analyzing the
expected angle} 
The angle $\an$ depends on both terms $-\nabla_x (f_y^{\mathcal{I}}(x)- f_c^{\mathcal{I}}(x))$
(for which we use the shorthand $ -\nabla_x \f(x)$ in the following)
and $\delta^{\A}$.
For further analysis, we first rewrite the gradient as
\begin{equation*}
    -\nabla_x f^{\mathcal{I}}_{y-c}(x) = \nabla_x \left( \frac{1}{S^{\mathcal{I}}}\sum_{s=1}^{S^{\mathcal{I}}} -f_{y-c}(x, \theta_s^i) \right) 
     =  \frac{1}{S^{\mathcal{I}}}\sum_{s=1}^{S^{\mathcal{I}}} -\nabla_x f_{y-c}(x, \theta_s^i) 
    \enspace.
\end{equation*}
Let $\mu= \mathbb{E}_{\Theta}[-\nabla_x f_{y-c}(x, \theta_s^i)]$ and $\Sigma$ be the covariance of $-\nabla_x f_{y-c}(x, \theta_s^i$), then it follows from the central limit theorem that for sufficiently many samples $-\nabla_x f^{\mathcal{I}}_{y-c}(x)
\sim \mathcal{N} \left( \mu, \frac{\Sigma}{S^{\mathcal{I}}} \right)$.
For simplicity let us assume that
the attack is based on the margin loss and only one iteration of gradient ascent. In this case,
$\delta^{\mathcal{A}}= \frac{\eta}{\|\hat{\delta}\|_2} \cdot \hat{\delta}$, with
$\hat{\delta} \sim \mathcal{N} \left( \mu, \frac{\Sigma}{S^{\mathcal{A}}} \right).$ %
Note, that we can neglect the scaling of the attack vector, since the angle only depends on $\hat{\delta}$.
Therefore, estimating the %
distribution of
$\alpha_c^{\mathcal{I}, \mathcal{A}}$ 
corresponds to estimating the 
distribution of the
angle
between two independent multivariate Gaussian random vectors with the same mean and (potentially) differently scaled versions of the same covariance.\footnote{In the case of more complicated attacks the mean will differ as well but we expect the findings of these section to generalize %
also to this scenario.
}
It is known for vectors from normalized standard multivariate Gaussian distributions that the %
mode of the distribution of the angle between two random vectors is equal to 90° and that %
with increased
dimension the concentration around this mode gets tighter ~\citep{cai_dist_angle}.
Unfortunately, deriving a closed form expression for the distribution or expectation in the general case is a challenging task and beyond the scope of this paper. However, we conjecture that the expectation of the angle increases proportionally with the variance and anti-proportionally with the norm of the mean. 
We illustrate %
the intuition behind this hypothesis in
figure~\ref{fig:exp_angle}. %
To empirically verify the correctness of this hypothesis
we conducted an %
experiment where we estimated the expected angle between two identically distributed
1,000-dimensional Gaussian random vectors for different  
choices of
means and diagonal covariances. %
More precisely, we sampled the mean and variances uniformly from $[0, t]$, where $t$ increased from zero to ten with step size $0.2$, and %
estimated
the expected angle based on 10,000 vector pairs drawn from the resulting distributions. Results are shown in figure~\ref{fig:exp_angle} on the right side.
As hypothesized, %
the
smaller the length of $\mu$ and the higher the %
average variances, the higher the expected angle.

\begin{figure}[]
\centering
{\includegraphics[width=0.35\columnwidth, height= 4.cm]{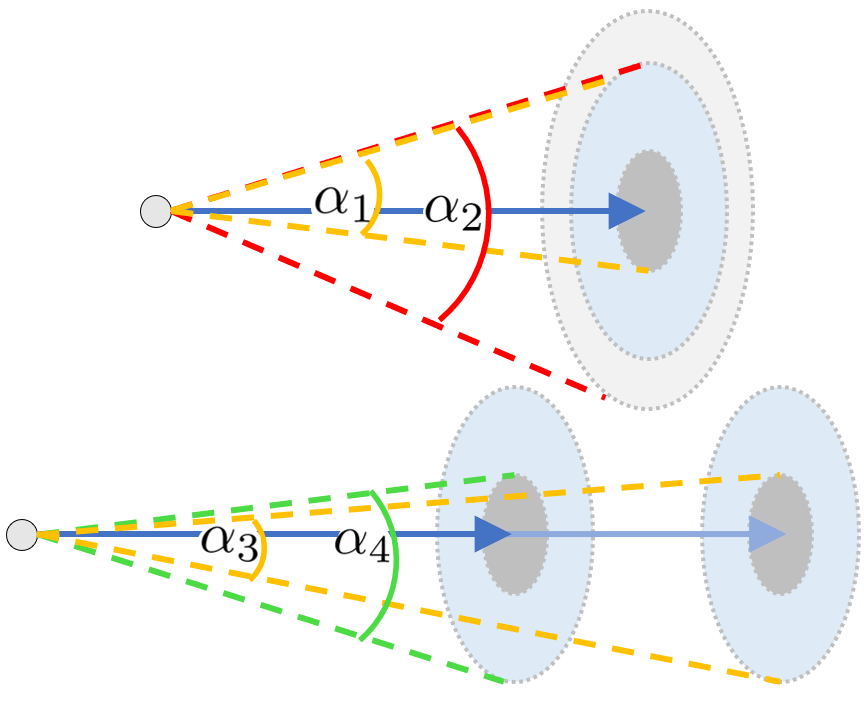} }%
\qquad
{\includegraphics[width=0.5\columnwidth]{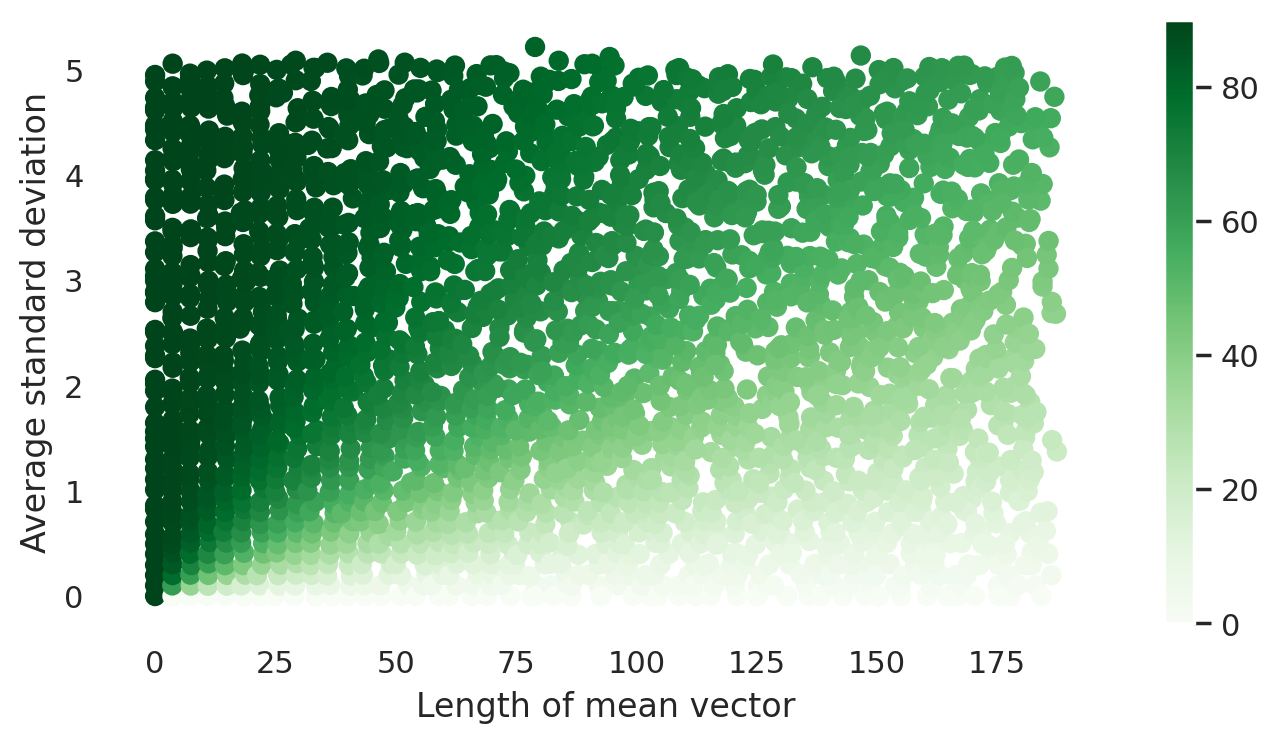} }
\caption{Dependence of the angle w.r.t~mean and variance of the gradient. \textbf{Left, top}: Sketch of decreased variance which leads to smaller angles. The mean of the gradient is shown as the blue arrow. The  circles indicate the areas of high probability for different scales of the covariance matrix. When the covariance is decreased (dark compared to light gray region) the maximum angle between vectors in these region to vectors from the blue region is decreased as well ($\alpha_1$ compared to $\alpha_2$). 
\textbf{Left, bottom}: Sketch of increased mean which leads to smaller maximal angles. When the covariance is fixed and the mean is increased (dark compared to light blue arrow) the maximum angle between vectors from high probability regions decreases ($\alpha_3$ compared to $\alpha_4$).
\textbf{Right}: Expected angle %
between two identically distributed 
1,000-dimensional Gaussian random vectors,
for Gaussian with different means and variances. Expected angles were calculated 
based on 10,000 pairwise draws 
from the respective Gaussian distributions.
}
\label{fig:exp_angle}
\end{figure}

\paragraph{Implications for the attack}
Based on the previous discussion, the only way the attacker can influence
the probability of a successful attack, and in this sense the
robustness of the model against the attack, %
is by increasing the amount of samples and thereby reducing the variance of the attack vector. %
A reduction of the variance
leads to a decrease of the expected angle as described in the previous section. This gives a more elaborated explanation of what was often loosely described as finding the ``correct'' gradient direction in previous work~\citep{obfuscated_grad}.
However, even if  the attacker would be able to take infinitely many samples, and thus the variance %
of the attack direction
would be reduced to zero, %
the expected value of the angle will be larger than zero because of the still existing stochasticity in the inference process. This stochasticity can even lead to an expected angle close to 90° if $\mu$ is short, and/or the %
covariance $\Sigma$ %
is high. That is, the advantage of obfuscating the optimal attack direction %
by incorporating stochasticity into the classifier can be decreased but not fully counterbalanced by taking more samples during the attack.
This might explain the finding in~\citet{He_2019_CVPR}, that the accuracy under attack stagnates at a higher level than the deterministic counterpart when increasing the number of iterations of iterative gradient-based attack methods.

\paragraph{Implications for the model}
From the perspective of the defender, %
the analysis of the angle shows that models with an increased gradient variance (which is often associated to  
a high
prediction variance) and a small norm of the mean gradient
are connected to larger values of %
$\an$ and thus to a higher probability of unsuccessfully attacks. %
This explains why including the norm of the mean gradient, the gradient variance, or the angle between gradients as regularization terms in the training of SNNs, as for example proposed by~\citet{diverse_directions_bender20a} and~\citet{GradDiv}\footnote{In these works the regularization terms were motivated  by maintaining input sensitivity and bounding the expected loss increase, respectively.},
lead to an increased empirical robustness.

It would be naturally to suspect that increasing the number of samples used during inference also decreases the robustness, since it decreases the expected angle. However, this is not the case, as it is counterbalanced by an decrease of the norm of the gradient estimate (i.e.~the second term in the denominator of $\tilde r_c$) with growing sample size.\footnote{%
We present a preposition showing that the
interval incorporating the gradient norm decreases to the norm of $\mu$ with increasing amount of samples in supplement~A.}
This can be seen by
rewriting
the expected denominator with respect to the %
two sample sets  $\mathcal{I}$ and $\mathcal{A}$
as

\begin{align*}
    & \mathbb{E}_{\mathcal{I},\mathcal{A}} \left[\| \nabla_x f^{\mathcal{I}}_{y-c} (x)\|_2 \cdot \cos(\alpha^{\mathcal{I}, \mathcal{A}}_c)\right]  
    = \mathbb{E}_{\mathcal{I},\mathcal{A}}\left[\| \nabla_x f^{\mathcal{I}}_{y-c} (x)\|_2 \cdot \frac{\langle -\nabla_x f^{\mathcal{I}}_{y-c} (x), \delta^{\mathcal{A}}  \rangle}{ \| \nabla_x f^{\mathcal{I}}_{y-c} (x)\|_2 \cdot \| \delta^{\mathcal{A}} \|_2 } \right]\\
    &= \mathbb{E}_{\mathcal{I},\mathcal{A}} \left[ \langle -\nabla_x f^{\mathcal{I}}_{y-c} (x), \frac{\delta^{\mathcal{A}}}{\| \delta^{\mathcal{A}} \|_2}  \rangle \right]  =  \sum_{i=1}^p  \mathbb{E}_{\mathcal{I}} \left[-\nabla_{x_i} f^{\mathcal{I}}_{y-c} (x) \right]\cdot  \mathbb{E}_{\A} \left[\frac{\delta_i^{\mathcal{A}}}{\| \delta^{\mathcal{A}} \|_2} \right]\enspace ,
\end{align*}
where the last equation %
holds
due to the independence of $\delta^{\A}$ and $-\nabla_x \f(x)$.
The expectation of
$\frac{\delta^{\mathcal{A}}}{\| \delta^{\mathcal{A}} \|_2}$ 
does not depend on the samples taken during inference and the expectation of the estimate of the derivative w.r.t.~the $i$-th input $ \mathbb{E}_{\mathcal{I}} \left[-\nabla_{x_i} f^{\mathcal{I}}_{y-c} (x) \right]$  is the same for different amounts of samples. Therefore, the expectation of the denominator does not change when changing the sample size during inference. This observation explains why the robustness of a stochastic classifier does not depend on the amount of samples taken during inference and gives an justification for picking an arbitrary sample size that allows for a good trade-off between efficiency and reduction of the variance of the MC estimate used for prediction.

\section{Experimental robustness analysis}
\label{sec:experiments}
In this section we empirically demonstrate that the findings of our theoretical analysis are transferable to SNNs and help to explain the mechanisms leveraging the experimentally observed robustness of previously proposed SNNs. 
\paragraph{Experimental setup}
Our experiments are conducted on two different image datasets: FashionMNIST~\citep{fashionmnist} and CIFAR10~\citep{cifar10}.\footnote{Additional experiments for CIFAR100 are presented in the supplement~C.5. %
}
For experiments on FashionMNIST we used 
feedforward neural networks (FNN) with  
two stochastic hidden layers, each with 128 neurons.
We trained the FNN as a Variational Matrix Gaussian (VMG, the BNN proposed by~\citet{louizos16}) via variational inference or as an infinite mixture (IM) with the maximum likelihood objective proposed by~\citet{daubener2020investigating}, %
with matrix variate normal distribution placed over the weights. We also trained FNNs of the same architecture, where we added Gaussian noise %
with $\sigma^2 = \{ 0.05, 0.1 \}$ %
to the input,
by minimizing the
cross-entropy. %
We refer to these models as \textit{stochastic input networks} (SINs) 0.05 and SIN 0.1, respectively.
Note, that these networks correspond to the basic networks proposed for randomized smoothing by \citet{lecuyer2019certified}.
For experiments on CIFAR10 %
we trained two wide residual networks (ResNet) with MC dropout layers~\citep{mc_dropout_gal}
applied after  
the convolution blocks
and dropout probabilities
$p=0.3$ and $p=0.6$. If not specified otherwise we used 100 samples of $p(\Theta)$ for inference on all datasets and calculated
adversarial attacks %
with the fast gradient (sign) 
method (FGM)~\citep{Goodfellow_fgsm},
\begin{wrapfigure}{r}{0.5\textwidth}
  \begin{center}
\includegraphics[width=0.5\textwidth]{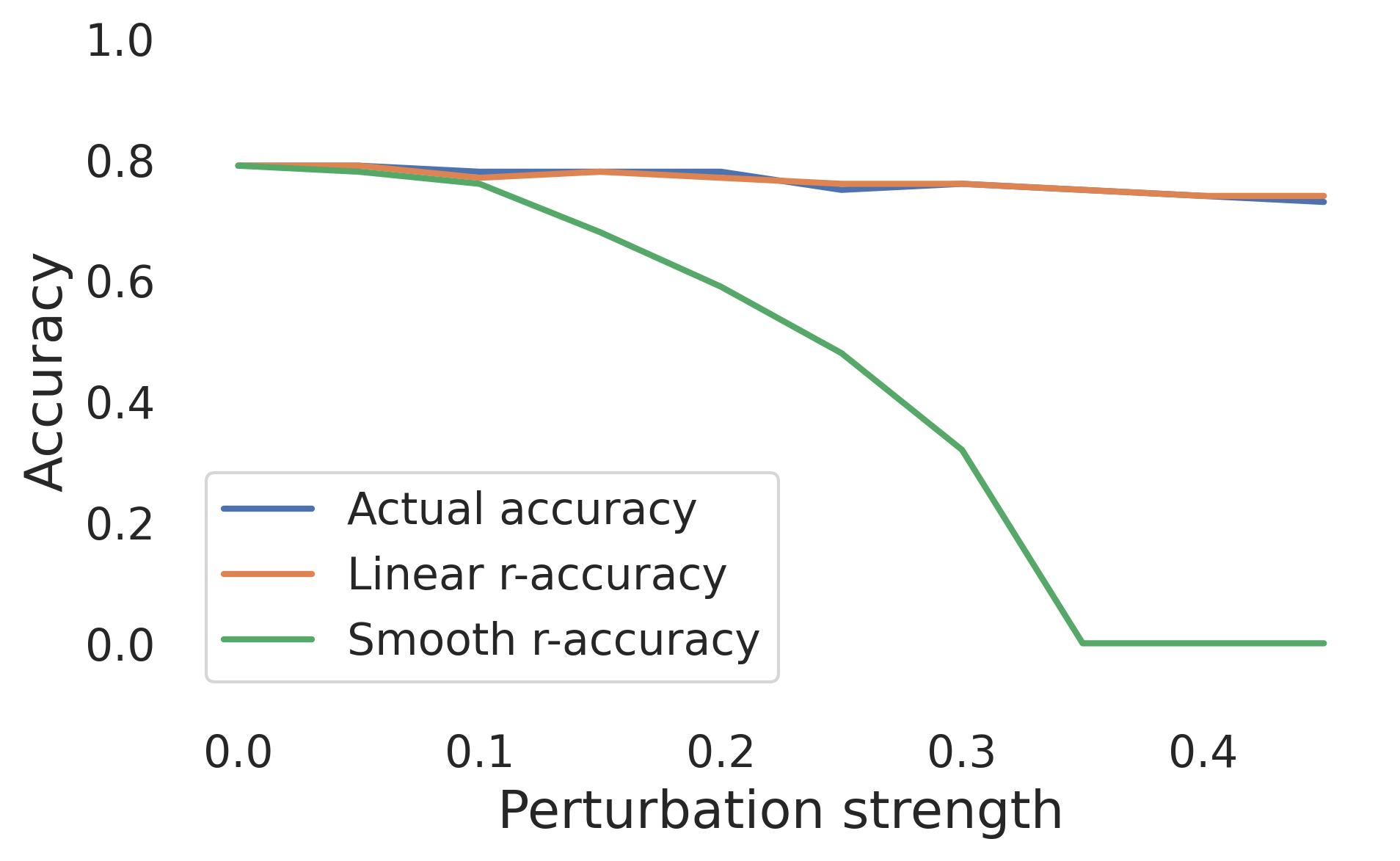}
  \end{center}
  \caption{Adversarial accuracy of the smoothed BNN model for  attacks based on 10 samples  vs percentage of images with
$\min_c r_c^{\mathcal{I}} > \| \delta^{\mathcal{A}} \|_2$ (smooth) and $\min_c \tilde{r}_c^{\mathcal{I}} > \| \delta^{\mathcal{A}} \|_2$ (linear) 
for the first 100 images from the FashionMNIST dataset.}
 \label{fig:r_accuracy_fmnist}
 \vspace{-1cm}
\end{wrapfigure}
 the cross-entropy loss for IMs, BNNs, and ResNets, the margin loss $\mathcal{L}_{\text{margin}}$ specified in section~\ref{sec:preliminaries} for SINs, and the $\ell_2$-norm constraint
based on the CleverHans repository~\citep{papernot2018cleverhans}. All 
experiments were run on a single NVIDIA GeForce RTX 2080 Ti.
We refer the reader to supplement~B for %
more
details on the datasets, models, and the training procedure.

\paragraph{Accuracy of robustness conditions}
\label{subsec:adv_r_acc}
We first investigated the practical transferability of 
the derived theorems. %
For enforcing the smoothness condition used in theorem \ref{thm:robustness_beta},
we built on the 
result from%
~\citet{yang2021ensemble_robust},  who showed that the $L$-smoothness parameter of a classifier $g:x \rightarrow \mathbb{E}_{\epsilon}[f(x+\epsilon)]$ smoothed with random noise  $\epsilon \sim \mathcal{N}(0, \sigma^2)$, is bounded by $L \leq 2/ \sigma^2$.
We therefore applied randomized smoothing during training of the models.
For the BNN, on which we focus in this section due to space restrictions (results for the other networks look qualitatively similar and can be found in supplement~C),
we replaced
each image in the batch by two noisy copies with Gaussian noise $\epsilon \sim \mathcal{N}(0, 0.1)$ which ensures $L\leq 20$. %
During prediction we estimated the expectation under Gaussian noise with
10 %
samples and also used 10 samples for calculating the FGM attack. 
We estimated the percentage of resulting attacks for which
$\min_c r_c^{\mathcal{I}} > \| \delta^{\mathcal{A}} \|_2$ and $\min_c \tilde{r}_c^{\mathcal{I}} > \| \delta^{\mathcal{A}} \|_2$ 
and compared this to the adversarial accuracy (i.e. the percentage of perturbed samples classified correctly) in figure~\ref{fig:r_accuracy_fmnist}. The percentage of samples fulfilling the  condition $\min_c r_c^{\mathcal{I}} > \| \delta^{\mathcal{A}} \|_2$ approaches zero with growing perturbation strength, indicating that the lower bound provided by $r_c^{\mathcal{I}}$ is rather loose due to the high (upper bound of the) L-smoothness constant.  On the other hand
the percentage of samples for which~\eqref{eq:r_linear} is fulfilled is a  good approximation of the real accuracy for small attack length, which suggests that the discriminant functions are approximately linear in a small neighborhood of the input.

\begin{figure}[htb]
\centering
\subfloat[IM]{\includegraphics[width=0.33\linewidth]{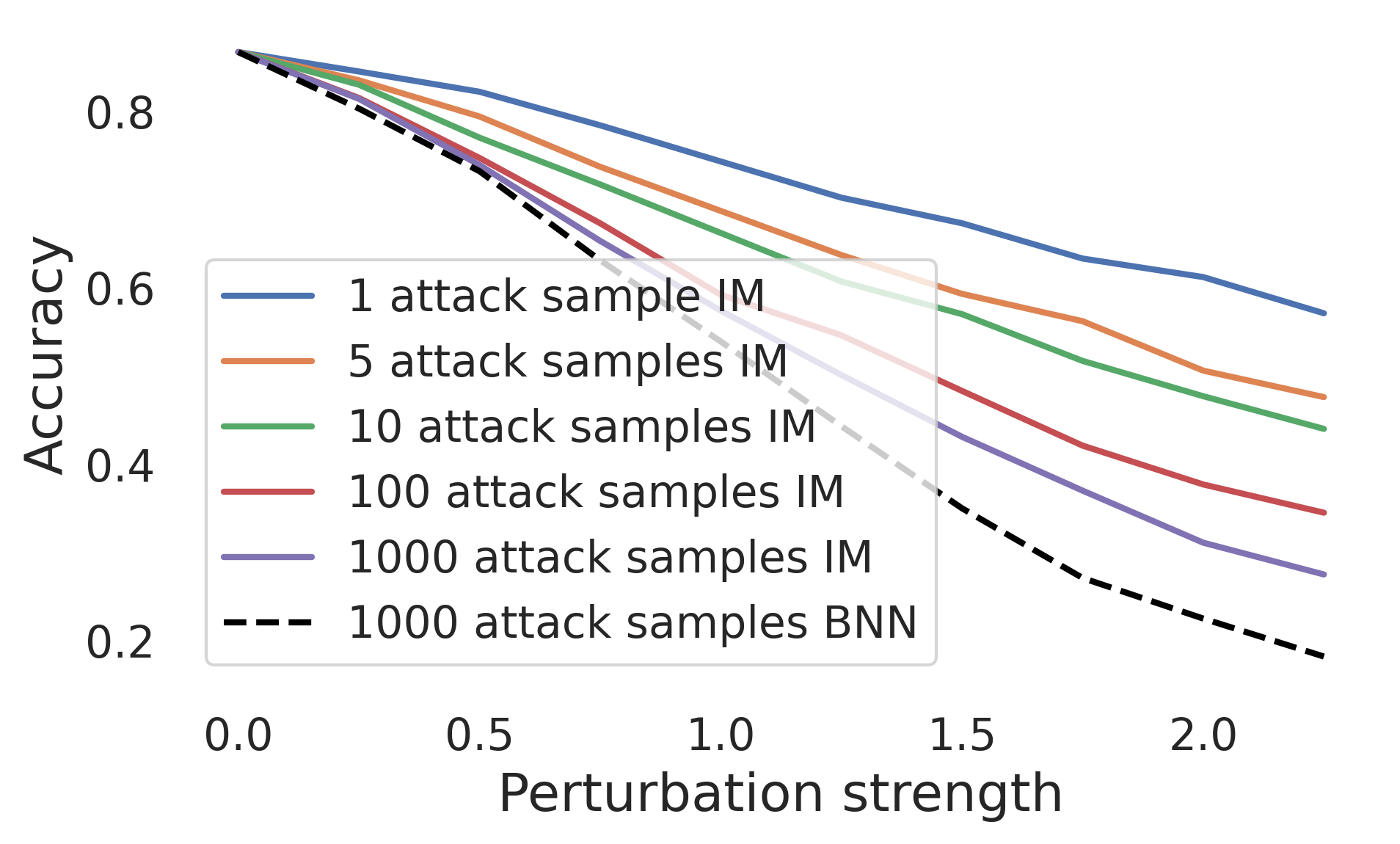}}
\subfloat[SIN]{\includegraphics[width=0.33\linewidth]{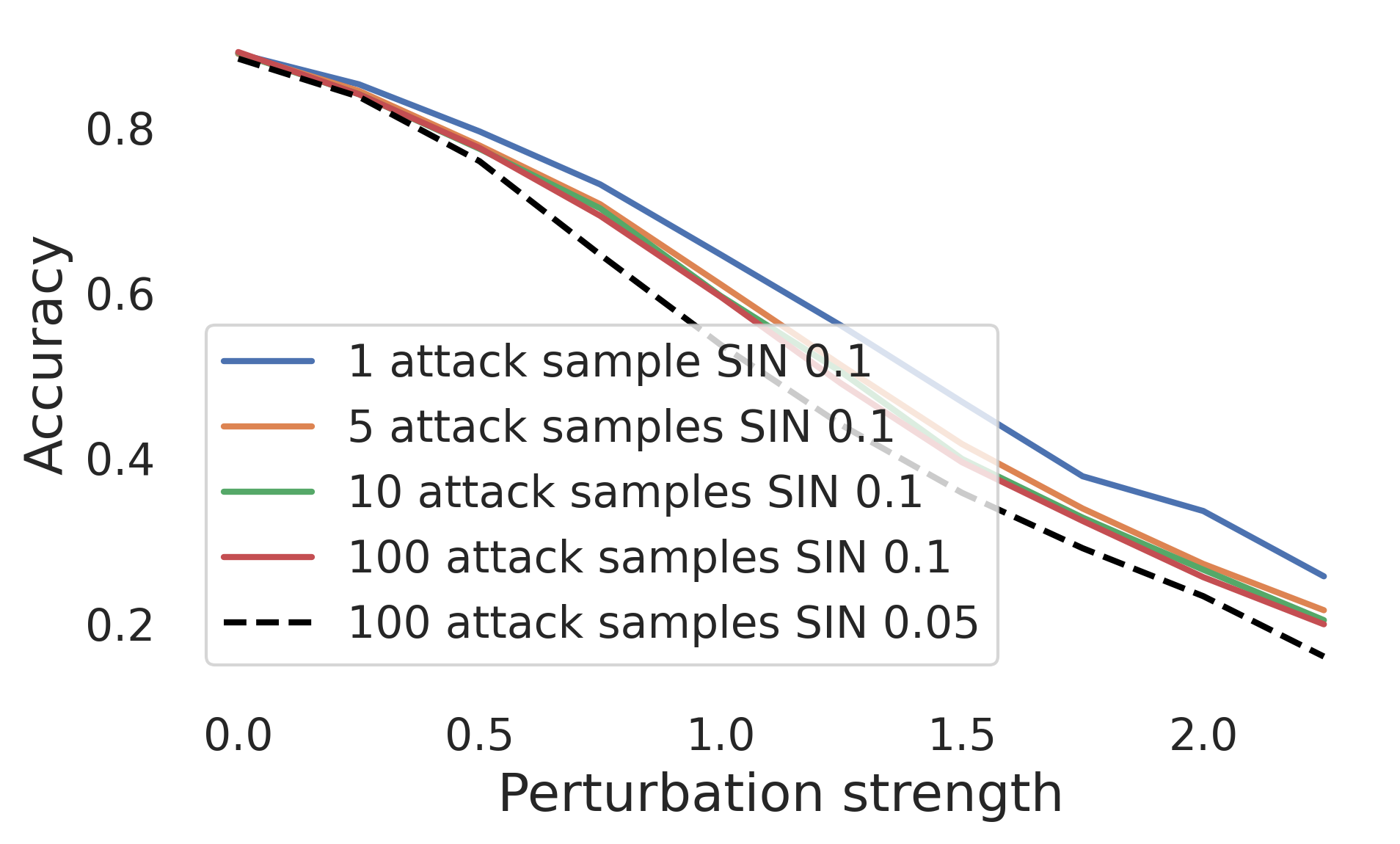}}
\subfloat[ResNet with dropout prob. 0.6 ]{\includegraphics[width=0.33\linewidth]{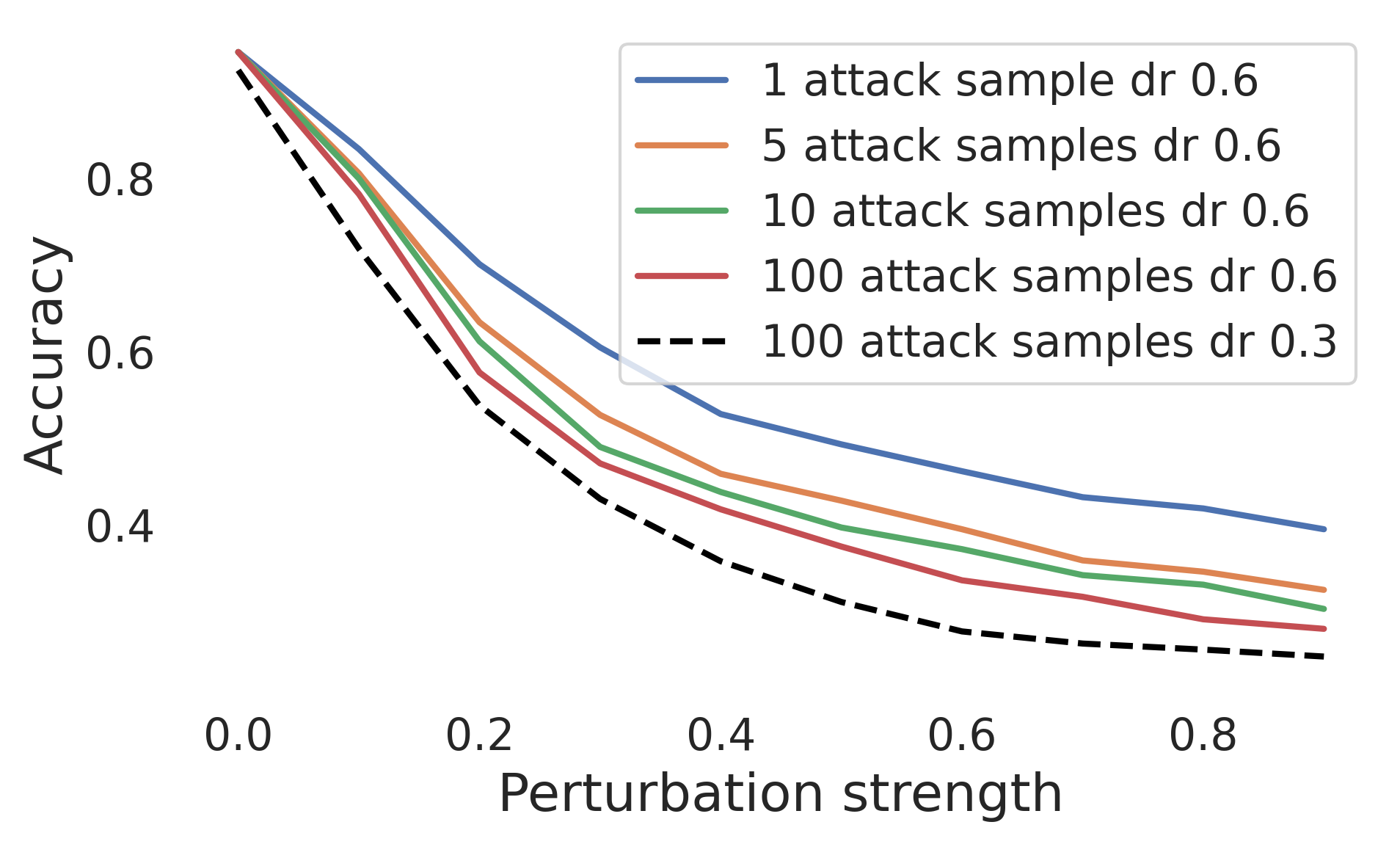}}
\caption{Accuracy under attack 
for different perturbation strength and amount of samples used for calculating the attack. The dashed line shows the adversarial accuracy for the 
models with the same architecture but less prediction variance.} 
\label{fig:acc_under_attack}
\end{figure}

\paragraph{Stronger attacks by increased sample size}

In this section we investigate the effect of varying the amount of samples used for calculating the attack, e.g. for $S^{\mathcal{A}} \in \{ 1, 5, 10, 100, (1000)\}$. 
Figure~\ref{fig:acc_under_attack} shows the resulting adversarial robustness of the
IM, SIN 0.1, and the ResNet with dropout probability 0.6. %
The accuracy under attack decreases for all models 
with increasing amount of samples used for calculating the attack as conjectured. %
We found this to hold also %
for stronger attacks, i.e.~attacks based on logits for IM and BNN (where we observe %
highly
confident softmax predictions), attacks with $L_{\infty}$ constraint, or \textit{projected gradient descent} (PGD)~\citep{madry2018towards} attacks with 100 iterations as shown in supplement~C. Note, that iterative attacks already increase the sample size due to their iterative nature, and therefore only few samples per iteration may %
sufficiently increase attack strength. %

Simultaneously to the accuracy we 
estimated %
the corresponding values of $\cos(\alpha_j^{\mathcal{I}, \mathcal{A}})$, for %
$j= \arg \min_c \tilde{r}_c^{\mathcal{I}}$,
for the first 1,000 test images and depicted them in figure~\ref{fig:angle_attack}.
It can be seen that the cosine values are increasing (which corresponds to decreasing angles) with growing sample size and that the larger the observed values, the lower the %
accuracy under attack as shown in figure~\ref{fig:acc_under_attack}.
This observation 
is in accordance to our theoretical analysis 
which predicts that an increased amount of samples leads to higher cosine values and in turn to less adversarial robustness.
Note however, that 
 even when taking 
 many samples $\cos(\alpha_j^{\mathcal{I}, \mathcal{A}})<1$, which underlines the fact that the optimal attack direction can not be recovered due to the still existing stochasticity in the inference procedure.
 Further results on the angle under different amounts of samples during attack can be found in supplement~C.2.

\begin{figure}[htb]
\centering
\subfloat[IM]{\includegraphics[width=0.33\linewidth]{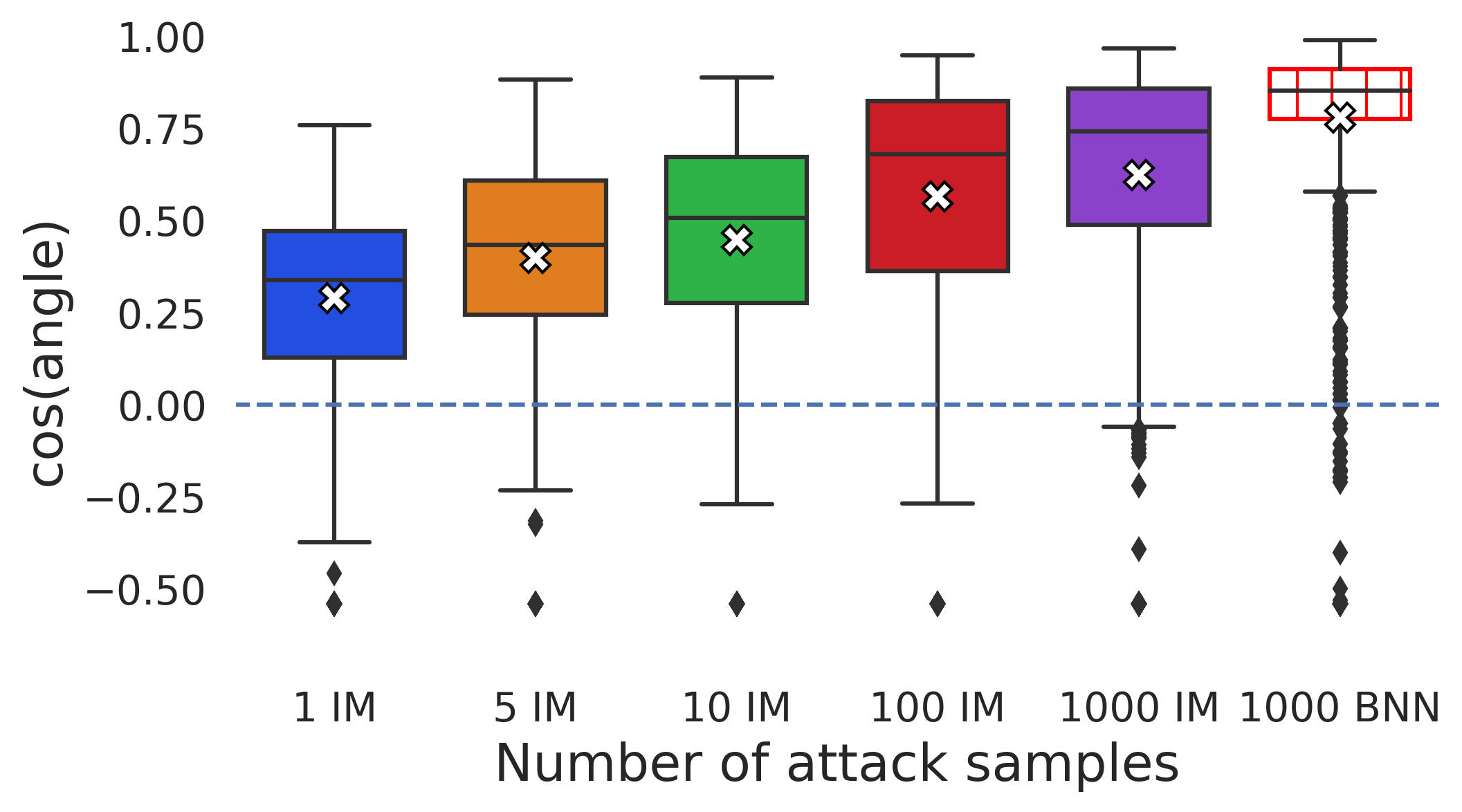}}
\subfloat[SIN 0.1]{\includegraphics[width=0.33\linewidth]{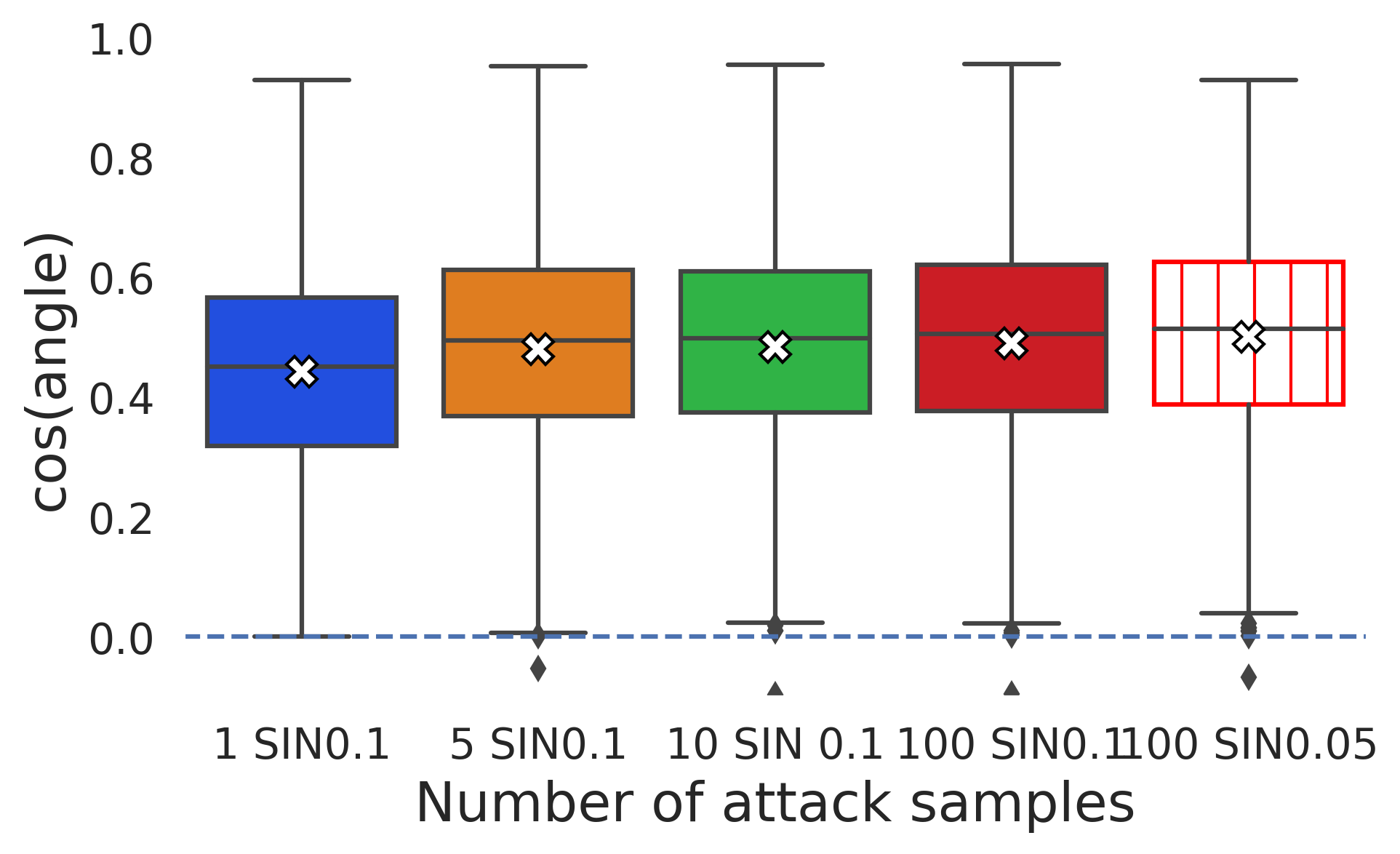}}
\subfloat[ResNet with dropout prob. 0.6 ]{\includegraphics[width=0.33\linewidth]{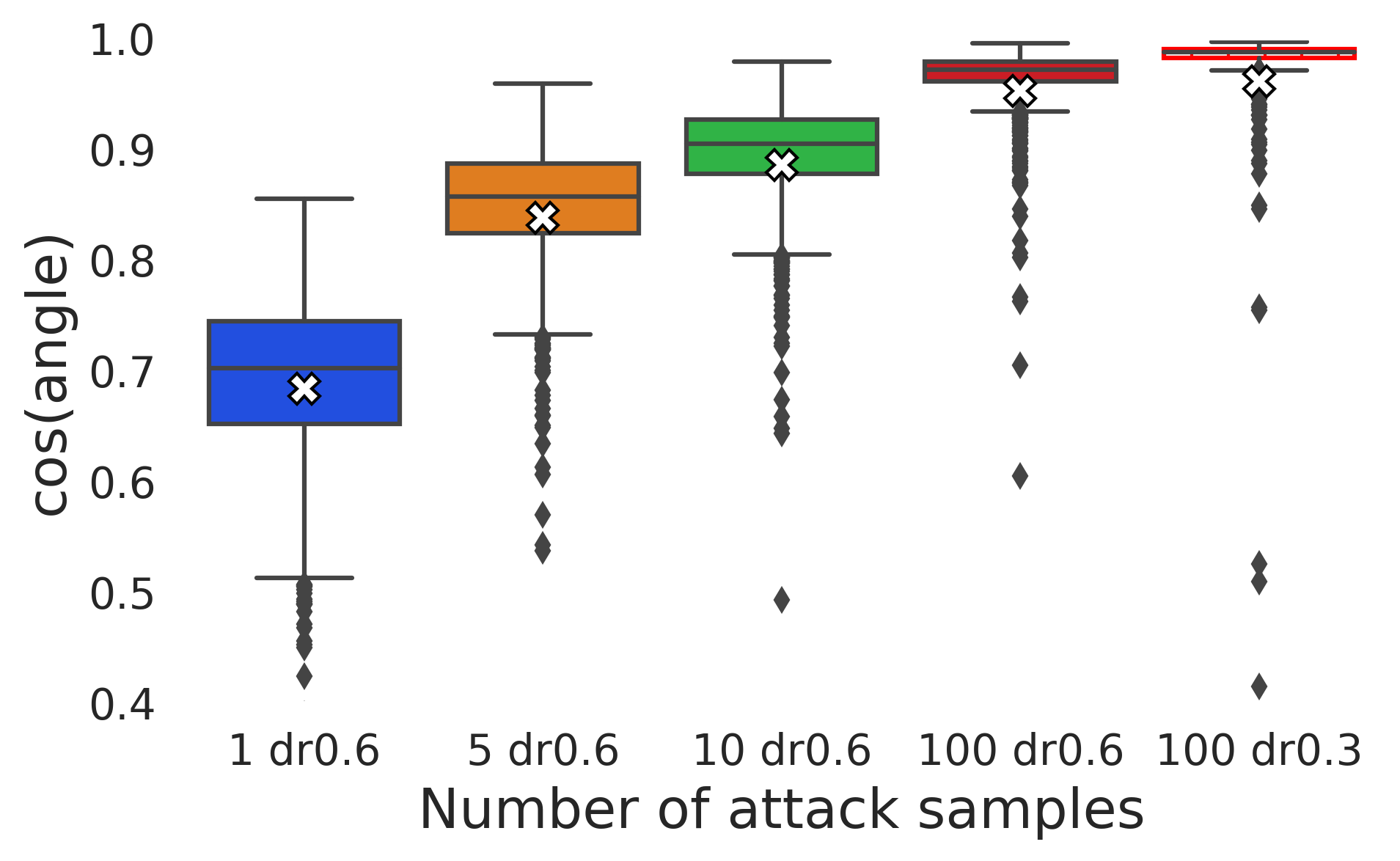}}
\caption{Cosine %
values
for adversarial examples with perturbation strength $1.5$ for a) and b) and $0.3$ for c) %
and
different amounts of samples. White crosses indicate mean values.} 
\label{fig:angle_attack}
\end{figure}

\paragraph{Prediction variance as robustness indicator}

In this section we compare the properties of SNNs that have
the same network architecture and a similar training procedure but different prediction variances. 
That is, we compare the BNN against the IM, SIN 0.05 against SIN 0.1, and %
the two ResNets with different dropout probabilities against each other. 
First, we estimated the standard deviation of the prediction  and the average standard deviation of the gradient entries of the models by 
calculating the average of the empirical estimates %
over the
first 1,000 examples from the respective test sets %
(see table~\ref{tab:pred_variance}).  %
As expected, 
the IM, SIN 0.1, and ResNet with dropout probability 0.6 have a %
larger standard deviation of the prediction and hence prediction variance than their respective counterpart. The higher standard deviation of the prediction %
translates to an increased standard deviation of the gradient. %
This explains why we observe smaller cosine values for these models compared to their counterparts (see right most boxplots in figure~\ref{fig:angle_attack} a) and c))  which in turn translate to an increased robustness (as indicated by the dotted accuracy curves in  figure~\ref{fig:acc_under_attack}).
We also found
that the length of the %
mean of the gradient that we estimated based on 1,000 samples is smaller for the  models with higher variance.\footnote{For SINs, it is known that a higher variance used during training relates to a lower Lipschitz-constant~\cite[e.g.][]{yang2021ensemble_robust, nips_smoothness_bound} which leads to a stronger smoothing effect.} This
might be another reason for the observed smaller angles
as discussed in the previous section.

\begin{table}
\caption{%
Empirical standard deviation %
of the prediction (for the correct class)
based on 1,000 single predictions (each with $1$ inference sample) and %
average length and variance of the corresponding %
gradient. %
We report averages over the first 1,000 images from the respective test set. 
}
\label{tab:pred_variance}
\vskip 0.15in
\begin{center}
\begin{small}
\begin{sc}
\begin{tabular}{l|cc|cc|cc}
\toprule
 & \multicolumn{4}{c|}{FashionMNIST} & \multicolumn{2}{c}{CIFAR10}\\ 
 & BNN & IM & SIN 0.05 & SIN 0.1 & dr 0.3 & dr 0.6  \\
\hline
Avg. std of predictions  & 0.0186 & 0.0473  & 85.4657  &  186.5599 & 0.0146 & 0.0196  \\
\hline
Avg. gradient length  & 0.5218 & 0.5044 & 49.0120 &  $\phantom{1}$48.8312 & 0.0713  & 0.0599  \\
Avg. Std of gradient  &  0.0316& 0.0959  &  $\phantom{1}$1.5741 &  $\phantom{1}\phantom{1}$1.8323 & 0.0000 & 0.0000 \\
\bottomrule
\end{tabular}
\end{sc}
\end{small}
\end{center}
\vskip -0.1in
\end{table}

\paragraph{Robustness in dependence of the amount of samples used during inference}
\label{subsec:ex_sample_inference}

In practice the amount of samples $S^{\mathcal{I}}$
drawn during inference is fixed to an arbitrary number. 
In this section we investigate the impact of varying $S^{\mathcal{I}} \in \{1, 5, 10, 100 \}$ which are values frequently used. %
First, we observed that only few samples 
are necessary to get reliable predictions on clean data for all models as shown by the test accuracies in  table~\ref{tab:accuracy_inference_samples}. 
Second, we found that increasing the amount of samples did not affect the adversarial accuracy. This can be explained by
the observation that the decrease of $\an$ caused by increasing the sample set is counterbalanced by an simultaneous decrease of the average norm of the gradient estimate, as can be seen by inspecting the results in figure \ref{fig:num_samples_inference}. That is, 
the product $\|  \nabla_x \f \|_2 \cdot \cos(\an)$ stays approximately the same %
regardless of the inference sample size as predicted by our theoretical analysis.
Interestingly, the effect of increasing the number of samples during inference has almost no effect on the prediction margin (more results are shown in 
supplement~C.4).

\begin{table}[h]
\caption{Test set and adversarial accuracy with 100 samples during the attack and allowed perturbation strength of $1.5$ on FashionMNIST and $0.3$ on CIFAR10 for increasing number of samples used during prediction. We estimated the average accuracy 10 times and report the average and standard deviation.}
\label{tab:accuracy_inference_samples}
\vskip 0.15in
\begin{center}
\begin{small}
\begin{sc}
\begin{tabular}{r|cc|cc|cc}
\toprule
 & \multicolumn{4}{c|}{FashionMNIST} & \multicolumn{2}{c}{CIFAR10} \\ 
$|\mathcal{I}|$& BNN & IM & SIN 0.05 & SIN 0.1 & dr 0.3 & dr 0.6  \\
\hline
\multicolumn{7}{c}{Test set accuracy}\\ 
\hline
1  &  $83.03 \pm 0.12$ & $79.15 \pm 0.27$ & $86.98 \pm 0.17$ &  $85.76 \pm 0.17$ & $92.08 \pm 0.13$ &  $92.68 \pm 0.15$ \\
5   &  $84.81 \pm 0.14$ & $84.29 \pm 0.21$ & $88.21 \pm 0.09$ &  
$87.49 \pm 0.15$ & $92.57 \pm 0.05$ &  $93.51 \pm 0.09$  \\
10   & $85.03 \pm 0.13$  & $85.02 \pm 0.18$ & $88.47 \pm 0.06$ &  $87.92 \pm 0.10$ & $92.66 \pm 0.06$ &  $93.60 \pm 0.06$  \\
100  & $85.21 \pm 0.06$  & $85.72 \pm 0.08$ & $88.63 \pm 0.07$ &  $88.27 \pm 0.06$ & $92.74 \pm 0.03$ &  $93.69 \pm 0.04$ \\
\hline
\multicolumn{7}{c}{Adversarial accuracy}\\ 
\hline
1   & $ 35.76 \pm 0.92 $ &  $ 44.67 \pm 1.19 $& $ 36.15 \pm 0.79 $  & $ 40.63 \pm 0.79 $  & $42.64 \pm 0.39$ &  $ 47.18 \pm 0.70 $ \\
5    & $ 36.19 \pm 0.47 $ &$ 47.08 \pm 0.73 $    & $ 35.92 \pm 0.51 $  & $ 40.57 \pm 0.44 $  & $42.73 \pm 0.31$ &  $ 47.00 \pm 0.38 $  \\
10    & $ 36.30 \pm 0.50 $ & $ 47.23 \pm 0.62 $ & $ 36.06 \pm 0.41 $  &  $ 39.93 \pm 0.66 $ & $42.69 \pm 0.37$ &  $ 47.01 \pm 0.33 $   \\
100    &  $ 36.50 \pm 0.39 $ &  $ 47.96 \pm 0.24 $& $ 35.91 \pm 0.23 $  &  $ 39.67 \pm 0.29 $ & $ 42.82 \pm 0.17 $ & $ 47.02 \pm 0.20 $    \\
\bottomrule
\end{tabular}
\end{sc}
\end{small}
\end{center}
\vskip -0.1in
\end{table}

\begin{figure}[htb]
\centering
\subfloat[$\cos(\an)$]{\includegraphics[width=0.33\linewidth]{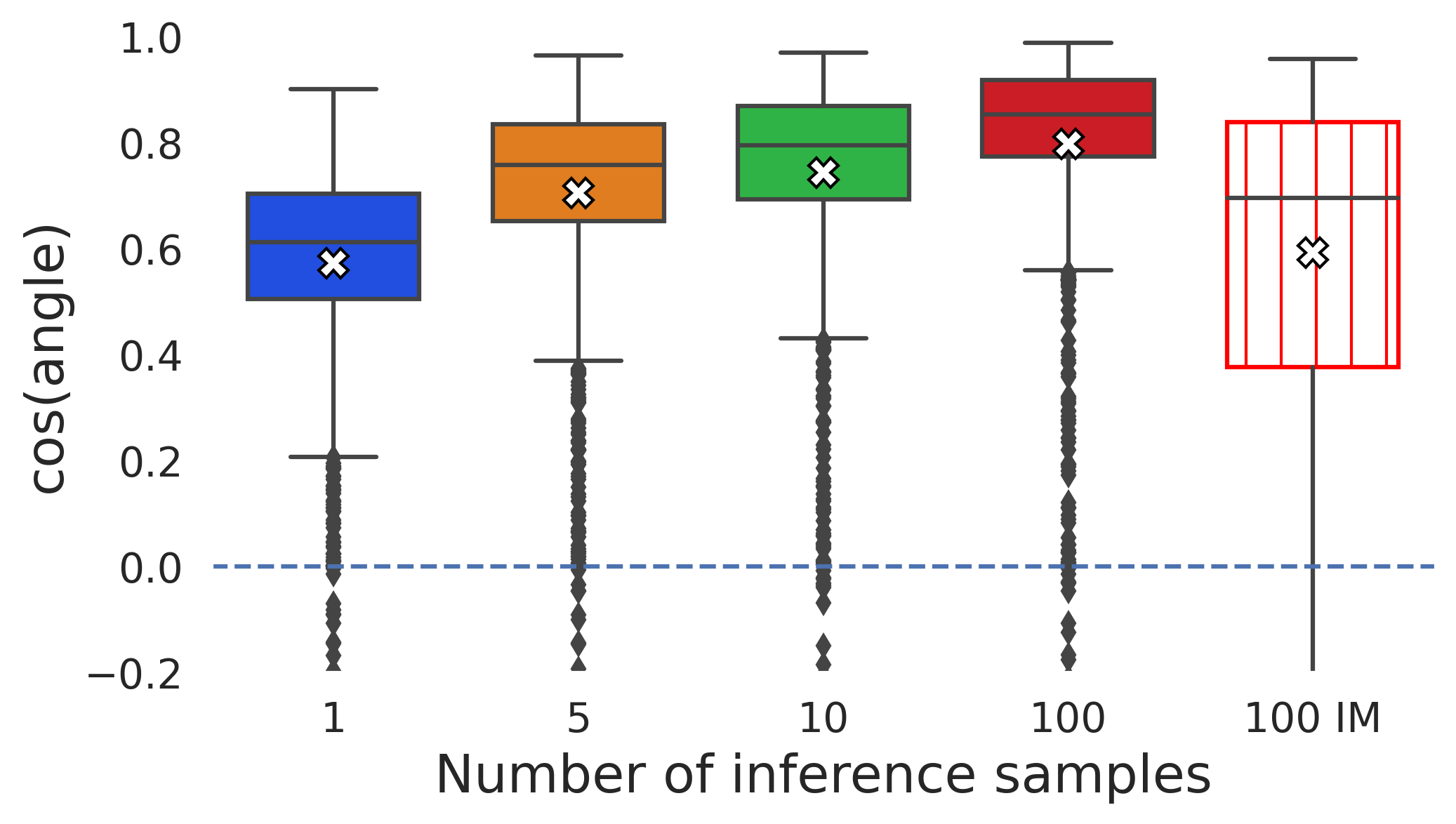}}
\subfloat[$\| \nabla_x f_{y-c}(x)\|_2$]{\includegraphics[width=0.33\linewidth]{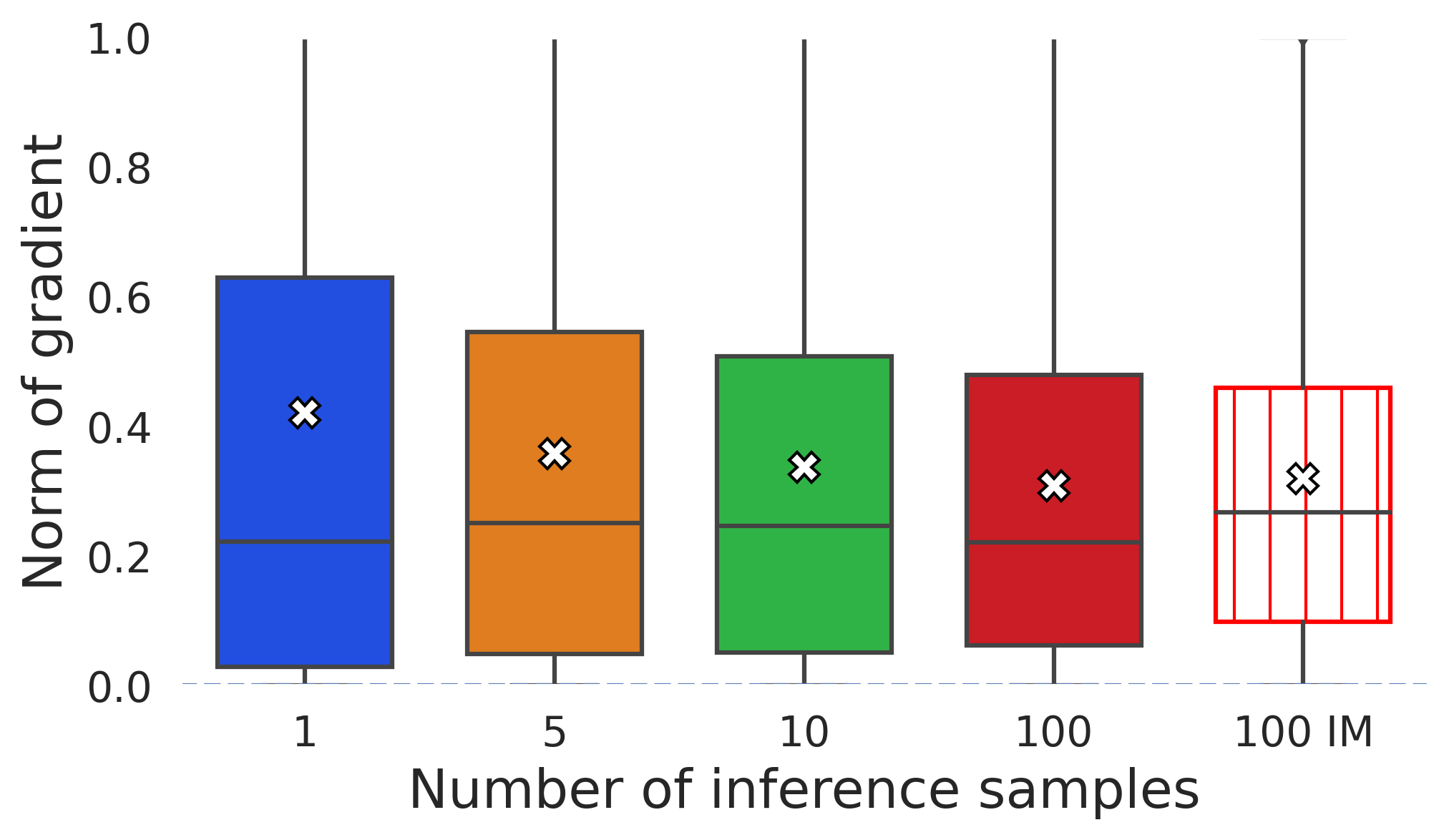}}
\subfloat[$f_{y-c}(x)$ ]{\includegraphics[width=0.33\linewidth]{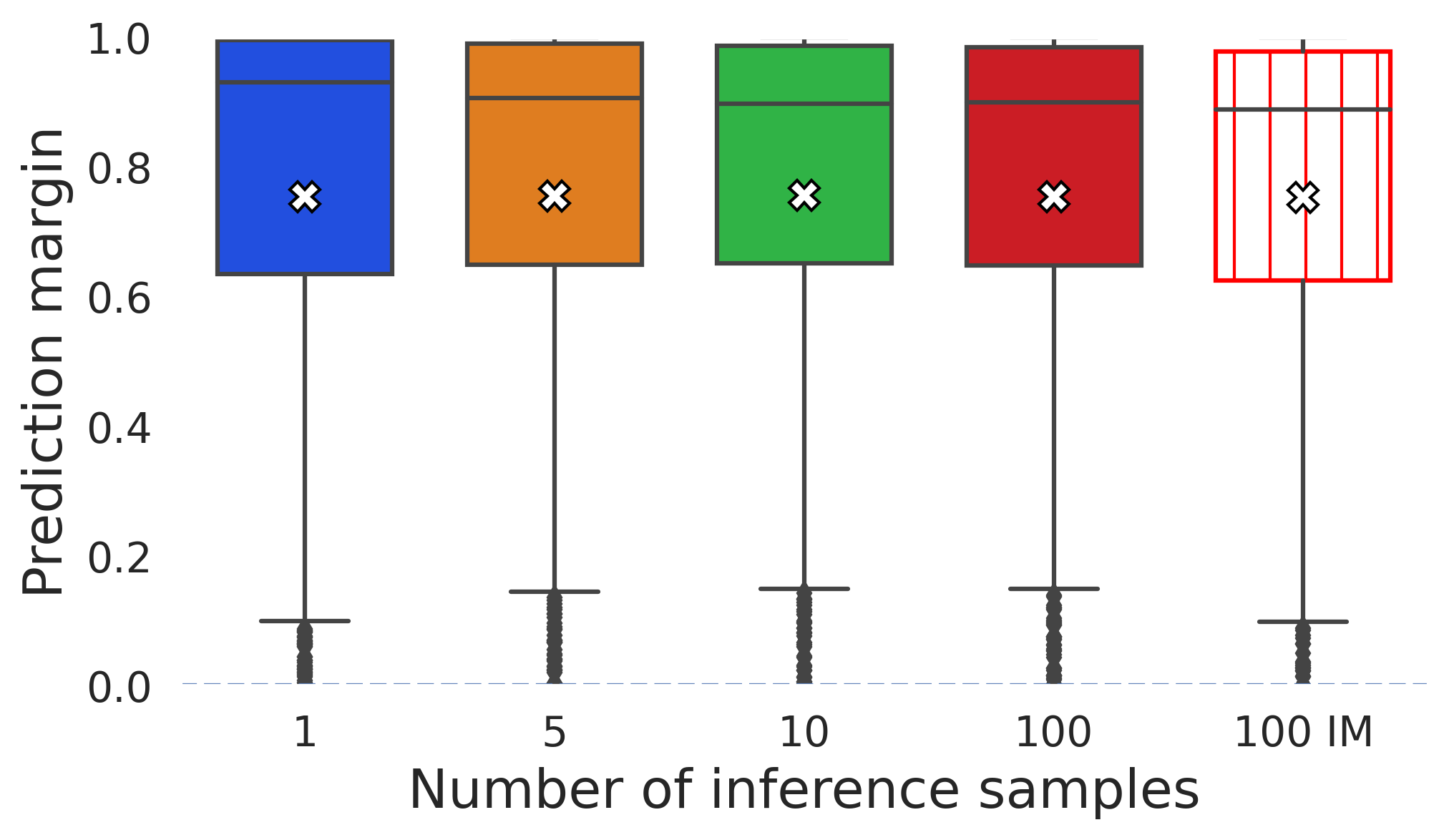}}
\caption{ %
Values of single factors
of $\tilde{r}^{\mathcal{I}}_c$ in dependence %
of varying number of samples used during inference for the BNN trained on FashionMNIST.} 
\label{fig:num_samples_inference}
\end{figure}

\section{Conclusion}
\label{sec:conclusion}
In this work we have stressed the fact that
stochastic neural networks (SNNs) (and stochastic classifiers in general) often depend on samples and thus their predictions are random variables themselves. For gradient-based adversarial attacks this means that the attack is calculated based on one realization of the stochastic network (which depends on multiple samples %
of the random variables used in the network) and applied to another which is used for inference. 
We derived a sufficient condition for this inference network to be robust against the calculated attack. This allowed us to identify the factors that lead to an increased robustness of stochastic classifiers: 
i) larger prediction margins
ii) a smaller norm of the gradient estimates %
and %
iii) higher angles between the attack direction and the direction to the closest decision boundary during inference.
The observed angles depend inverse proportionally on the norm of the expected gradient and proportionally on the variance of the gradient estimates. 
This variance can be reduced by increasing the sample size.
These insights enable us to explain previously %
reported empirical findings for SNNs from a geometrical perspective,
e.g.~that the robustness of SNNs is higher than the robustness of their deterministic counterparts %
even for strong attacks that are based on several samples~\citep{He_2019_CVPR, weight_covariance-eustratiadis21a}, why regularization of the gradient variance~\citep{diverse_directions_bender20a}, norm of the mean gradient, and angle~\citep{GradDiv} improves the adversarial robustness, and last but not least why increasing the sample size during attack is important to exploit its potential~\citep{obfuscated_grad}.
Therefore, our work poses a general applicable and simple framework,  %
that helps understanding the mode of operation of existing stochastic defense mechanisms, even if they were motivated from a different point of view. 
Moreover, we derived a justification for the common practice of choosing the sample size during inference in a way that balances its prediction certainty against the computational cost.
Finally, we %
believe 
our findings will be useful to evaluate and compare %
the robustness of different models, since they point out that they might require different amounts of samples during attack to sufficiently reduce variance and hope that they will help to improve the robustness of stochastic classifiers in future.
\paragraph{Potential negative societal impact and limitations.} Since we do not propose a new attack strategy, but %
contribute to a better
understanding of the robustness of SNNs and advise to cautiousness when determining sample sizes, we do not see negative impact.

\section*{Acknowledgments}
We would like to thank Denis Lukovnikov for his useful comments on our work and beautifying our graphics.
This work is funded by the Deutsche Forschungsgemeinschaft (DFG, German Research Foundation) under Germany's Excellence Strategy - EXC 2092 CASA - 390781972.
\bibliographystyle{plainnat}
\bibliography{library}

\begin{thebibliography}{49}
\providecommand{\natexlab}[1]{#1}
\providecommand{\url}[1]{\texttt{#1}}
\expandafter\ifx\csname urlstyle\endcsname\relax
  \providecommand{\doi}[1]{doi: #1}\else
  \providecommand{\doi}{doi: \begingroup \urlstyle{rm}\Url}\fi

\bibitem[Addepalli et~al.(2021)Addepalli, Jain, Sriramanan, and
  Babu]{Addepalli_2021_CVPR}
Sravanti Addepalli, Samyak Jain, Gaurang Sriramanan, and R.~Venkatesh Babu.
\newblock Boosting adversarial robustness using feature level stochastic
  smoothing.
\newblock In \emph{Proceedings of the IEEE/CVF Conference on Computer Vision
  and Pattern Recognition (CVPR) Workshops}, pages 93--102, 2021.

\bibitem[Akhtar and Mian(2018)]{survey_adv_vision}
Naveed Akhtar and Ajmal Mian.
\newblock Threat of adversarial attacks on deep learning in computer vision: A
  survey.
\newblock \emph{IEEE Access}, 6:\penalty0 14410--14430, 2018.

\bibitem[Athalye et~al.(2018)Athalye, Carlini, and Wagner]{obfuscated_grad}
Anish Athalye, Nicholas Carlini, and David~A. Wagner.
\newblock Obfuscated gradients give a false sense of security: Circumventing
  defenses to adversarial examples.
\newblock In \emph{Proceedings of the 35th International Conference on Machine
  Learning (ICML)}, pages 274--283, 2018.

\bibitem[Bender et~al.(2020)Bender, Li, Shi, Reiter, and
  Oliva]{diverse_directions_bender20a}
Christopher Bender, Yang Li, Yifeng Shi, Michael~K. Reiter, and Junier Oliva.
\newblock Defense through diverse directions.
\newblock In \emph{Proceedings of the 37th International Conference on Machine
  Learning (ICML)}, pages 756--766, 2020.

\bibitem[Biggio et~al.(2013)Biggio, Corona, Maiorca, Nelson,
  {\v{S}}rndi{\'{c}}, Laskov, Giacinto, and Roli]{biggio_adv_att}
Battista Biggio, Igino Corona, Davide Maiorca, Blaine Nelson, Nedim
  {\v{S}}rndi{\'{c}}, Pavel Laskov, Giorgio Giacinto, and Fabio Roli.
\newblock Evasion attacks against machine learning at test time.
\newblock In Hendrik Blockeel, Kristian Kersting, Siegfried Nijssen, and Filip
  {\v{Z}}elezn{\'y}, editors, \emph{Machine Learning and Knowledge Discovery in
  Databases}, pages 387--402. Springer Berlin Heidelberg, 2013.

\bibitem[Bubeck(2015)]{bubeck2015convex}
Sébastien Bubeck.
\newblock Convex optimization: Algorithms and complexity.
\newblock In \emph{arXiv preprint \url{https://arxiv.org/pdf/1405.4980}}, 2015.

\bibitem[Cai et~al.(2013)Cai, Fan, and Jiang]{cai_dist_angle}
Tony Cai, Jianqing Fan, and Tiefeng Jiang.
\newblock Distributions of angles in random packing on spheres.
\newblock \emph{Journal of Machine Learning Research}, 14\penalty0
  (21):\penalty0 1837--1864, 2013.

\bibitem[Carbone et~al.(2020)Carbone, Wicker, Laurenti, Patane\textquotesingle,
  Bortolussi, and Sanguinetti]{carbone2020robustness}
Ginevra Carbone, Matthew Wicker, Luca Laurenti, Andrea Patane\textquotesingle,
  Luca Bortolussi, and Guido Sanguinetti.
\newblock Robustness of bayesian neural networks to gradient-based attacks.
\newblock In \emph{Advances in Neural Information Processing Systems
  (NeurIPS)}, volume~33, pages 15602--15613, 2020.

\bibitem[Carlini and Wagner(2017{\natexlab{a}})]{carlini2017evaluating}
Nicholas Carlini and David Wagner.
\newblock Towards evaluating the robustness of neural networks.
\newblock In \emph{IEEE Symposium on Security and Privacy (SP)}, pages 39--57,
  2017{\natexlab{a}}.

\bibitem[Carlini and Wagner(2017{\natexlab{b}})]{carlini_bypassing}
Nicholas Carlini and David Wagner.
\newblock \emph{Adversarial Examples Are Not Easily Detected: Bypassing Ten
  Detection Methods}, page 3–14.
\newblock Association for Computing Machinery, 2017{\natexlab{b}}.

\bibitem[Cohen et~al.(2019)Cohen, Rosenfeld, and Kolter]{cohen2019certified}
Jeremy Cohen, Elan Rosenfeld, and Zico Kolter.
\newblock Certified adversarial robustness via randomized smoothing.
\newblock In \emph{Proceedings of the 36th International Conference on Machine
  Learning (ICML)}, volume~97, pages 1310--1320, 2019.

\bibitem[Croce and Hein(2020)]{Croce2020Provable}
Francesco Croce and Matthias Hein.
\newblock Provable robustness against all adversarial $l_p$-perturbations for
  $p\geq 1$.
\newblock In \emph{International Conference on Learning Representations
  (ICLR)}, 2020.

\bibitem[Croce et~al.(2019)Croce, Andriushchenko, and
  Hein]{croce_max_lin_regions}
Francesco Croce, Maksym Andriushchenko, and Matthias Hein.
\newblock Provable robustness of relu networks via maximization of linear
  regions.
\newblock In \emph{Proceedings of the Twenty-Second International Conference on
  Artificial Intelligence and Statistics (AISTATS)}, volume~89, pages
  2057--2066, 2019.

\bibitem[Dabouei et~al.(2020)Dabouei, Soleymani, Taherkhani, Dawson, and
  Nasrabadi]{cvpr_robust}
Ali Dabouei, Sobhan Soleymani, Fariborz Taherkhani, Jeremy Dawson, and
  Nasser~M. Nasrabadi.
\newblock Exploiting joint robustness to adversarial perturbations.
\newblock In \emph{IEEE/CVF Conference on Computer Vision and Pattern
  Recognition (CVPR)}, pages 1119--1128, 2020.

\bibitem[Daxberger et~al.(2021)Daxberger, Kristiadi, Immer, Eschenhagen, Bauer,
  and Hennig]{laplace2021}
Erik Daxberger, Agustinus Kristiadi, Alexander Immer, Runa Eschenhagen,
  Matthias Bauer, and Philipp Hennig.
\newblock Laplace redux--effortless {B}ayesian deep learning.
\newblock In \emph{Advances in Neural Information Processing Systems
  (NeurIPS)}, 2021.

\bibitem[Däubener and Fischer(2020)]{daubener2020investigating}
Sina Däubener and Asja Fischer.
\newblock Investigating maximum likelihood based training of infinite mixtures
  for uncertainty quantification.
\newblock In \emph{arXiv preprint \url{https://arxiv.org/abs/2008.03209}},
  2020.

\bibitem[Eustratiadis et~al.(2021)Eustratiadis, Gouk, Li, and
  Hospedales]{weight_covariance-eustratiadis21a}
Panagiotis Eustratiadis, Henry Gouk, Da~Li, and Timothy Hospedales.
\newblock Weight-covariance alignment for adversarially robust neural networks.
\newblock In \emph{Proceedings of the 38th International Conference on Machine
  Learning (ICML)}, volume 139, pages 3047--3056, 2021.

\bibitem[Gal and Ghahramani(2016)]{mc_dropout_gal}
Yarin Gal and Zoubin Ghahramani.
\newblock Dropout as a bayesian approximation: Representing model uncertainty
  in deep learning.
\newblock In \emph{Proceedings of The 33rd International Conference on Machine
  Learning (ICML)}, volume~48, pages 1050--1059, 2016.

\bibitem[Goodfellow et~al.(2015)Goodfellow, Shlens, and
  Szegedy]{Goodfellow_fgsm}
Ian~J. Goodfellow, Jonathon Shlens, and Christian Szegedy.
\newblock Explaining and harnessing adversarial examples.
\newblock In \emph{3rd International Conference on Learning Representations,
  (ICLR)}, 2015.

\bibitem[Gowal et~al.(2020)Gowal, Qin, Uesato, Mann, and
  Kohli]{gowal2020uncovering}
Sven Gowal, Chongli Qin, Jonathan Uesato, Timothy Mann, and Pushmeet Kohli.
\newblock Uncovering the limits of adversarial training against norm-bounded
  adversarial examples.
\newblock \emph{arXiv preprint}, 2020.
\newblock URL \url{https://arxiv.org/pdf/2010.03593}.

\bibitem[He et~al.(2019)He, Rakin, and Fan]{He_2019_CVPR}
Zhezhi He, Adnan~Siraj Rakin, and Deliang Fan.
\newblock Parametric noise injection: Trainable randomness to improve deep
  neural network robustness against adversarial attack.
\newblock In \emph{Proceedings of the IEEE/CVF Conference on Computer Vision
  and Pattern Recognition (CVPR)}, 2019.

\bibitem[Hein and Andriushchenko(2017)]{hein_formal_gurantees}
Matthias Hein and Maksym Andriushchenko.
\newblock Formal guarantees on the robustness of a classifier against
  adversarial manipulation.
\newblock In \emph{Advances in Neural Information Processing Systems
  (NeurIPS)}, volume~30, 2017.

\bibitem[Jeddi et~al.(2020)Jeddi, Shafiee, Karg, Scharfenberger, and
  Wong]{learn2perturb}
Ahmadreza Jeddi, Mohammad~Javad Shafiee, Michelle Karg, Christian
  Scharfenberger, and Alexander Wong.
\newblock Learn2perturb: An end-to-end feature perturbation learning to improve
  adversarial robustness.
\newblock In \emph{2020 IEEE/CVF Conference on Computer Vision and Pattern
  Recognition (CVPR)}, pages 1238--1247, 2020.

\bibitem[Kingma and Ba(2015)]{adam}
Diederik~P. Kingma and Jimmy Ba.
\newblock Adam: {A} method for stochastic optimization.
\newblock In \emph{3rd International Conference on Learning Representations
  (ICLR)}, 2015.

\bibitem[Krizhevsky et~al.({\natexlab{a}})Krizhevsky, Nair, and
  Hinton]{cifar10}
Alex Krizhevsky, Vinod Nair, and Geoffrey Hinton.
\newblock Cifar-10 (canadian institute for advanced research).
\newblock In \emph{\url{http://www.cs.toronto.edu/~kriz/cifar.html}},
  {\natexlab{a}}.

\bibitem[Krizhevsky et~al.({\natexlab{b}})Krizhevsky, Nair, and
  Hinton]{cifar10-100}
Alex Krizhevsky, Vinod Nair, and Geoffrey Hinton.
\newblock Cifar-10, cifar-100 (canadian institute for advanced research).
\newblock In \emph{\url{http://www.cs.toronto.edu/~kriz/cifar.html}},
  {\natexlab{b}}.

\bibitem[L{\'{e}}cuyer et~al.(2019)L{\'{e}}cuyer, Atlidakis, Geambasu, Hsu, and
  Jana]{lecuyer2019certified}
Mathias L{\'{e}}cuyer, Vaggelis Atlidakis, Roxana Geambasu, Daniel Hsu, and
  Suman Jana.
\newblock Certified robustness to adversarial examples with differential
  privacy.
\newblock In \emph{IEEE Symposium on Security and Privacy (SP)}, pages
  656--672, 2019.

\bibitem[Lee et~al.(2022)Lee, Kim, and Lee]{GradDiv}
Sungyoon Lee, Hoki Kim, and Jaewook Lee.
\newblock Graddiv: Adversarial robustness of randomized neural networks via
  gradient diversity regularization.
\newblock \emph{IEEE Transactions on Pattern Analysis and Machine
  Intelligence}, pages 1--1, 2022.

\bibitem[Liu et~al.(2018)Liu, Cheng, Zhang, and Hsieh]{liu_selfensemble}
Xuanqing Liu, Minhao Cheng, Huan Zhang, and Cho-Jui Hsieh.
\newblock Towards robust neural networks via random self-ensemble.
\newblock In \emph{Computer Vision -- ECCV}, pages 381--397, 2018.

\bibitem[Louizos and Welling(2016)]{louizos16}
Christos Louizos and Max Welling.
\newblock Structured and efficient variational deep learning with matrix
  gaussian posteriors.
\newblock In \emph{Proceedings of The 33rd International Conference on Machine
  Learning (ICML)}, volume~48, pages 1708--1716, 2016.

\bibitem[MacKay(1992)]{mackay92}
David J.~C. MacKay.
\newblock A practical bayesian framework for backpropagation networks.
\newblock \emph{Neural Comput.}, 4\penalty0 (3):\penalty0 448–472, 1992.
\newblock ISSN 0899-7667.

\bibitem[Madry et~al.(2018)Madry, Makelov, Schmidt, Tsipras, and
  Vladu]{madry2018towards}
Aleksander Madry, Aleksandar Makelov, Ludwig Schmidt, Dimitris Tsipras, and
  Adrian Vladu.
\newblock Towards deep learning models resistant to adversarial attacks.
\newblock In \emph{International Conference on Learning Representations
  (ICLR)}, 2018.

\bibitem[Neal(1996)]{neal}
Radford~M. Neal.
\newblock \emph{Bayesian Learning for Neural Networks}.
\newblock Springer-Verlag, Berlin, Heidelberg, 1996.

\bibitem[Papernot et~al.(2016)Papernot, McDaniel, Wu, Jha, and
  Swami]{papernot_distillation}
Nicolas Papernot, Patrick McDaniel, Xi~Wu, Somesh Jha, and Ananthram Swami.
\newblock Distillation as a defense to adversarial perturbations against deep
  neural networks.
\newblock In \emph{IEEE Symposium on Security and Privacy (SP)}, pages
  582--597, 2016.

\bibitem[Papernot et~al.(2018)Papernot, Faghri, Carlini, Goodfellow, Feinman,
  Kurakin, Xie, Sharma, Brown, Roy, Matyasko, Behzadan, Hambardzumyan, Zhang,
  Juang, Li, Sheatsley, Garg, Uesato, Gierke, Dong, Berthelot, Hendricks,
  Rauber, and Long]{papernot2018cleverhans}
Nicolas Papernot, Fartash Faghri, Nicholas Carlini, Ian Goodfellow, Reuben
  Feinman, Alexey Kurakin, Cihang Xie, Yash Sharma, Tom Brown, Aurko Roy,
  Alexander Matyasko, Vahid Behzadan, Karen Hambardzumyan, Zhishuai Zhang,
  Yi-Lin Juang, Zhi Li, Ryan Sheatsley, Abhibhav Garg, Jonathan Uesato, Willi
  Gierke, Yinpeng Dong, David Berthelot, Paul Hendricks, Jonas Rauber, and
  Rujun Long.
\newblock Technical report on the cleverhans v2.1.0 adversarial examples
  library.
\newblock In \emph{arXiv preprint \url{https://arxiv.org/abs/1610.00768}},
  2018.

\bibitem[Pinot et~al.(2019)Pinot, Meunier, Araujo, Kashima, Yger, Gouy-Pailler,
  and Atif]{nips19_theory_robustness_randomization}
Rafael Pinot, Laurent Meunier, Alexandre Araujo, Hisashi Kashima, Florian Yger,
  Cedric Gouy-Pailler, and Jamal Atif.
\newblock Theoretical evidence for adversarial robustness through
  randomization.
\newblock In \emph{Advances in Neural Information Processing Systems
  (NeurIPS)}, volume~32, 2019.

\bibitem[Raff et~al.(2019)Raff, Sylvester, Forsyth, and McLean]{BART}
Edward Raff, Jared Sylvester, Steven Forsyth, and Mark McLean.
\newblock Barrage of random transforms for adversarially robust defense.
\newblock In \emph{IEEE/CVF Conference on Computer Vision and Pattern
  Recognition (CVPR)}, pages 6521--6530, 2019.

\bibitem[Ross and Doshi-Velez(2018)]{Ross_Doshi-Velez_2018}
Andrew Ross and Finale Doshi-Velez.
\newblock Improving the adversarial robustness and interpretability of deep
  neural networks by regularizing their input gradients.
\newblock \emph{Proceedings of the AAAI Conference on Artificial Intelligence},
  32\penalty0 (1), 2018.

\bibitem[Salman et~al.(2019)Salman, Li, Razenshteyn, Zhang, Zhang, Bubeck, and
  Yang]{nips_smoothness_bound}
Hadi Salman, Jerry Li, Ilya Razenshteyn, Pengchuan Zhang, Huan Zhang, Sebastien
  Bubeck, and Greg Yang.
\newblock Provably robust deep learning via adversarially trained smoothed
  classifiers.
\newblock In \emph{Advances in Neural Information Processing Systems
  (NeurIPS)}, volume~32, 2019.

\bibitem[Szegedy et~al.(2014)Szegedy, Zaremba, Sutskever, Bruna, Erhan,
  Goodfellow, and Fergus]{intruding_Szegedy}
Christian Szegedy, Wojciech Zaremba, Ilya Sutskever, Joan Bruna, Dumitru Erhan,
  Ian~J. Goodfellow, and Rob Fergus.
\newblock Intriguing properties of neural networks.
\newblock In \emph{2nd International Conference on Learning Representations
  (ICLR)}, 2014.

\bibitem[Uesato et~al.(2018)Uesato, O'Donoghue, Kohli, and van~den
  Oord]{obscurity_robustness}
Jonathan Uesato, Brendan O'Donoghue, Pushmeet Kohli, and A{\"{a}}ron van~den
  Oord.
\newblock Adversarial risk and the dangers of evaluating against weak attacks.
\newblock In \emph{Proceedings of the 35th International Conference on Machine
  Learning (ICML)}, volume~80, pages 5032--5041, 2018.

\bibitem[Wicker et~al.(2020)Wicker, Laurenti, Patane, and
  Kwiatkowska]{wicker2020probabilistic}
Matthew Wicker, Luca Laurenti, Andrea Patane, and Marta Kwiatkowska.
\newblock Probabilistic safety for bayesian neural networks.
\newblock In \emph{Proceedings of the Thirty-Sixth Conference on Uncertainty in
  Artificial Intelligence (UAI)}, volume 124, pages 1198--1207, 2020.

\bibitem[Wicker et~al.(2021)Wicker, Laurenti, Patane, Chen, Zhang, and
  Kwiatkowska]{wicker_adv_train_BNN}
Matthew Wicker, Luca Laurenti, Andrea Patane, Zhuotong Chen, Zheng Zhang, and
  Marta Kwiatkowska.
\newblock Bayesian inference with certifiable adversarial robustness.
\newblock In \emph{Proceedings of The 24th International Conference on
  Artificial Intelligence and Statistics (AISTATS)}, volume 130, pages
  2431--2439, 2021.

\bibitem[Xiao et~al.(2017)Xiao, Rasul, and Vollgraf]{fashionmnist}
Han Xiao, Kashif Rasul, and Roland Vollgraf.
\newblock Fashion-mnist: a novel image dataset for benchmarking machine
  learning algorithms.
\newblock In \emph{arXiv preprint \url{https://arxiv.org/abs/1708.07747}},
  2017.

\bibitem[Xie et~al.(2018)Xie, Wang, Zhang, Ren, and Yuille]{xie2018mitigating}
Cihang Xie, Jianyu Wang, Zhishuai Zhang, Zhou Ren, and Alan Yuille.
\newblock Mitigating adversarial effects through randomization.
\newblock In \emph{International Conference on Learning Representations
  (ICLR)}, 2018.

\bibitem[Yang et~al.(2022)Yang, Li, Xu, Kailkhura, Xie, and
  Li]{yang2021ensemble_robust}
Zhuolin Yang, Linyi Li, Xiaojun Xu, Bhavya Kailkhura, Tao Xie, and Bo~Li.
\newblock On the certified robustness for ensemble models and beyond.
\newblock In \emph{International Conference on Learning Representations
  (ICLR)}, 2022.

\bibitem[Yu et~al.(2021)Yu, Yang, Li, Hospedales, and Xiang]{simpleSNN}
Tianyuan Yu, Yongxin Yang, Da~Li, Timothy Hospedales, and Tao Xiang.
\newblock Simple and effective stochastic neural networks.
\newblock \emph{Proceedings of the AAAI Conference on Artificial Intelligence},
  35\penalty0 (4):\penalty0 3252--3260, 2021.

\bibitem[Zagoruyko and Komodakis(2017)]{zagoruyko2017wide}
Sergey Zagoruyko and Nikos Komodakis.
\newblock Wide residual networks.
\newblock In \emph{arXiv preprint \url{https://arxiv.org/abs/1605.07146}},
  2017.

\bibitem[Zhang et~al.(2019)Zhang, Yu, Jiao, Xing, Ghaoui, and Jordan]{trades}
Hongyang Zhang, Yaodong Yu, Jiantao Jiao, Eric Xing, Laurent~El Ghaoui, and
  Michael Jordan.
\newblock Theoretically principled trade-off between robustness and accuracy.
\newblock In \emph{Proceedings of the 36th International Conference on Machine
  Learning (ICML)}, volume~97, pages 7472--7482, 2019.

\end{thebibliography}

\newpage
\section{Checklist}

\begin{enumerate}

\item For all authors...
\begin{enumerate}
  \item Do the main claims made in the abstract and introduction accurately reflect the paper's contributions and scope?
    \answerYes{}
  \item Did you describe the limitations of your work?
    \answerYes{The theoretical results are limited by their assumptions in section~\ref{sec:theory}.}
  \item Did you discuss any potential negative societal impacts of your work?
    \answerYes{ Below the conclusion.}
  \item Have you read the ethics review guidelines and ensured that your paper conforms to them?
    \answerYes{}
\end{enumerate}

\item If you are including theoretical results...
\begin{enumerate}
  \item Did you state the full set of assumptions of all theoretical results?
    \answerYes{In section~\ref{sec:theory}.}
  \item Did you include complete proofs of all theoretical results?
    \answerYes{Proofs are given in supplement~A.}
\end{enumerate}

\item If you ran experiments...
\begin{enumerate}
  \item Did you include the code, data, and instructions needed to reproduce the main experimental results (either in the supplemental material or as a URL)?
    \answerYes{We extended the description of datasets, models and training procedure from section~\ref{sec:experiments} in supplement~B and provided the code used in the main paper in the supplemental material.}
  \item Did you specify all the training details (e.g., data splits, hyperparameters, how they were chosen)?
    \answerYes{Details for the experiments which are not mention in section~\ref{sec:experiments} are given in supplement~B.}
  \item Did you report error bars (e.g., with respect to the random seed after running experiments multiple times)?
    \answerYes{Partially, where it was applicable and interesting (c.f. table~\ref{tab:accuracy_inference_samples}).}
  \item Did you include the total amount of compute and the type of resources used (e.g., type of GPUs, internal cluster, or cloud provider)?
    \answerYes{All experiments are run on one NVIDIA GeForce RTX 2080 Ti and package version requirements for the virtual environment are given in the code base.}
\end{enumerate}

\item If you are using existing assets (e.g., code, data, models) or curating/releasing new assets...
\begin{enumerate}
  \item If your work uses existing assets, did you cite the creators?
    \answerYes{Main python packages and models are cited in the paper as well as highlighted in the code.}
  \item Did you mention the license of the assets?
    \answerYes{The used code bases from others were published under MIT license. }
  \item Did you include any new assets either in the supplemental material or as a URL?
    \answerNA{Not applicable.}
  \item Did you discuss whether and how consent was obtained from people whose data you're using/curating?
    \answerNA{Not applicable.}
  \item Did you discuss whether the data you are using/curating contains personally identifiable information or offensive content?
    \answerNA{Not applicable.}
\end{enumerate}

\item If you used crowdsourcing or conducted research with human subjects...
\begin{enumerate}
  \item Did you include the full text of instructions given to participants and screenshots, if applicable? \answerNA{Not applicable.}
  \item Did you describe any potential participant risks, with links to Institutional Review Board (IRB) approvals, if applicable?
  \answerNA{Not applicable.}
  \item Did you include the estimated hourly wage paid to participants and the total amount spent on participant compensation?
    \answerNA{Not applicable.}
\end{enumerate}
\end{enumerate}

\newpage
\appendix
\section*{Supplement}
\section{Proofs}
\label{app:proofs}
In this section we proof the theorems of the main paper and recap the needed definitions to do so.

\begin{definition}[Maximum magnitude attack]
Given a multi-class classifier $f(\cdot)$, a loss function $\mathcal{L}(\cdot, \cdot)$, $\ell_p$-norm $\| \cdot \|_p$ and a given perturbation strength $\eta$ the optimization problem for a maximum magnitude attack can be written as
\begin{equation}\label{eq:app_max_allowable_attack}
    \textit{maximize} \enspace \mathcal{L}(f(x+ \delta), y) \ , \ w.r.t. \enspace \delta  \enspace s.t.  \ \| \delta\|_p \leq \eta \enspace.
\end{equation}
\end{definition}

\begin{theorem}[Sufficient and necessary robustness condition for linear classifiers]
\label{thm:app_linear}
Let $f:\mathbb{R}^d\times \Omega^h  \rightarrow \mathbb{R}^k$ be a stochastic classifier with linear discriminant functions and $f^{\mathcal{A}}$ and  $f^{\mathcal{I}}$ be two MC estimates of the classifier. %
Let 
$x \in \mathbb{R}^d $ be a data point with label $y \in \{1,\dots,k\}$ and $\arg \max_c f_c^{\mathcal{A}}(x)= \arg \max_c f_c^{\mathcal{I}}(x)=y$, and let $x_{\text{adv}}=x+ \delta^{\mathcal{A}}$ be an adversarial example computed for solving the minimization problem~\eqref{eq:app_max_allowable_attack}
 for $f^{\mathcal{A}}$. 
It holds that $\arg \max_{c} f^{\mathcal{I}}_c(x+ \delta^{\mathcal{A}}) =y $ if and only if
\begin{equation}\label{eq:app_r_linear}
     \min_{c \neq y} \tilde{r}^{\mathcal{I}}_c > \|\delta^{\mathcal{A}}\|_2 \enspace,      \text{ with }
\end{equation}
\begin{equation}
      \tilde{r}^{\mathcal{I}}_c = \begin{cases}
      \infty \enspace ,   &\text{if} \;\;   \cos(\alpha_c^{\mathcal{I}, \mathcal{A}})= \frac{\langle -\nabla_x (f^{\mathcal{I}}_y(x) - f^{\mathcal{I}}_c(x)), \delta^{\mathcal{A}} \rangle}{\|\nabla_x (f^{\mathcal{I}}_y(x) - f^{\mathcal{I}}_c(x)) \|_2 \cdot \|\delta^{\mathcal{A}} \|_2}  \leq 0 \\
      \frac{ f^{\mathcal{I}}_y \left(x \right) - f^{\mathcal{I}}_c\left(x \right) }{ \| \nabla_x (f^{\mathcal{I}}_y(x) - f^{\mathcal{I}}_c(x) ) \|_2  \cdot \cos(\alpha_c^{\mathcal{I}, \mathcal{A}})} \enspace ,
        &\text{otherwise \enspace,
        }
        \nonumber
      \end{cases}
\end{equation}
where $\alpha_c^{\mathcal{I}, \mathcal{A}}$ is the angle between $-\nabla_x (f^{\mathcal{I}}_y(x) - f^{\mathcal{I}}_c(x) )$ and $\delta^{\mathcal{A}}$.
\end{theorem}

\begin{proof}
An adversarial attack on $f^{\mathcal{I}}$ with the adversarial example $x+\delta^{\mathcal{A}}$ is not successful iff $\forall c \in \{1,2,\dots,k\}, c \neq y:$
\begin{equation*} \label{eq:app_unsucc_attack}
    f^{\mathcal{I
    }}_y \left(x +  \delta^{\mathcal{A}} \right) - f^{\mathcal{I}}_c\left(x + \delta^{\mathcal{A}} \right) >0 \enspace .
\end{equation*}
With Taylor expansion around $x$ we can rewrite 
$ f^{\mathcal{I
    }}_y \left(x +  \delta^{\mathcal{A}} \right) - f^{\mathcal{I}}_c\left(x + \delta^{\mathcal{A}} \right) $
as
\begin{align}
    & f^{\mathcal{I}}_y \left(x \right)  +\langle \nabla_x f^{\mathcal{I}}_y(x),  \delta^{\mathcal{A}} \rangle  -
    f^{\mathcal{I}}_c\left(x \right)
    - \langle \nabla_x
     f^{\mathcal{I}}_c(x), {\delta^{\mathcal{A}}} \rangle \nonumber \\
    =& f^{\mathcal{I}}_y \left(x \right)
    -  f^{\mathcal{I}}_c\left(x \right)
    +\langle \nabla_x f^{\mathcal{I}}_y(x)- \nabla_x
     f^{\mathcal{I}}_c(x),  \delta^{\mathcal{A}} \rangle \nonumber \\
    =& f^{\mathcal{I}}_y \left(x \right)
    -  f^{\mathcal{I}}_c\left(x \right) - \| \nabla_x (f^{\mathcal{I}}_y(x) - f^{\mathcal{I}}_c(x) ) \|_2 \cdot \| \delta^{\mathcal{A}} \|_2 \cdot \cos(\alpha_c^{\mathcal{I}, \mathcal{A}}) \label{eq:app_taylor-lin}
\end{align}
where $\alpha_c^{\mathcal{I}, \mathcal{A}} := \angle (- \nabla_x (f^{\mathcal{I}}_y(x) - f^{\mathcal{I}}_c(x) ),  \delta^{\mathcal{A}} ) $.
We can distinguish two cases for each $c$.

\paragraph{Case 1: $\cos(\alpha_c^{\mathcal{I}, \mathcal{A}}) \leq 0$.}
In this case last term
of eq.~\eqref{eq:app_taylor-lin} is negative or zero and thus
\begin{equation*}
    f^{\mathcal{I
    }}_y \left(x +  \delta^{\mathcal{A}} \right) - f^{\mathcal{I}}_c\left(x + \delta^{\mathcal{A}} \right) \geq  f^{\mathcal{I}}_y \left(x \right)
    -  f^{\mathcal{I}}_c\left(x \right) > 0 \enspace,
\end{equation*}
where the second inequality holds since 
$ \arg \max_c f_c^{\mathcal{I}}(x)=y$. 
\paragraph{Case 2: $\cos(\alpha_c^{\mathcal{I}, \mathcal{A}}) > 0$.}
In this case the last term of 
equation eq.~\eqref{eq:app_taylor-lin} is positive and thus
\begin{equation*}\label{eq:app_cos_alpha<1}
    f^{\mathcal{I
    }}_y \left(x +  \delta^{\mathcal{A}} \right) - f^{\mathcal{I}}_c\left(x + \delta^{\mathcal{A}} \right) \leq  f^{\mathcal{I}}_y \left(x \right)  -  f^{\mathcal{I}}_c\left(x \right) \enspace .
\end{equation*}
As we see from rearranging eq.~\eqref{eq:app_taylor-lin}, it holds $f^{\mathcal{I}}_y \left(x +  \delta^{\mathcal{A}} \right) - f^{\mathcal{I}}_c\left(x + \delta^{\mathcal{A}} \right) >0$ if 
\begin{equation}
     \tilde{r}_c^{\mathcal{I}} :=  \frac{ f^{\mathcal{I}}_y \left(x \right) - f^{\mathcal{I}}_c\left(x \right) }{ \| \nabla_x (f^{\mathcal{I}}_y(x) - f^{\mathcal{I}}_c(x) ) \|_2 \cdot \cos(\alpha_c^{\mathcal{I}, \mathcal{A}})} > \| \delta^{\mathcal{A}} \|_2 \label{eq:app_eta_value} \enspace  .
\end{equation}
For each class $c$ either case 1 holds and we define $\tilde{r}_c^{\mathcal{I}}:=\infty$, or condition \eqref{eq:app_eta_value} is fulfilled, which yields the condition stated in the theorem.
\end{proof}

In the main part of the paper we relaxed the linear classifier assumption by $L$-smoothness, which was defined as follows:

\begin{definition}[$L$-smoothness~\citep{yang2021ensemble_robust}]
A differentiable function $f:\mathbb{R}^d \rightarrow \mathbb{R}^k$ is $L$-smooth, if for any $x_1, x_2\in \mathbb{R}^d$ and any output dimension $c\in \{1,\dots,k\}$: 
\[
\frac{\| \nabla_{x_1} f_c(x_1)- \nabla_{x_2} f_c(x_2)\|_2}{\| x_1 - x_2\|_2} \leq L \enspace.
\]
\end{definition}
Next we restate one property of L-smooth functions which we will use in our proof of theorem 2. 

\begin{proposition}[\cite{bubeck2015convex}]
\label{prop:app_smooth}
Let $f$ be an $L$-smooth function on $\mathbb{R}^n$. For any $x,y \in \mathbb{R}^n$ it holds:
\begin{equation*}\label{eq:app_decent_lemma}
    |f(y) -f(x) - \langle \nabla_x f(x) , y-x \rangle | \leq  \frac{L}{2} \| y-x\|_2^2 \enspace.
\end{equation*}
\end{proposition}
\begin{proof}
From the fundamental theorem of calculus we know that for a differentiable function $f$ it holds that $f(y)- f(x) = \int_x^y \nabla_t f(t) dt$. By substituting  $x_t = x + t(y-x)$ we see that $x_0 = x$ and $x_1 = y$ and thus we can write $f(y)- f(x) = \int_0^1 \nabla f(x + t(y-x)) ^T \cdot (y-x) dt$. This allows the following approximations
\begin{align*}
    &|f(y) -f(x) - \langle \nabla_x f(x) , y-x \rangle | \\
    = &| \int_0^1 \nabla  f(x + t(y-x))^T \cdot (y-x) dt - (\nabla_x f(x))^T  (y-x)  |  \\
    \leq&  \int_0^1 | (\nabla f(x + t(y-x)) - \nabla_x f(x))^T ) \cdot (y-x) | dt  \\
    \overset{\text{Cauchy-Schwarz}}{\leq }&  \int_0^1 \| \nabla f(x + t(y-x)) - \nabla_x f(x)  \|_2 \cdot \| y-x \|_2 dt\\
    \overset{L \text{-smoothness}}{\leq }& L \cdot \| y-x \|_2^2  \cdot \int_0^1 t dt \\
    =& \frac{L}{2} \cdot \| y-x \|_2^2 \enspace .
\end{align*}
\end{proof}

\begin{theorem}[Sufficient condition for the robustness of an L-smooth stochastic classifier]
Let $f:\mathbb{R}^d\times \Omega^h \rightarrow \mathbb{R}^k$ be a stochastic classifier with $L$-smooth discriminant functions and $f^{\mathcal{A}}$ and  $f^{\mathcal{I}}$ be two MC estimates of the prediction. Let 
$x\in \mathbb{R}^d$ be a data point with label $y \in \{1,\dots ,k \}$ and $\arg \max_c f_c^{\mathcal{A}}(x)=\arg \max_c f_c^{\mathcal{I}}(x)=y$, and let $x_{\text{adv}}=x+ \delta^{\mathcal{A}}$ be an adversarial example computed for solving the minimization problem~\eqref{eq:app_max_allowable_attack}
 for $f^{\mathcal{A}}$. 
It holds that $\arg \max_{c} f^{\mathcal{I}}_c(x+ \delta^{\mathcal{A}}) =y $ if 
\begin{equation*}
     \min_{c \neq y} r_c^{\mathcal{I}} > \|\delta^{\mathcal{A}}\|_2 \enspace,
\end{equation*}
 with
\begin{equation*}\label{eq:app_ri_beta}
      r_c^{\mathcal{I}} =  \begin{cases}
      &\infty  \enspace \text{, if}  \enspace  \| \nabla_x (f^{\mathcal{I}}_y(x) - f^{\mathcal{I}}_c(x))\|_2  \cdot \cos(\alpha_c^{\mathcal{I}, \mathcal{A}}) + \frac{L}{2}  \cdot \| \delta^{\mathcal{A}} \|_2 \leq 0 ,  \\
      &\frac{f^{\mathcal{I}}_y(x) - f^{\mathcal{I}}_c(x)}{ \| \nabla_x (f^{\mathcal{I}}_y(x) - f^{\mathcal{I}}_c(x))\|_2  \cdot \cos(\alpha_c^{\mathcal{I}, \mathcal{A}}) + \frac{L}{2} \cdot \| \delta^{\mathcal{A}} \|_2  } \enspace  \text{, else} \end{cases}  \\
 \end{equation*}
and
\begin{equation*}
    \cos(\alpha^{\mathcal{I}, \mathcal{A}}_c)= \frac{\langle - \nabla_x (f^{\mathcal{I}}_y(x) - f^{\mathcal{I}}_c(x)), \delta^{\mathcal{A}} \rangle}{\|\nabla_x (f^{\mathcal{I}}_y(x) - f^{\mathcal{I}}_c(x)) \|_2 \cdot \|\delta^{\mathcal{A}} \|_2} \nonumber \enspace .
\end{equation*}
\end{theorem}

\begin{proof}
For better readability we write $f^{\mathcal{I}}_{y-c}(x):=f^{\mathcal{I}}_y(x) - f^{\mathcal{I}}_c(x)$.
Using the result from  proposition~\ref{prop:app_smooth} and reordering the terms,
we get the following lower bound: %
\begin{align}
     &  f_{y-c}^{\mathcal{I}}(x + \delta^{\mathcal{A}})  \nonumber \\
    & \geq f^{\mathcal{I}}_{y-c}(x) + \langle  \nabla_x ( f^{\mathcal{I}}_{y-c}(x) ),  \delta^{\mathcal{A}} \rangle - \frac{L}{2} \cdot \| \delta^{\mathcal{A}} \|_2^2  \nonumber  \\
     &= f^{\mathcal{I}}_{y-c}(x) - \langle - \nabla_x ( f^{\mathcal{I}}_{y-c}(x) ),  \delta^{\mathcal{A}} \rangle - \frac{L}{2} \cdot \| \delta^{\mathcal{A}} \|_2^2  \nonumber  \\
    &=f^{\mathcal{I}}_{y-c}(x) - \|-\nabla_x f^{\mathcal{I}}_{y-c}(x)\|_2 \cdot \| \delta^{\mathcal{A}} \|_2 \cdot \cos(\alpha_c^{\mathcal{I}, \mathcal{A}}) 
    - \frac{L}{2} \cdot \| \delta^{\mathcal{A}} \|_2^2 \nonumber  \\
    & = f^{\mathcal{I}}_{y-c}(x) - \left(\| \nabla_x f^{\mathcal{I}}_{y-c}(x)\|_2  \cdot \cos(\alpha_c^{\mathcal{I}, \mathcal{A}}) + \frac{L}{2} \cdot \| \delta^{\mathcal{A}} \|_2  \right) \cdot \| \delta^{\mathcal{A}} \|_2 \label{eq:app_lower} \enspace .
\end{align}
If eq.~\eqref{eq:app_lower} is bigger than zero, the attack cannot be successful.
Hence,
\begin{align}
& f^{\mathcal{I}}_{y-c}(x) - \bigg( \| \nabla_x f^{\mathcal{I}}_{y-c}(x)\|_2  \cdot \cos(\alpha_c^{\mathcal{I}, \mathcal{A}}) + \frac{L}{2}\cdot \| \delta^{\mathcal{A}} \|_2  \bigg) \cdot \| \delta^{\mathcal{A}} \|_2  \overset{!}{>} 0  \label{eq:app_pre}\\
&f^{\mathcal{I}}_{y-c}(x)  > \bigg(\| \nabla_x f^{\mathcal{I}}_{y-c}(x)\|_2  \cdot \cos(\alpha_c^{\mathcal{I}, \mathcal{A}}) + \frac{L}{2} \cdot \| \delta^{\mathcal{A}} \|_2  \bigg) \cdot \| \delta^{\mathcal{A}} \|_2    \label{eq:app_calc_bound} \enspace .
\end{align}
\paragraph{Case 1: $ \| \nabla_x f^{\mathcal{I}}_{y-c}(x)\|_2  \cdot  \cos(\alpha_c^{\mathcal{I}, \mathcal{A}}) + \frac{L}{2} \cdot \| \delta^{\mathcal{A}} \|_2 > 0$.}
Transforming eq.~\eqref{eq:app_calc_bound} leads to 
\begin{align*}
\frac{f^{\mathcal{I}}_{y-c}(x)}{ \| \nabla_x f^{\mathcal{I}}_{y-c}(x)\|_2  \cdot \cos(\alpha_c^{\mathcal{I}, \mathcal{A}}) + \frac{L}{2} \cdot \| \delta^{\mathcal{A}} \|_2  } >&  \| \delta^{\mathcal{A}} \|_2 \enspace.
\end{align*}
\paragraph{Case 2 : $\| \nabla_x f^{\mathcal{I}}_{y-c}(x)\|_2  \cdot \cos(\alpha_c^{\mathcal{I}, \mathcal{A}}) + \frac{L}{2} \cdot \| \delta^{\mathcal{A}} \|_2  = 0$. }
In this case eq.~\eqref{eq:app_pre} is trivially fulfilled because $f^{\mathcal{I}}_{y-c}(x)>0$ per definition.
\paragraph{Case 3: $\| \nabla_x f^{\mathcal{I}}_{y-c}(x)\|_2  \cdot  \cos(\alpha_c^{\mathcal{I}, \mathcal{A}}) + \frac{L}{2} \cdot \| \delta^{\mathcal{A}} \|_2  < 0$.}
For this case we get, that
\begin{align*}
-\frac{f^{\mathcal{I}}_{y-c}(x)}{ | \| \nabla_x f^{\mathcal{I}}_{y-c}(x)\|_2  \cdot \cos(\alpha_c^{\mathcal{I}, \mathcal{A}}) + \frac{L}{2} \cdot \| \delta^{\mathcal{A}} \|_2 | } <&  \| \delta^{\mathcal{A}} \|_2 \enspace ,
\end{align*}
which is always guaranteed based on the initial assumption that the benign input was classified correctly, which concludes the proof.
\end{proof}

The following proposition relates to footnote~5 from the main paper. We show, that the interval, in which the expectation of the gradient norm lies, decreases to the true length of $\mu$ with reducing the covariance.

\begin{proposition}
Let X be an $n$-dimensional random vector following a multivariate normal distribution with mean vector $\mu$ and diagonal covariance matrix $\Sigma$. Then the expectation of $\| X \|_2$ can be upper and lower bounded by
\begin{equation*}
  \| \mu \|_2 \leq \mathbb{E}[\| X \|_2 ] \leq \sqrt{ \| \mu \|_2^2 + tr\left(\Sigma\right)} \enspace .
\end{equation*}
\end{proposition}

\begin{proof}
We first look at the lower bound which by convexity of the norm and Jensen inequality can be derived via
\begin{equation*}
    \mathbb{E}[\| X \|_2 ] \leq \| \mathbb{E}[ X ] \|_2  = \| \mu \|_2 \enspace .
\end{equation*}
Let $X'\sim \mathcal{N}(0,\textbf{1}_n)$, with $\textbf{1}_n$ an $n\times n$-dimensional unit matrix.
With Jensen inequality and concavity of the square-root function we derive the upper bound 
\begin{align*}
        &\mathbb{E}[ \| X \|_2 ]
        =\mathbb{E}\left[ \| \mu + X' \cdot 
         \Sigma^{1/2}
         \|_2 \right]\\
      &\leq   \left(
      \mathbb{E}\left[\left\| \mu + X' \cdot 
      \Sigma^{1/2} \right\|_2^2 \right] \right)^{1/2} \\
      &= \bigg ( \| \mu \|_2^2 + 2 \mu^T 
      \Sigma^{1/2}   \mathbb{E}[ X'] + \mathbb{E}\left[X'^T \Sigma^{{1/2}^{T}}  
      \Sigma^{1/2} X'\right]  \bigg )^{1/2} \\
      &= \sqrt{ \| \mu \|_2^2 + tr\left(\Sigma\right)  } \enspace .
\end{align*}
\end{proof}

\section{Additional information on datasets, models and training}\label{app:hyper}
In the following we describe additional details on the datasets, models and training procedures which were not stated in the main paper due to space restrictions. Additionally, please find the code for the results in the main paper attached in the supplementary material. %

\subsection{Datasets}
We used three well know datasets: FashionMNIST~\citep{fashionmnist},  CIFAR10 and CIFAR100~\citep{cifar10-100}, which consist out of 60,000 training and  10,000 test images of dimension $28\times 28$ or $32 \times 32 \times 3$ in case of CIFAR where each image is uniquely associated to one out of 10 or 100  possible labels. We took all datasets from the torchvision package with the predefined training and test split.

\subsection{Models trained on FashionMNIST}
\label{app:hyper_fashion}
For training the BNN and IM we used the exact same hyperparameters. First, we assumed a standard normal prior decomposed as matrix variate normal distributions for the BNN and approximated the posterior distribution via maximizing the evidence lower bound (ELBO). For the IM we
added a Kullback-Leibler distance from the trained parameter distribution to a standard matrix variate normal distribution as a regularization term. %
For both models we used a batch size of 100 and trained for 50 epochs with Adam~\citep{adam} and an initial learning rate of $0.001$. To leverage the difference between IM and BNN we used 5 samples to approximate the expectation in the ELBO/IM-objective. 
We used the same learning rate, batch size, optimizer and amount of epochs for training the stochastic input networks. During a forward pass in training we created and used five noisy versions of each input, where the noise was drawn from a centered Gaussian distribution with variance $0.05$ or $0.1$ and the average prediction was fed into the cross-entropy loss.

\subsection{Models trained on CIFAR10}
\label{app:hyper_cifar10}
As stated in the main part, we used the wide ResNet~\citep{zagoruyko2017wide} of depth 28 and widening factor 10 provided by \url{https://github.com/meliketoy/wide-resnet.pytorch} with dropout probabilities 0.3 and 0.6 and also used the learning hyperparameters provided with the code which are: training for 200 epochs with batch size 100, stochastic gradient descent as optimizer with momentum 0.9, weight decay 5e-4 and a scheduled learning rate decreasing from an initial 0.1 for epoch 0-60 to 0.02 for 60-120 and lastly 0.004 for epochs 120-200.

\section{Additional experimental results}

\label{app:further_experiments}
In this section we present the results which were not shown in the main part due to space restrictions.

\begin{figure}[thb]
\centering
\subfloat[IM]{\includegraphics[width=0.33\textwidth]{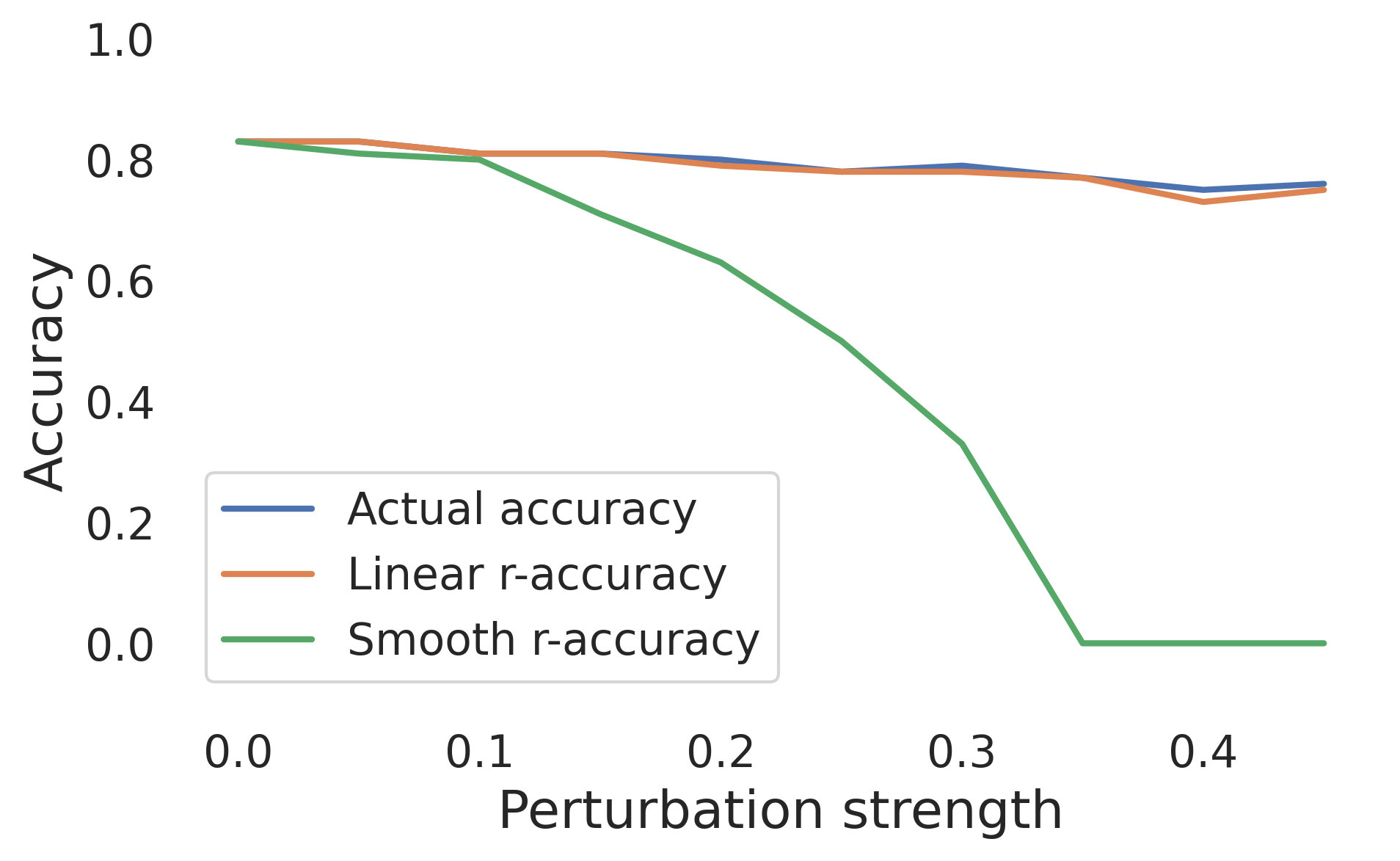} }%
\subfloat[dr 0.3]{\includegraphics[width=0.33\textwidth]{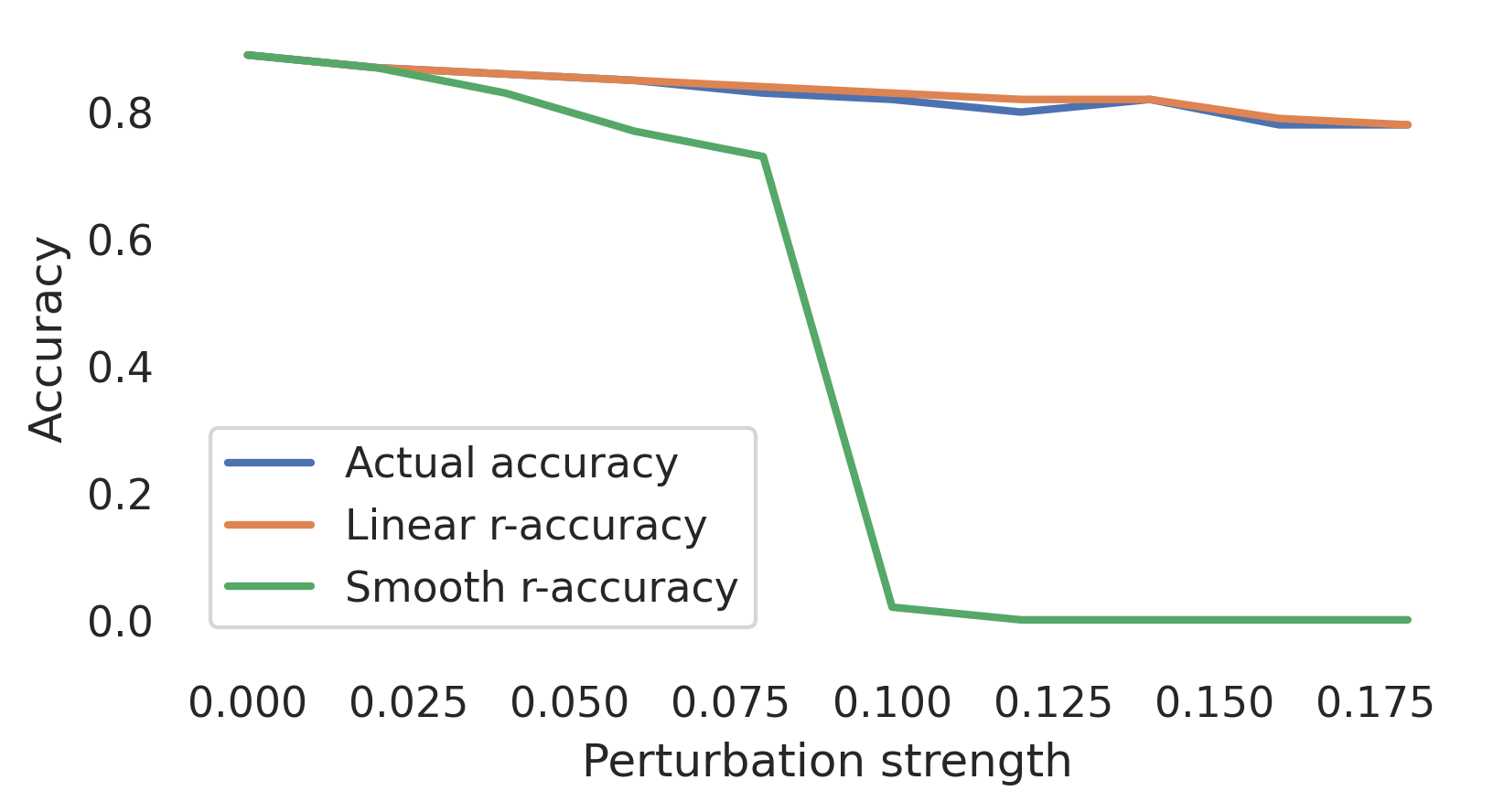} }%
\subfloat[dr 0.6]{\includegraphics[width=0.33\textwidth]{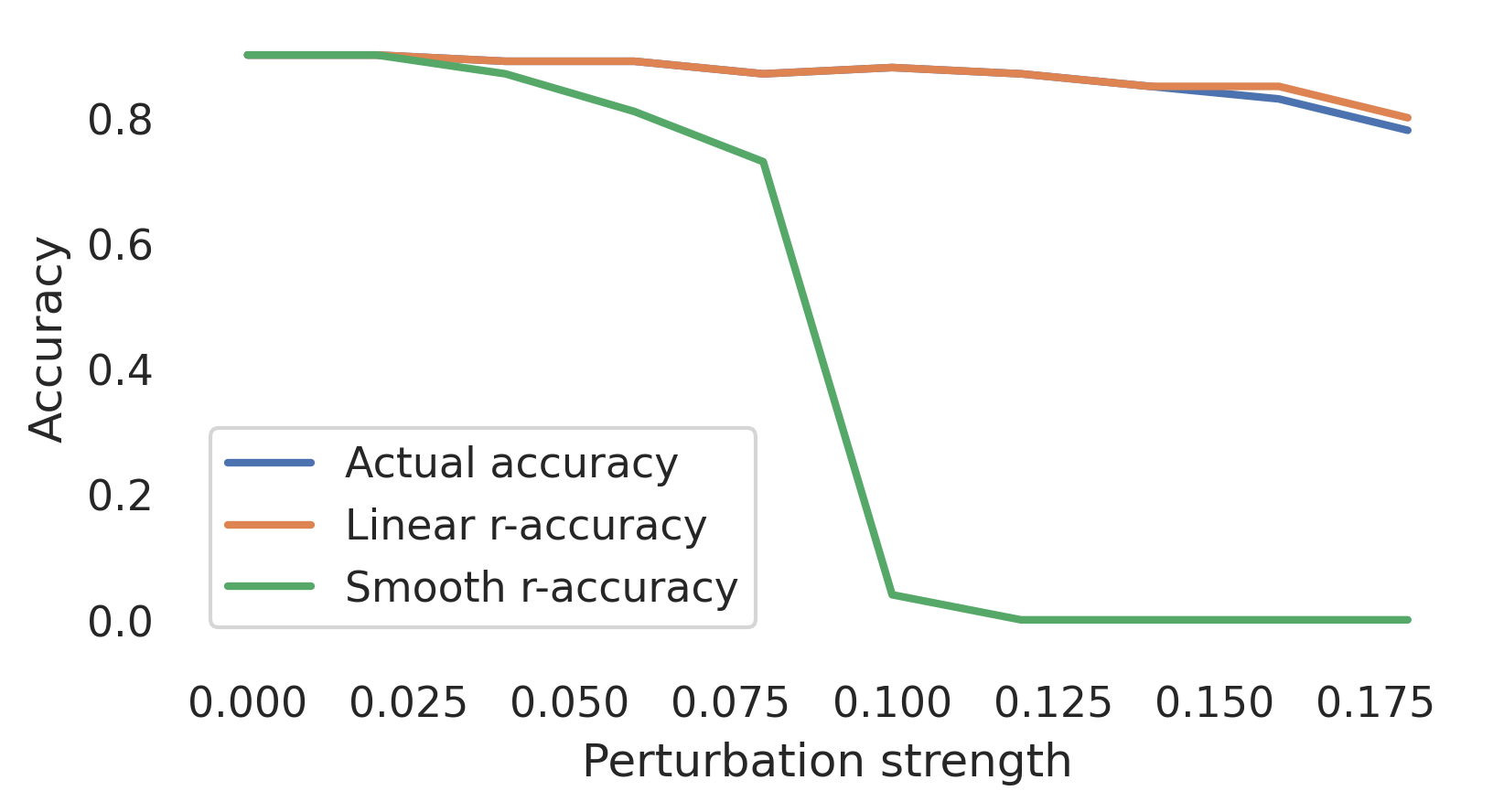} }%
\caption{Adversarial accuracy of the a) smoothed IM on FashionMNIST and b),c) smoothed ResNet with dropout probability 0.3 and 0.6 on CIFAR10 vs percentage of images for which $\min_c r^{\mathcal{I}}_c> \| \delta^{\mathcal{A}}\|_2$ (smooth) and $\min_c \tilde{r}^{\mathcal{I}}_c> \| \delta^{\mathcal{A}}\|_2$ (linear) for 100 images from the respective test sets. Attacks were conducted on the smoothed classifier with using 10 attack samples.
}
\label{fig:app_r_value_BNN_drop03}%
\end{figure} 

\subsection{Complementary experiments on accuracy of robustness conditions}
\label{app:com_accuracy_of_robustness_condition}
For completeness we attached the results on the transferability of our derived sufficient conditions to: the IM on FashionMNIST and the two ResNet with dropout probability 0.3 and 0.6 on CIFAR10. For the BNN we used the same setting as described in the main paper and derive similar results (c.f. figure~\ref{fig:app_r_value_BNN_drop03}): while the percentage of samples fulfilling the condition $\min_c r_c^{\mathcal{I}} > \| \delta^{\mathcal{A}} \|_2$ approaches zero with growing perturbation strength the percentage of samples fulfilling the condition from theorem 1 in the main paper closely matches the real adversarial accuracy in a narrow environment. For models on CIFAR10 we had to adapt the noise added for the smooth classifier to $0.01$ and reduce the amount of samples during inference to 50 such that it fits on one GPU.

\subsection{Complementary experiments on stronger attacks}
\label{app:comp_experiments_attack_sample}

We first present the results not shown in the main paper. That is, we investigate the accuracy under FGM attack with an increasing amount of samples used during the attack (c.f. figure~\ref{fig:app_acc_under_attack_BNN_drop03}). Similar to the observations in the main paper, the accuracy under attack is reduced by an increased amount of samples. However, 
we observe only a very small decrease in accuracy 
when increasing the amount of samples from
 100 and 1,000  for the BNN, from 1 to 5 or above for the SIN 0.05 and 
 when using 5 instead of
 10 or 100 samples for the attack on the ResNet trained with dropout probability 0.3.
 This observation is mirrored by the reduction of  $\cos(\alpha_c^{\mathcal{I}, \mathcal{A}})$ displayed in figure~\ref{fig:app_angle_BNN_drop03} where we observe an increase of the cosine boxplots which matches the decrease of the adversarial accuracy when taking more samples during the attack. %
 
\begin{figure}[]
\centering
\subfloat[BNN]{\includegraphics[width=0.33\textwidth]{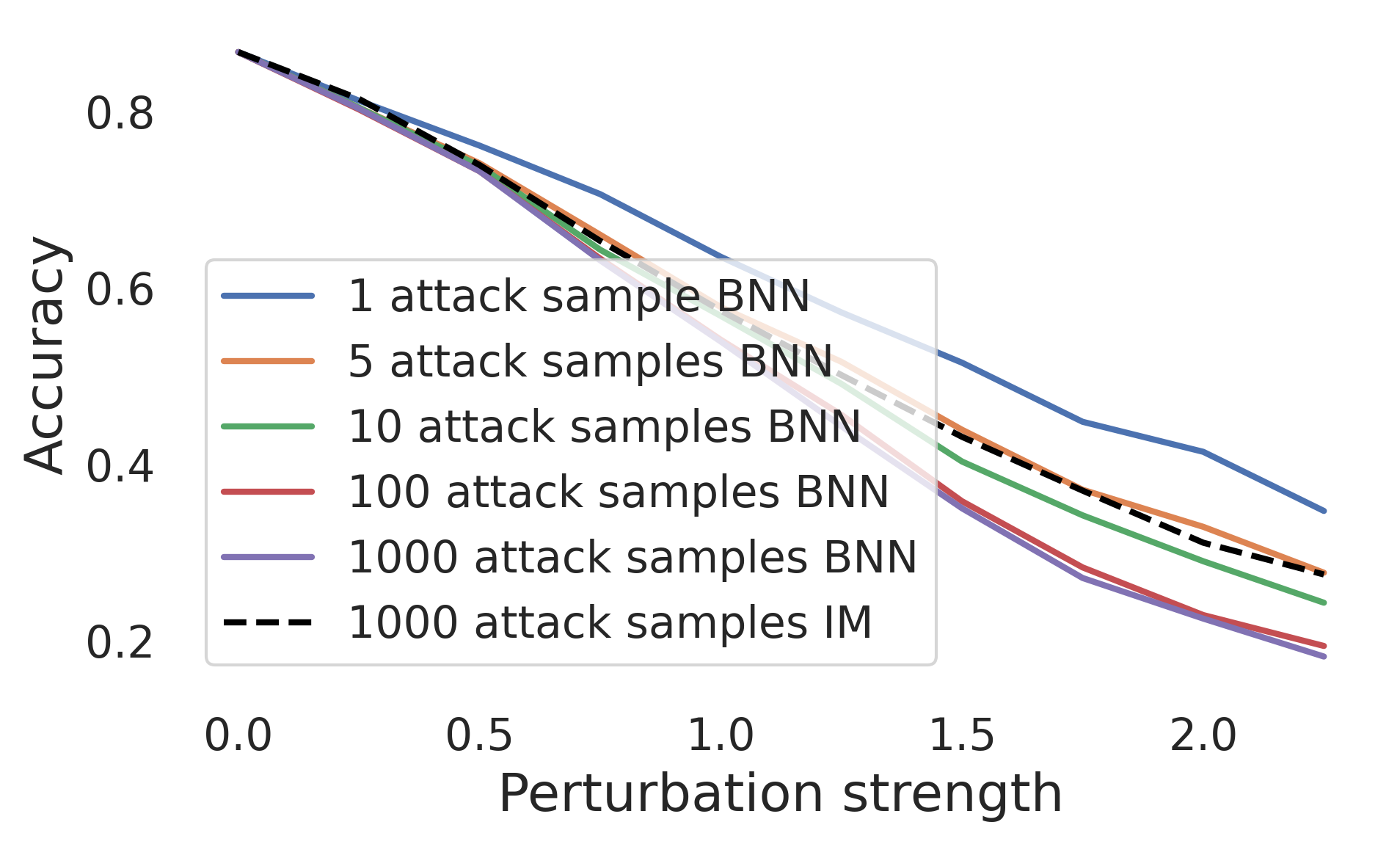} }%
\subfloat[SIN 0.05]{\includegraphics[width=0.33\textwidth]{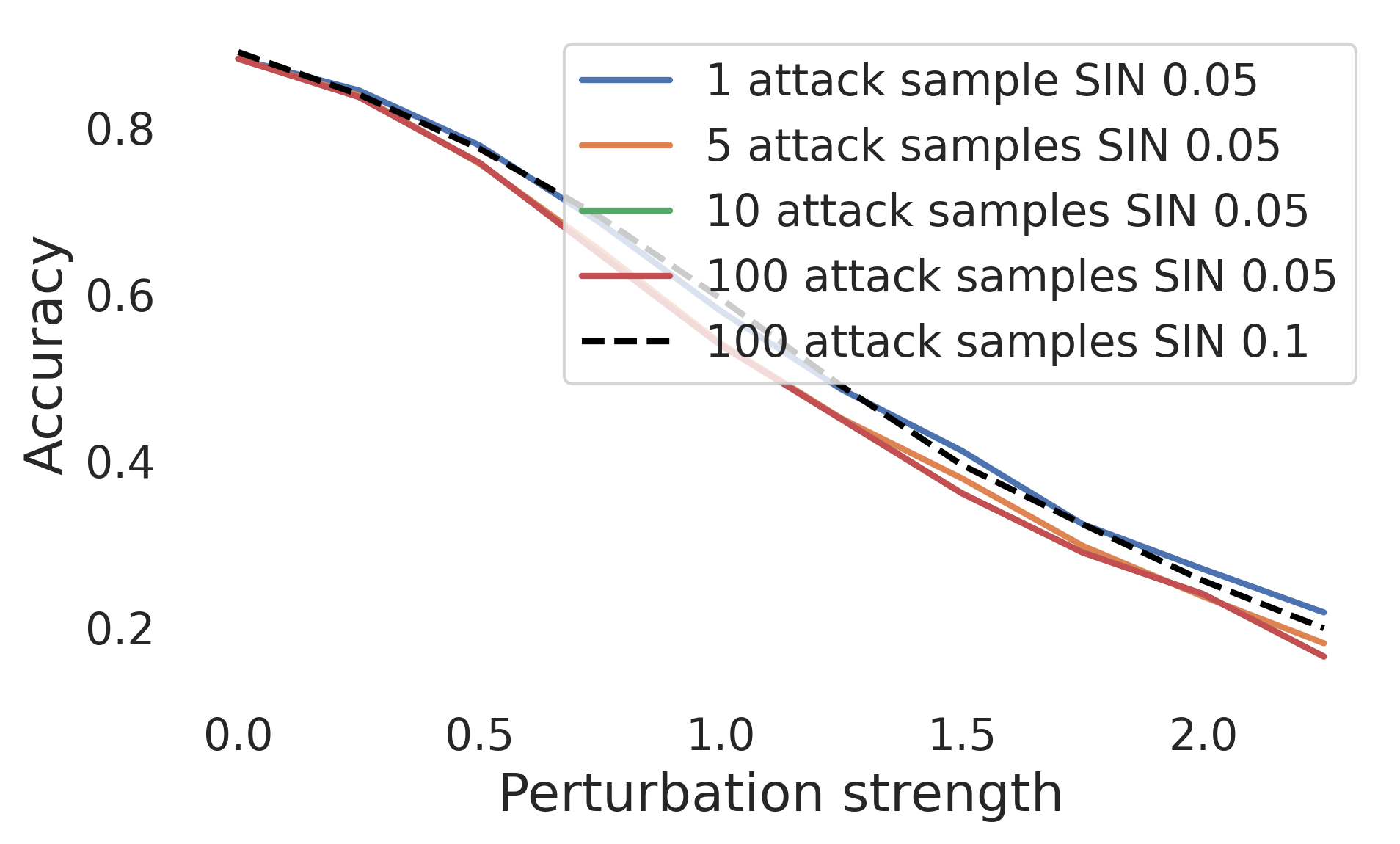} }%
\subfloat[dr 0.3]{\includegraphics[width=0.33\textwidth]{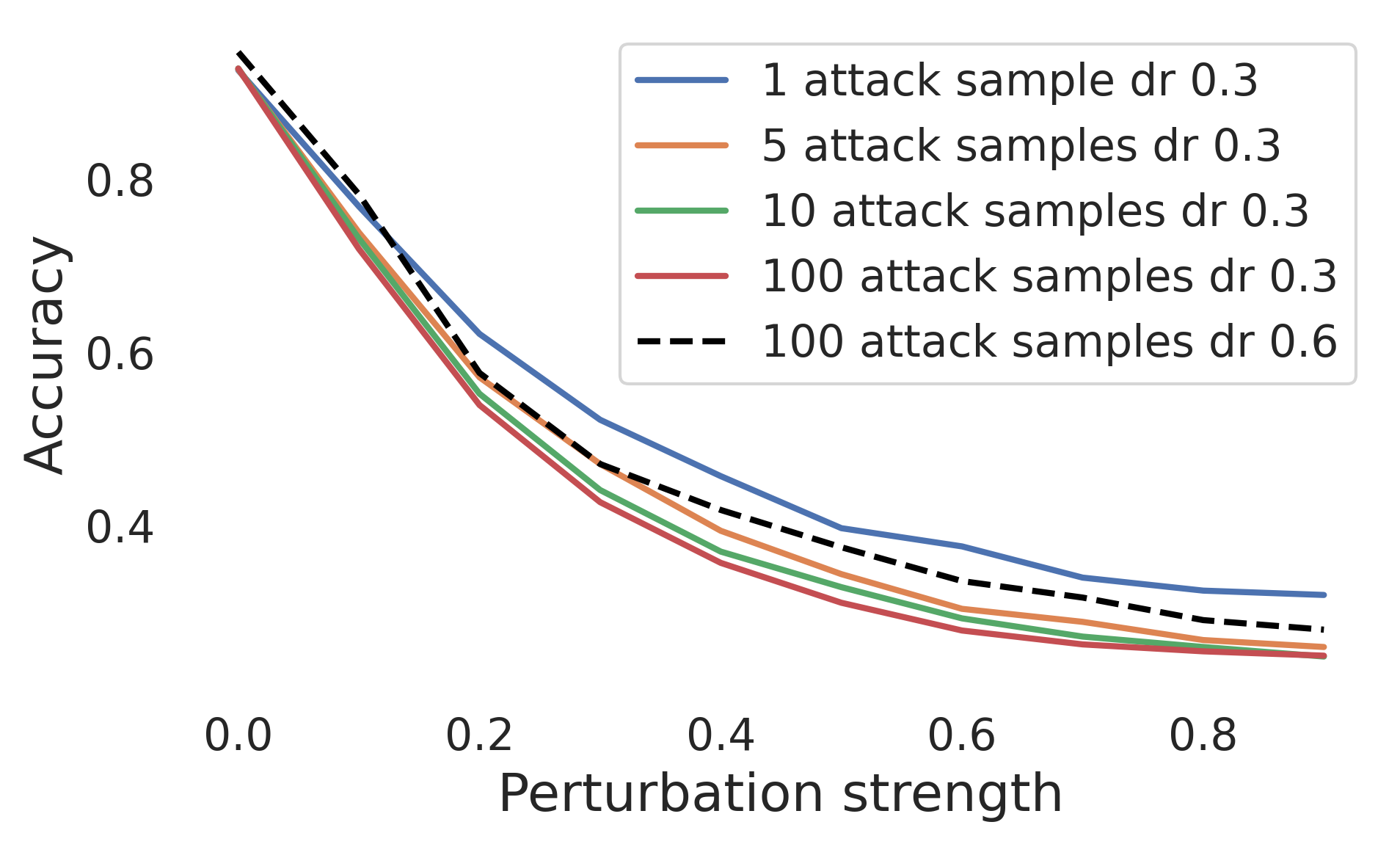} }%
\caption{Accuracy under FGM attack for a) the BNN and b) the SIN 0.05 on FashionMNIST and c) ResNet with dropout probability 0.3 on CIFAR10 for different perturbation strengths and amount of samples used for calculating the attack. During inference we used 100 samples.}%
\label{fig:app_acc_under_attack_BNN_drop03}%
\end{figure}

\begin{figure}[]
\centering
\subfloat[BNN]{\includegraphics[width=0.33\textwidth]{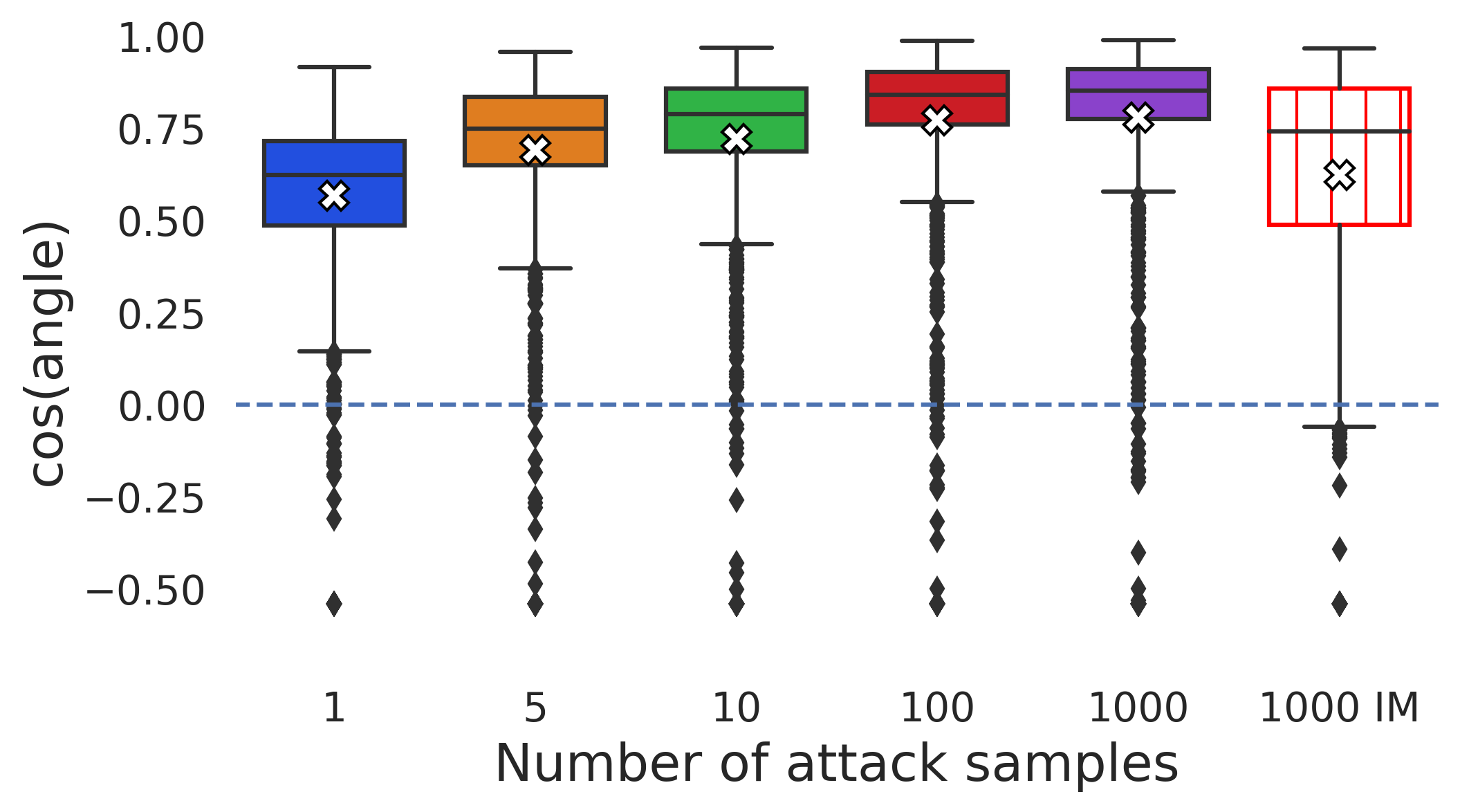}}%
\subfloat[SIN 0.05]{\includegraphics[width=0.33\textwidth]{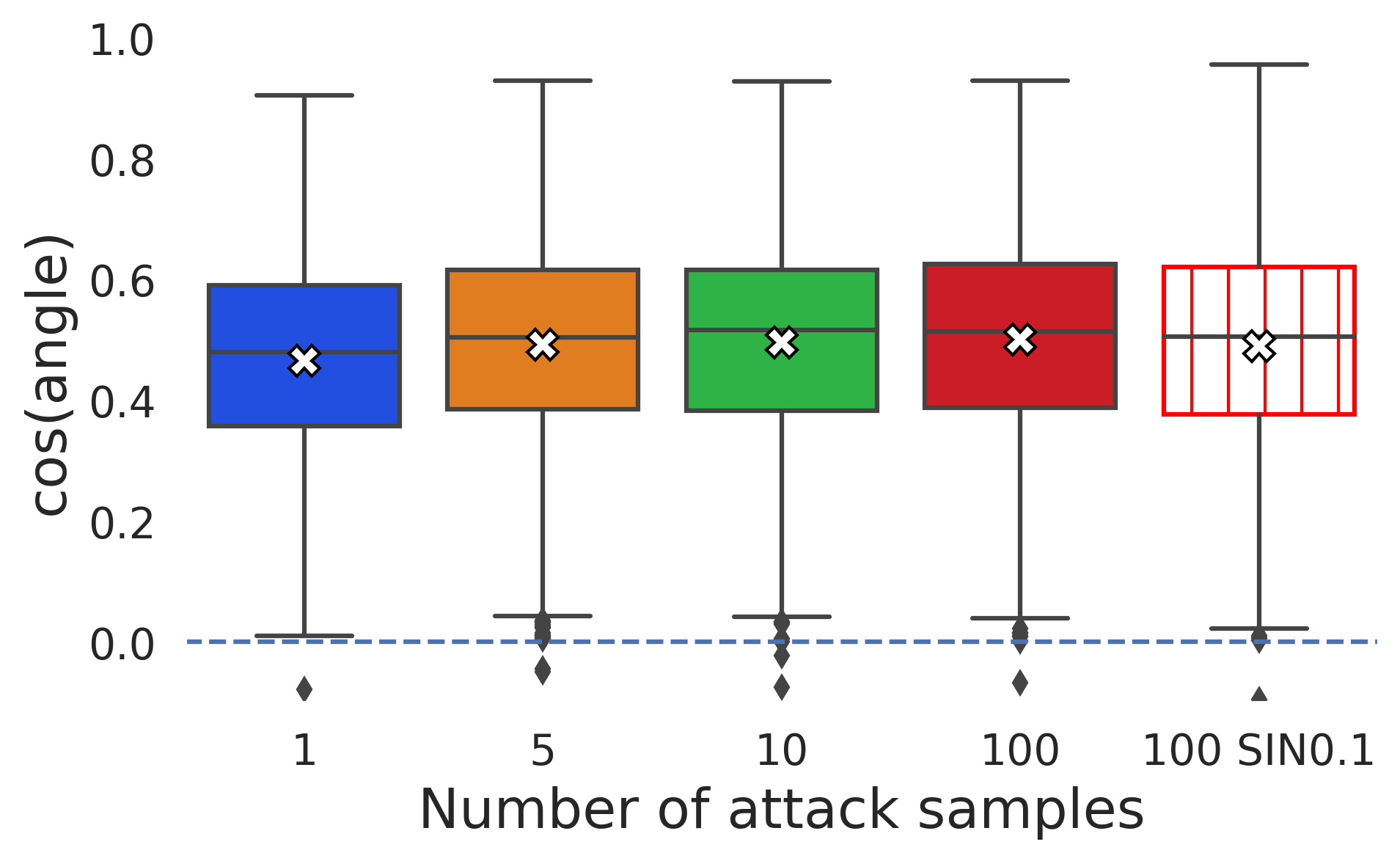}}%
\subfloat[dr 0.3]{\includegraphics[width=0.33\textwidth]{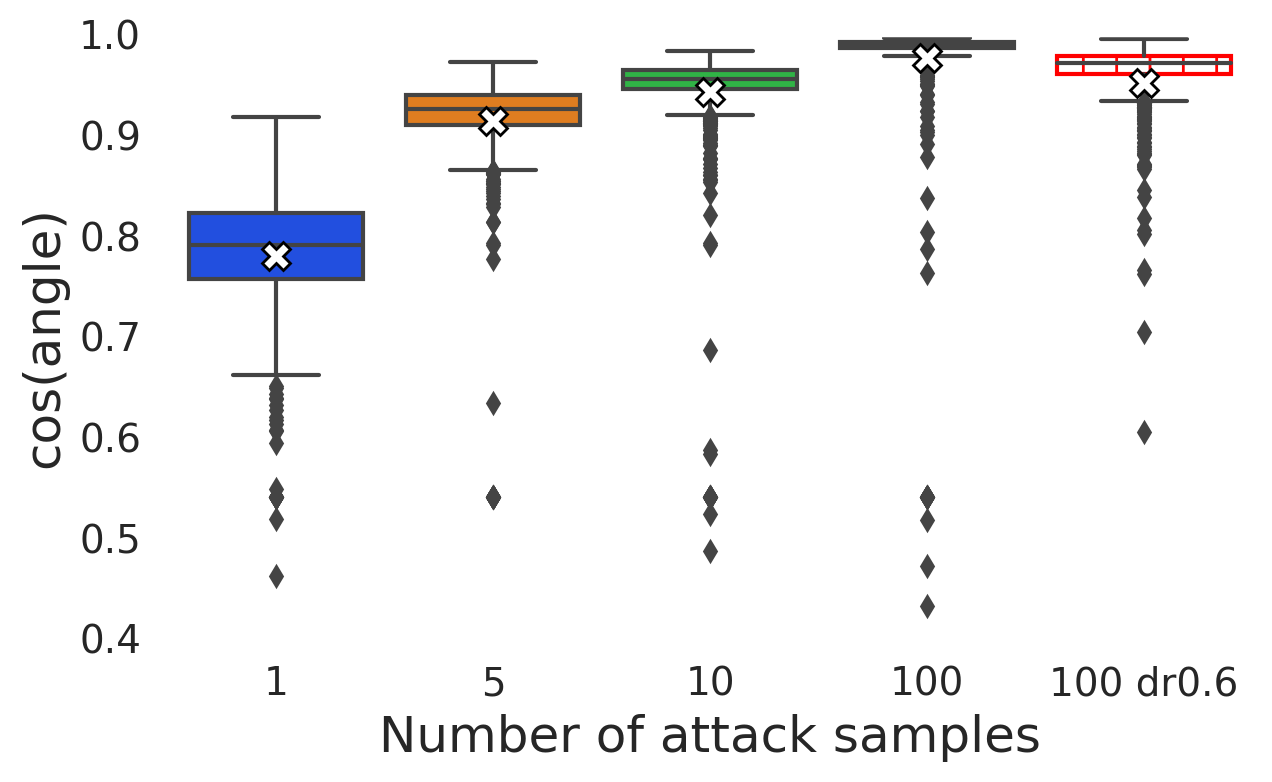}}%
\caption{Cosine of the angle for the first 1,000 test set images from the
FashionMNIST and CIFAR10 test set for an a) BNN, b) SIN 0.05 and c) ResNet trained with dropout probability 0.3 when attacked with
different amounts of samples and attack strength 1.5 and 0.3 with FGM respectively. White crosses indicate the mean value.}
\label{fig:app_angle_BNN_drop03}%
\end{figure}

\begin{figure}[h]
\centering
\subfloat[BNN]{\includegraphics[width=0.45\textwidth]{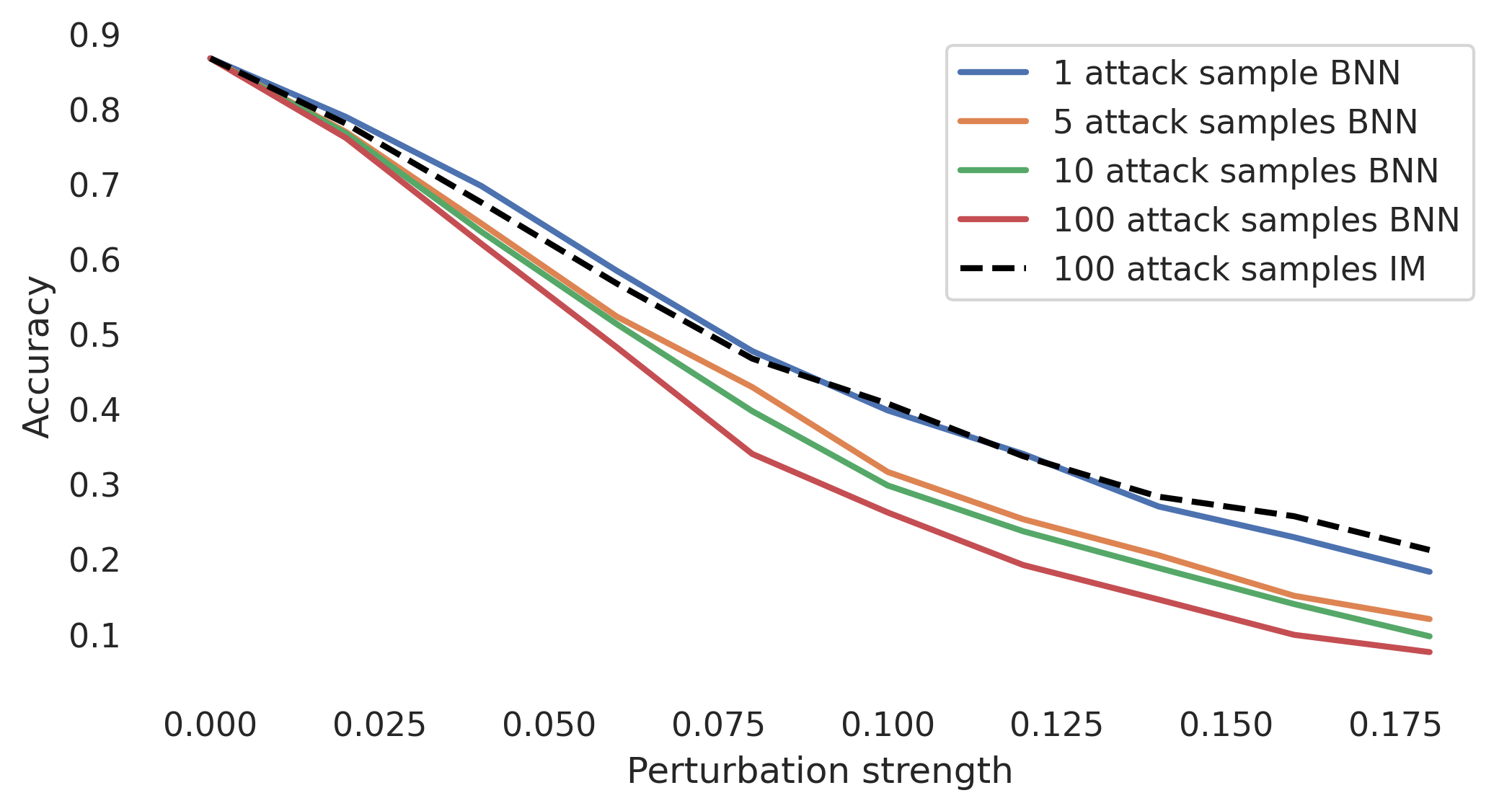} }%
\qquad
\subfloat[IM]{\includegraphics[width=0.45\textwidth]{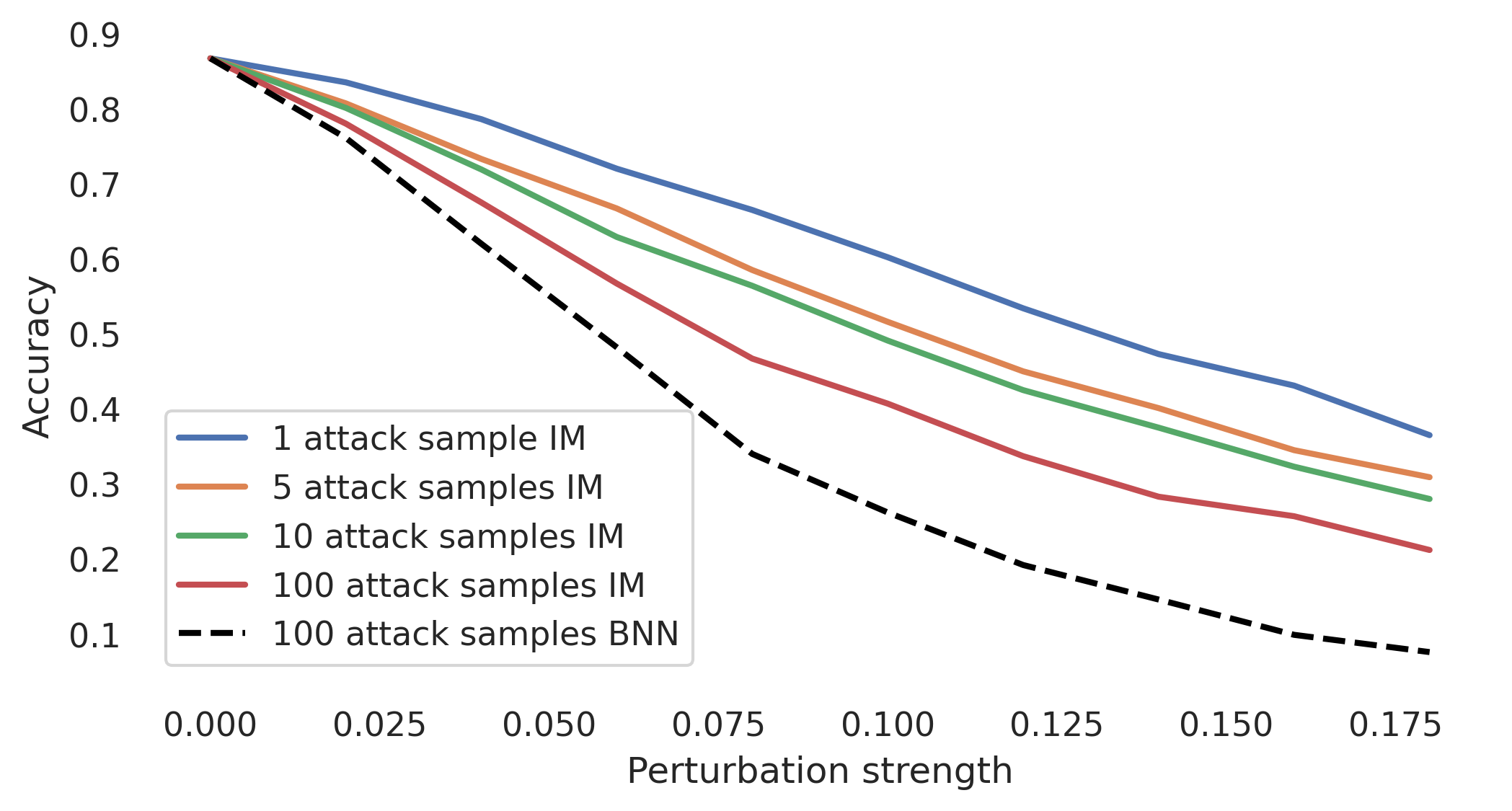} }%
\qquad
\subfloat[SIN 0.05]{\includegraphics[width=0.45\textwidth]{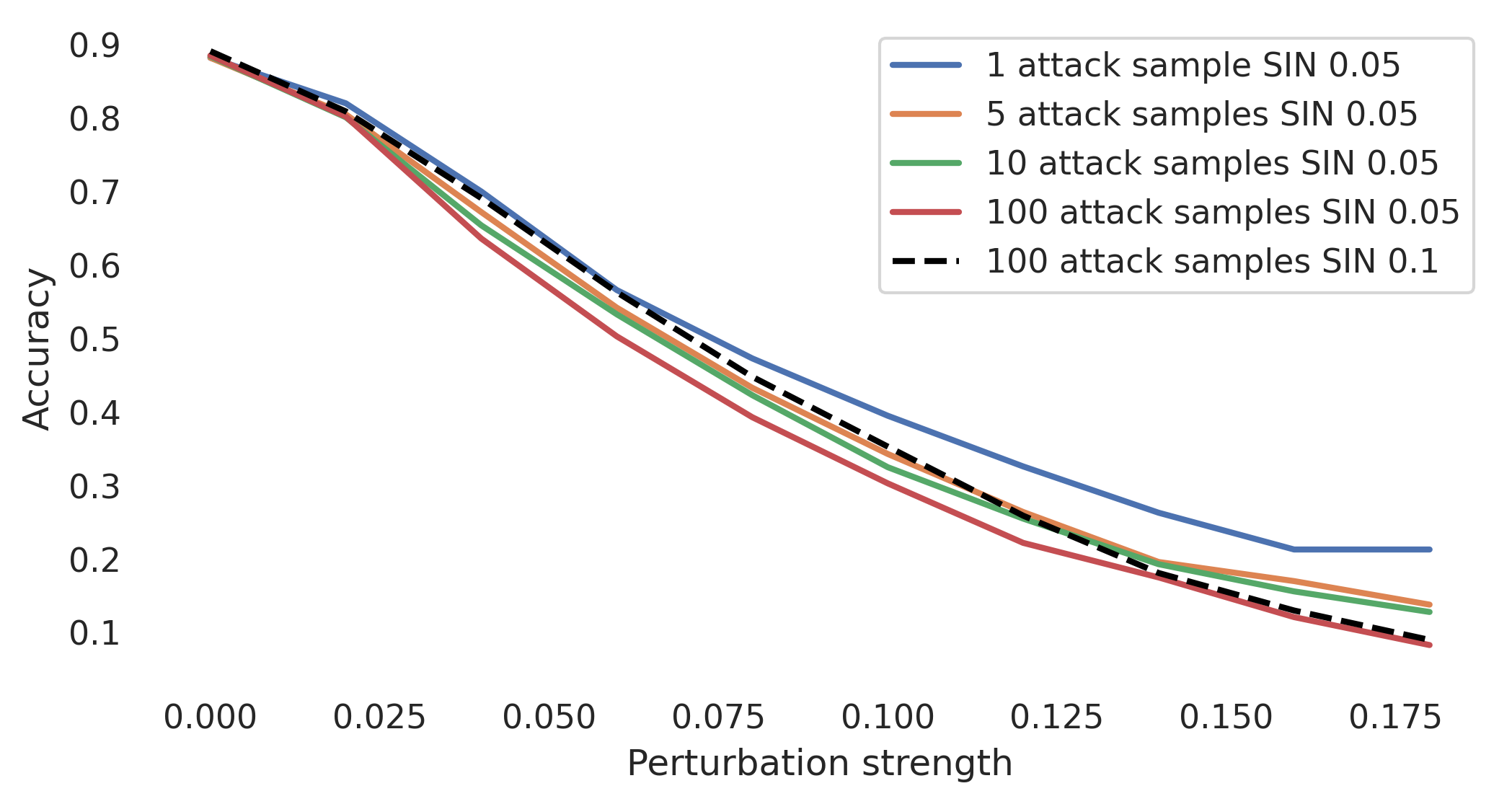} }%
\qquad
\subfloat[SIN 0.1]{\includegraphics[width=0.45\textwidth]{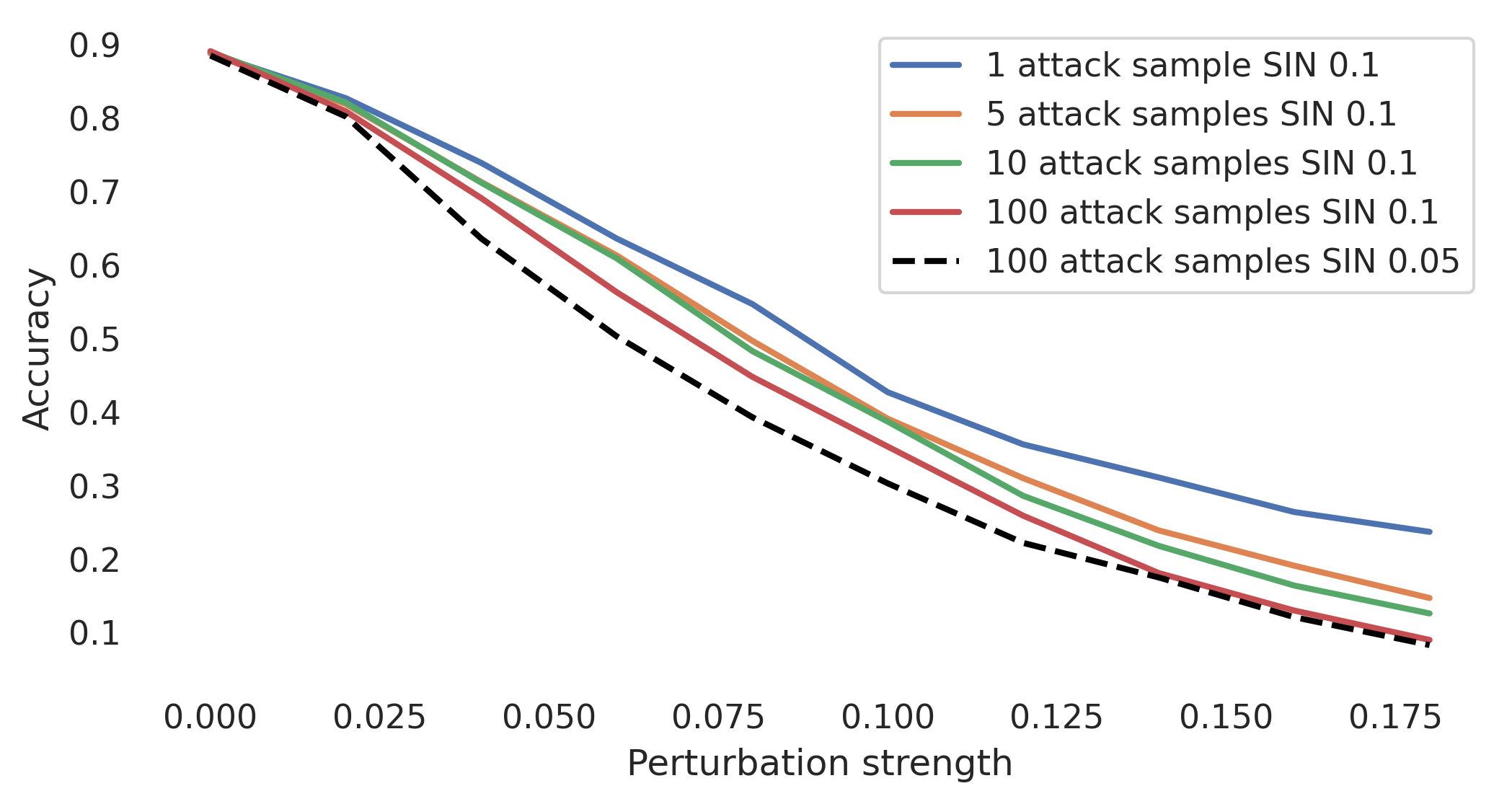} }%
\caption{Accuracy under FGSM attack under $\ell_\infty$-norm constraint for a) the BNN, b) the IM, c) SIN 0.05 and d) SIN 0.1 on the first 1,000 test set images from FashionMNIST for different perturbation strengths and amount of samples used for calculating the attack. Predictions during inference are based on 100 samples.}%
\label{fig:app_fgsm_l_infty_fmnist}%
\end{figure}

\begin{figure}[h]
\centering
\subfloat[droprate 0.3]{\includegraphics[width=0.45\textwidth]{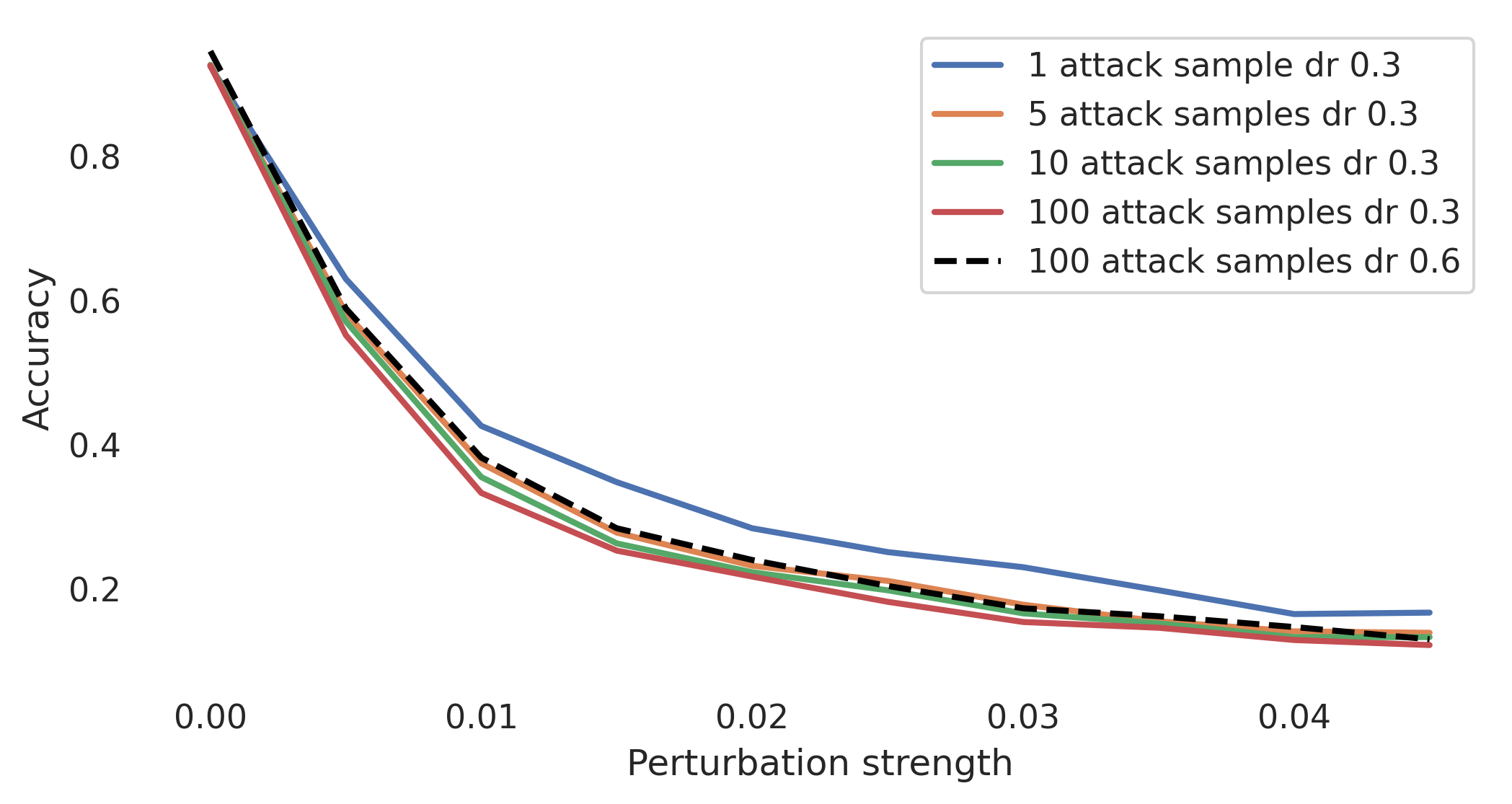} }%
\qquad
\subfloat[droprate 0.6]{\includegraphics[width=0.45\textwidth]{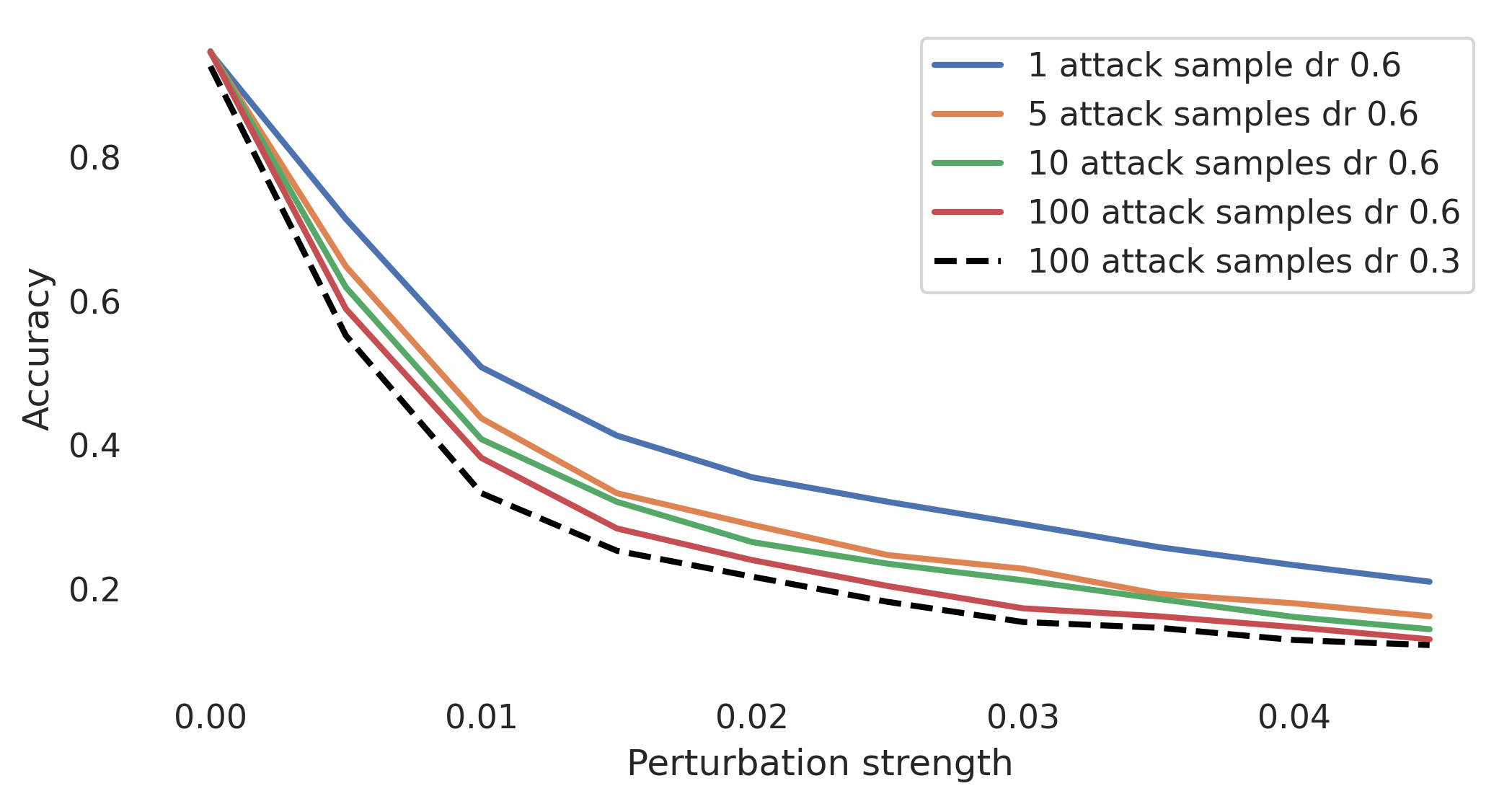} }%
\caption{Accuracy under FGSM attack under $\ell_\infty$-norm constraint for the ResNet model with a dropout probability a) of 0.3 and b) of 0.6  on the first 1,000 test set images from  CIFAR10 for different perturbation strengths and amount of samples used for calculating the attack. Predictions during inference are based on 100 samples.}%
\label{fig:app_fgsm_l_infty_cifar}%
\end{figure}

\subsubsection{Attacks with FGSM ($\ell_\infty$- norm)}
All adversarial examples in the main part of the paper were based on an $\ell_2$-norm constraint, which we chose for the nice geometric distance interpretation. However, the first proposed attack scheme~\citep{Goodfellow_fgsm} was based on $\ell_\infty$-norm, which we test in the following. In figure~\ref{fig:app_fgsm_l_infty_fmnist} and~\ref{fig:app_fgsm_l_infty_cifar}  we see the respective results on the different data sets. For all models the robustness is decreased with multiple samples and models with higher prediction variance also have a higher accuracy under this attack. Note, that the values of the perturbation strength are not comparable to the values under $\ell_2$- norm constraint, since $\|x \|_\infty \leq \| x \|_2$.

\newpage
\subsubsection{Attacks with PGD}
\label{app:pgd_results}
Projected gradient descent~\citep{madry2018towards} is a strong iterative attack, where multiple small steps of size $\nu$ of fast gradient method are applied. Specifically, we used the same $\ell_2$-norm length constraint on $\eta$ as in the experiments of the main part of the paper but chose step size $\nu = \eta/50$ and 100 iterations. Note that at each iteration a new network is sampled such that for an attack based on 1 samples, 100 different attack networks were seen, for an attack based on 5 samples 500 different networks and so on. In figure~\ref{fig:app_pgd_fmnist} and~\ref{fig:app_pgd_cifar} we see that the overall accuracy is decreased compared to the results for FGM, %
but still, IM has a higher accuracy under attack than the BNN and so does the ResNet with a higher dropout probability. Further, the attacks get stronger with taking more samples, so the general observations made in the main paper also hold for strong attacks and are not due to a sub-optimal attack. 

\begin{figure}[h!]
\centering
\subfloat[BNN]{\includegraphics[width=0.45\textwidth]{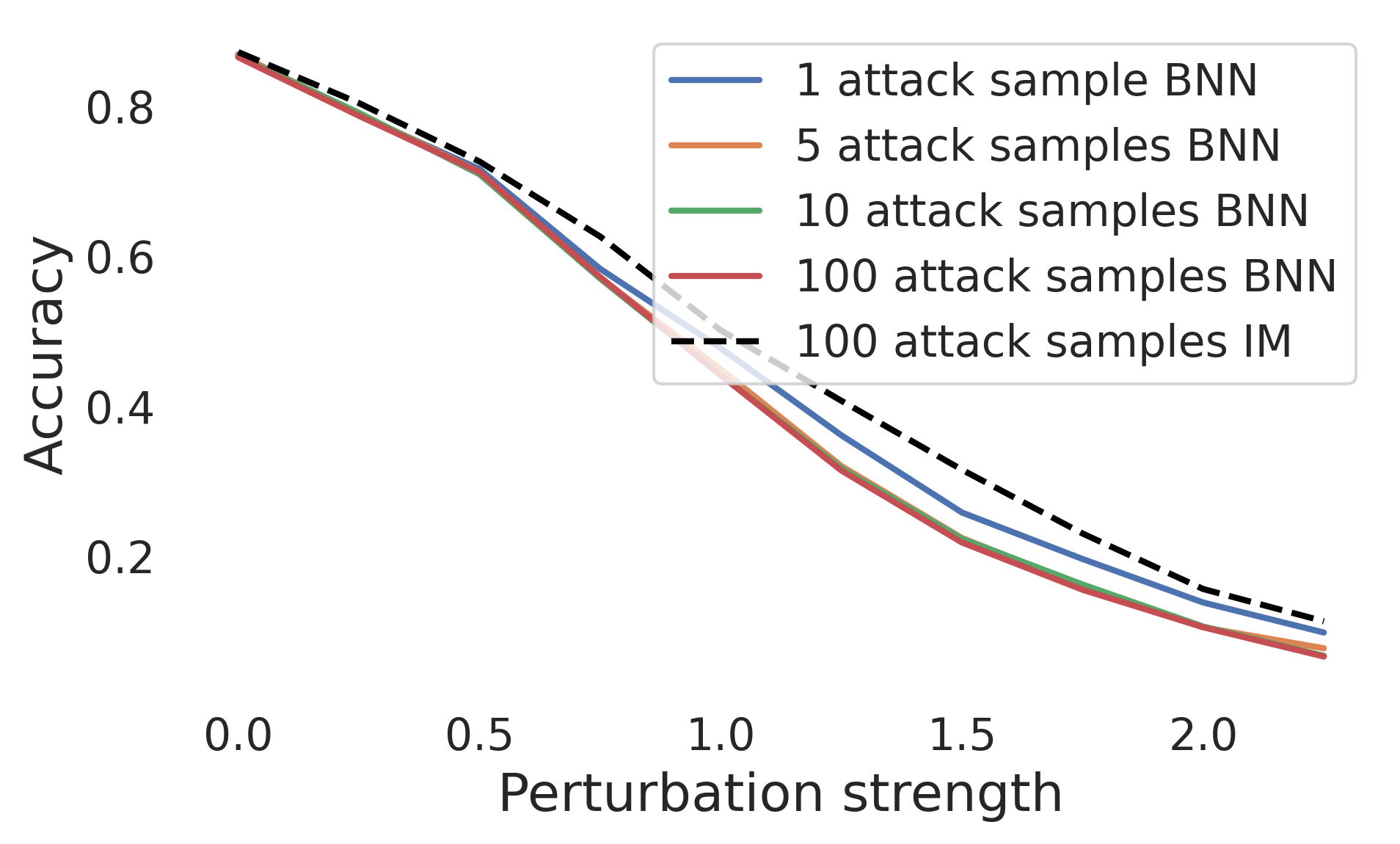} }%
\qquad
\subfloat[IM]{\includegraphics[width=0.45\textwidth]{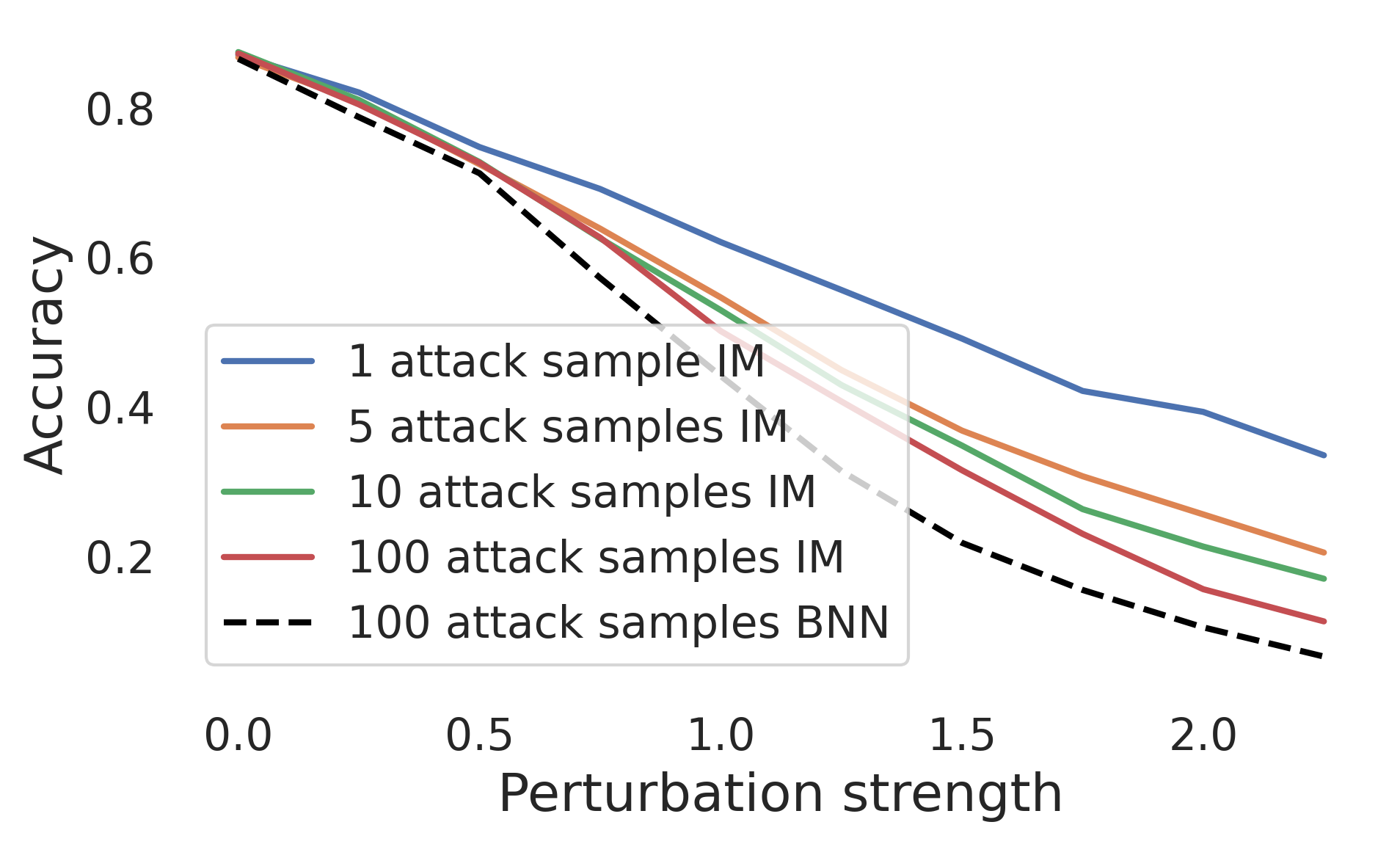} }%
\caption{Accuracy under PGD attack with 100 iterations for a) the BNN  and b) the IM on the first 1,000 test set images from FashionMNIST for different perturbation strengths and amount of samples used for calculating the attack. }%
\label{fig:app_pgd_fmnist}%
\end{figure}

\begin{figure}[htb]
\centering
\subfloat[droprate 0.3]{\includegraphics[width=0.45\textwidth]{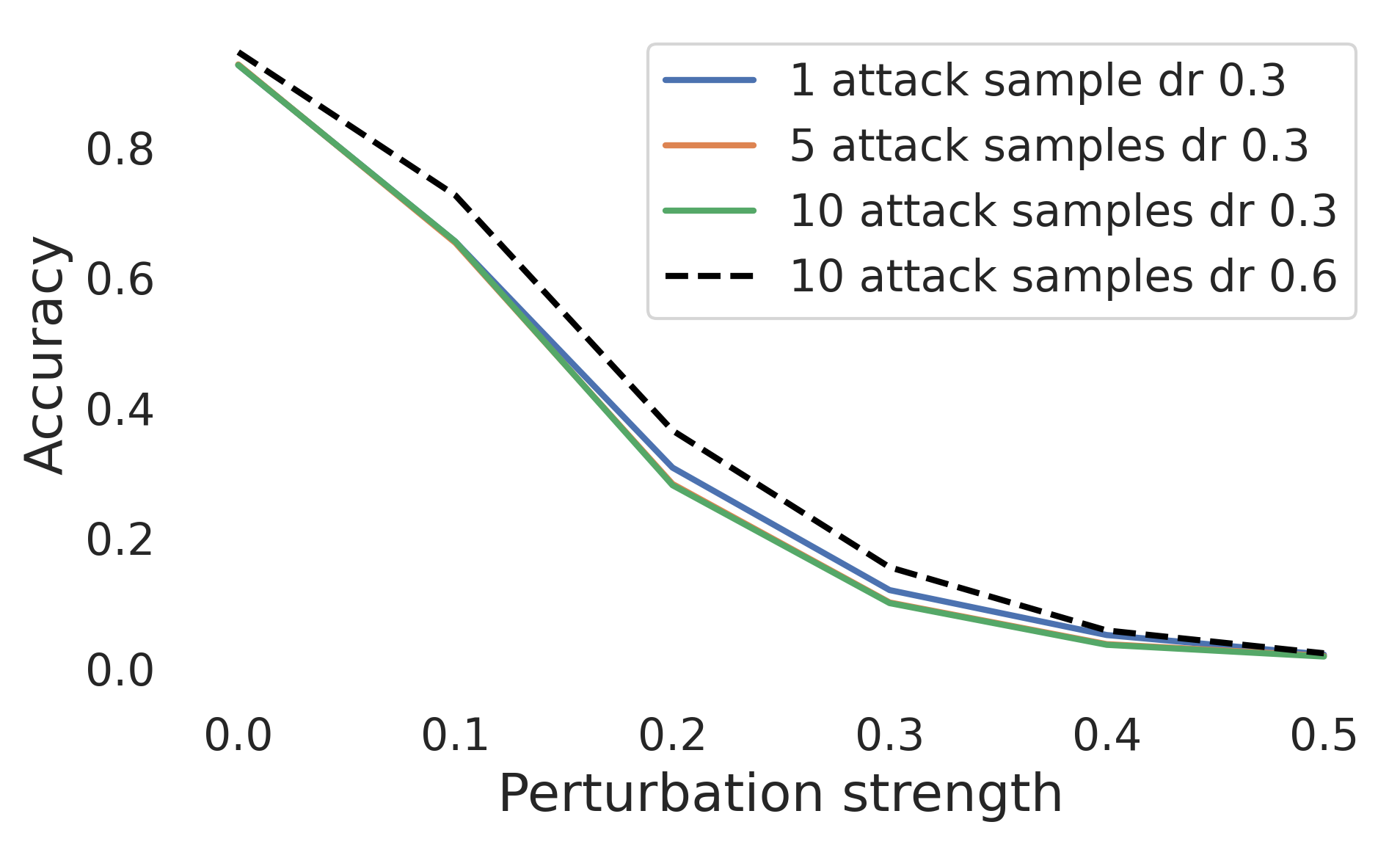} }%
\qquad
\subfloat[droprate 0.6]{\includegraphics[width=0.45\textwidth]{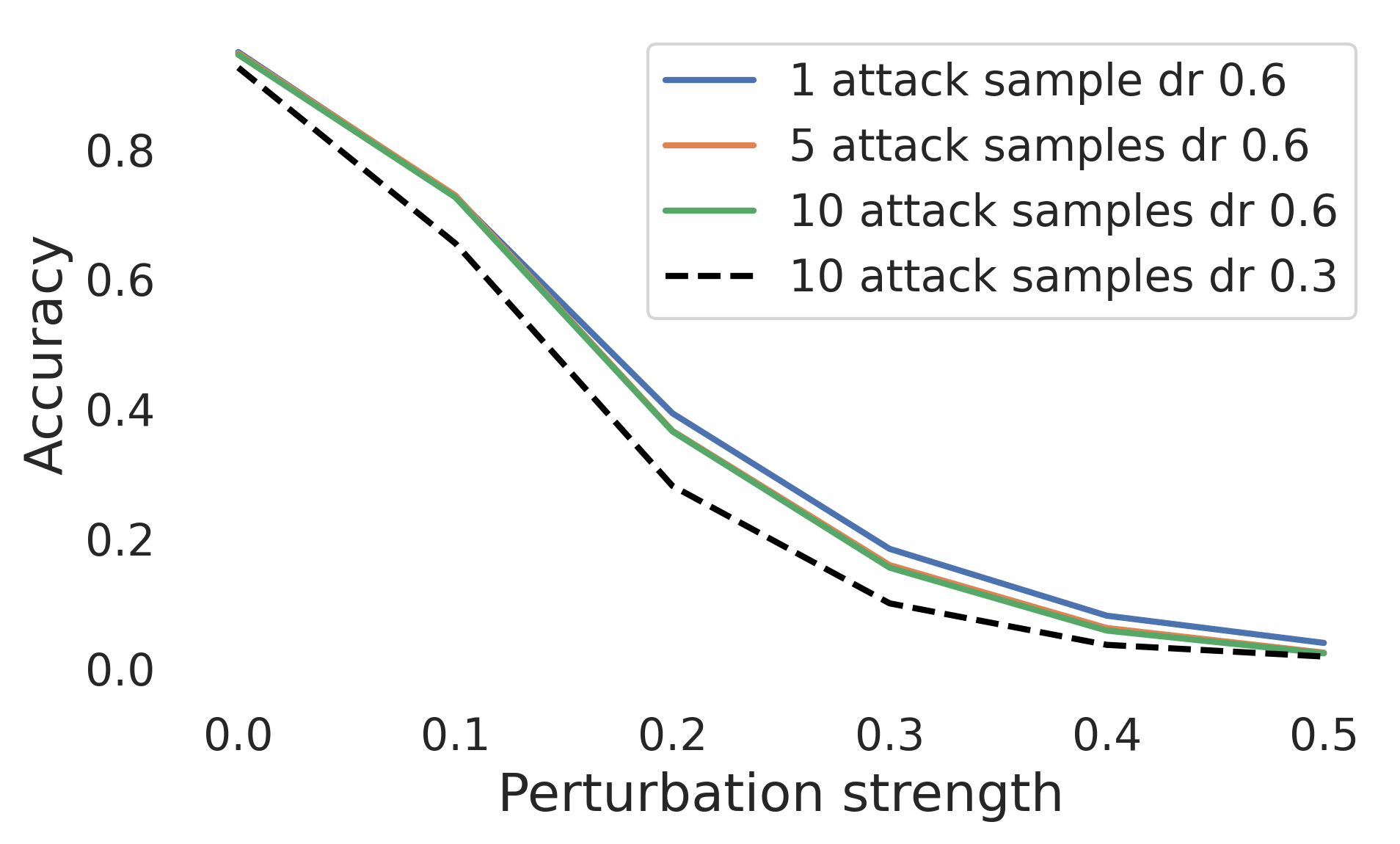} }%
\caption{Accuracy under PGD attack with 100 iterations for the ResNet with  a dropout probability of a) 0.3  and b) of 0.6 on CIFAR10 for different perturbation strengths and amount of samples used for calculating the attack.}%
\label{fig:app_pgd_cifar}%
\end{figure}

\subsection{Discussing the impact of extreme prediction values}
\label{app:comp_extrem_pred_and_length}
Attacks based on softmax predictions can be unsuccessful for deterministic and stochastic neural networks alike when encountering overly confident predictions. That is, predictions where the softmax output is equal to 1, since these lead to zero gradients.
A practical solution to circumvent this problem is to calculate the gradient based on the logits~\citep{carlini2017evaluating}. 
This approach is feasible for classifier whose predictions do not depend on the scaling of the output, that is, outputs which are equally expressive in both intervals $[0,1]$ and $[-\infty, \infty]$. In Bayesian neural networks, where $f_c(x, \Theta)$ is per definition a probability the shortcut over taking the gradient over logits leads to distorted gradients.
For completeness, we nevertheless look at the performance of an attack based on the logits for the two different models trained on FashionMNIST with softmax outputs: IM and BNN.
We conducted an adversarial FGM attack based on 100 samples, but instead of using the cross-entropy loss we used the Carlini-Wagner (CW) loss~\citep{carlini2017evaluating} on the averaged logits, given by:
\begin{equation*}
    CW(x, \Theta^{\mathcal{A}}) = \max\left(\max_{i \neq t}(  Z(x, \Theta^{\mathcal{A}})_i) - Z(x, \Theta^{\mathcal{A}})_t, 0\right) \enspace,
\end{equation*}
where $ Z(x, \Theta^{\mathcal{A}})_t = \frac{1}{S^A} \sum_{s=1}^{S^A}  Z(x, \theta_s)_t$ is an arbitrary averaged logit of an output node for input $x$.
In figure~\ref{fig:app_carlini_wagner_loss} it is shown, that the attack with the CW loss on the infinite mixture model improves upon the original attack scheme (FGM), whereas it did not improve the attack's success for the BNN. We additionally conducted an attack based on the logit margin loss $\mathcal{L}(x+\delta^\mathcal{A}, y) = - (\min_{c \neq y} f_{y-c}(x) )$ equivalent to the attack conducted on the SIN, but we found that it performs similar to the CW loss.

\begin{figure}[thb]
\centering
\subfloat[IM]{\includegraphics[width=0.45\textwidth]{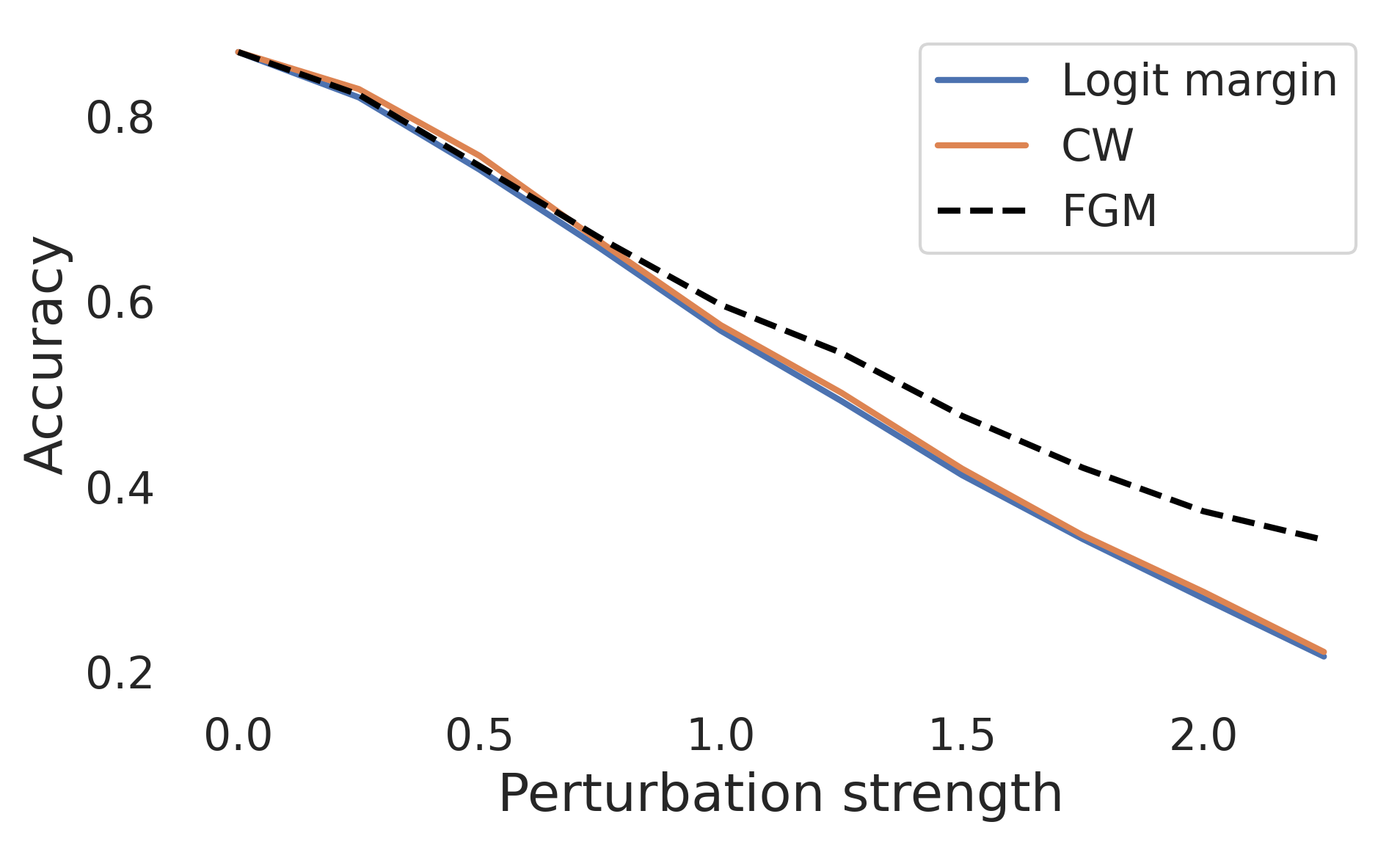} }%
\subfloat[BNN]{\includegraphics[width=0.45\textwidth]{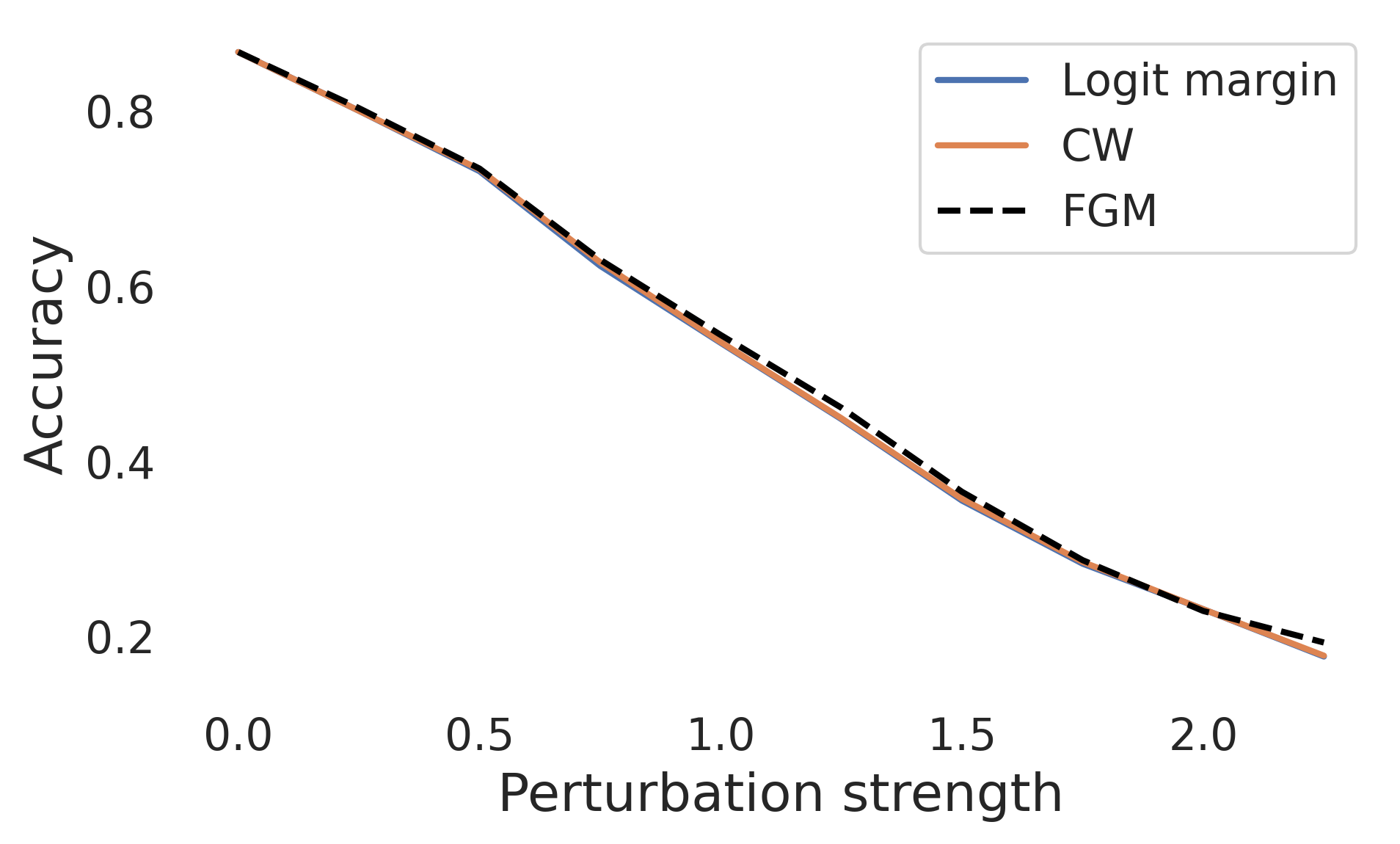} }%
\caption{Accuracy under attack for the first 1,000 test set images from FashionMNIST with varying perturbation strengths and different attack objectives. Each attack was calculated based on 100 samples and $\ell_2$-constraint for models trained on FashionMNIST.}
\label{fig:app_carlini_wagner_loss}
\end{figure}

\subsection{Complementary experiments on robustness in dependence of the amount of samples
used during inference}
\label{app:comp_inference}
In the main part of the paper we argued why, surprisingly, the amount of samples during inference does not influence the robustness, even though we see in figure~\ref{fig:app_angle_inference} that less sample lead to the smallest values for $\cos(\alpha_c^{\mathcal{I}, \mathcal{A}})$. 
As stated in the main part, the increased gradient norm when using only few samples seems to compensate the assumed benefits with regard to the angle for using few samples (c.f. figure~\ref{fig:app_norm_inference}). This can also be seen in table 2 and figure~6 from the main paper, where a (negative) effect on the robustness can only be observed for one inference sample, which also leads to the worst test set accuracy. 
The benign prediction margins are also hardly effected by the increased number of samples during prediction (c.f.~figure~\ref{fig:app_margin_inference}).

\begin{figure}[h]
\centering
\subfloat[IM]{\includegraphics[width=0.3\textwidth]{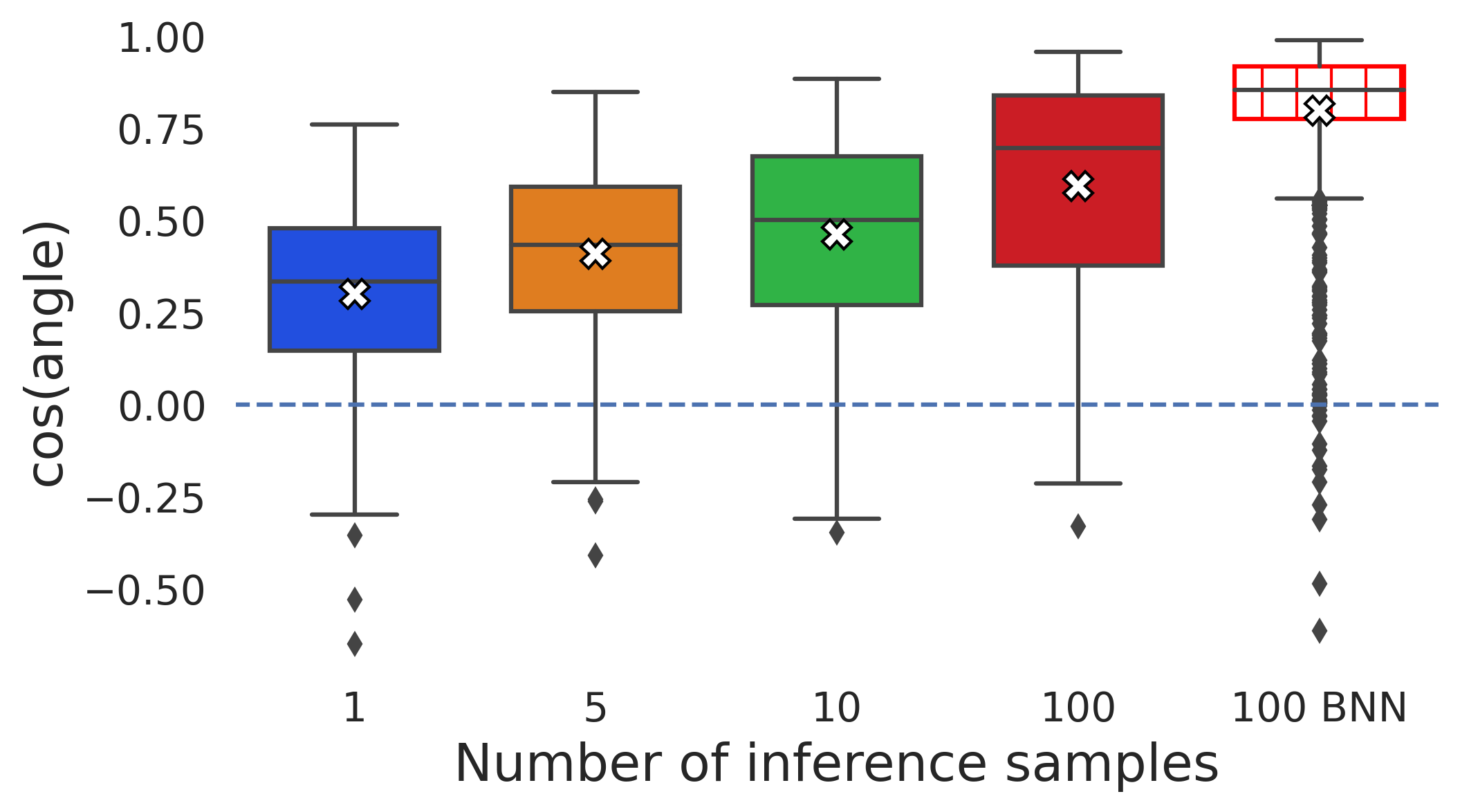} }%
\subfloat[SIN 0.1]{\includegraphics[width=0.3\textwidth]{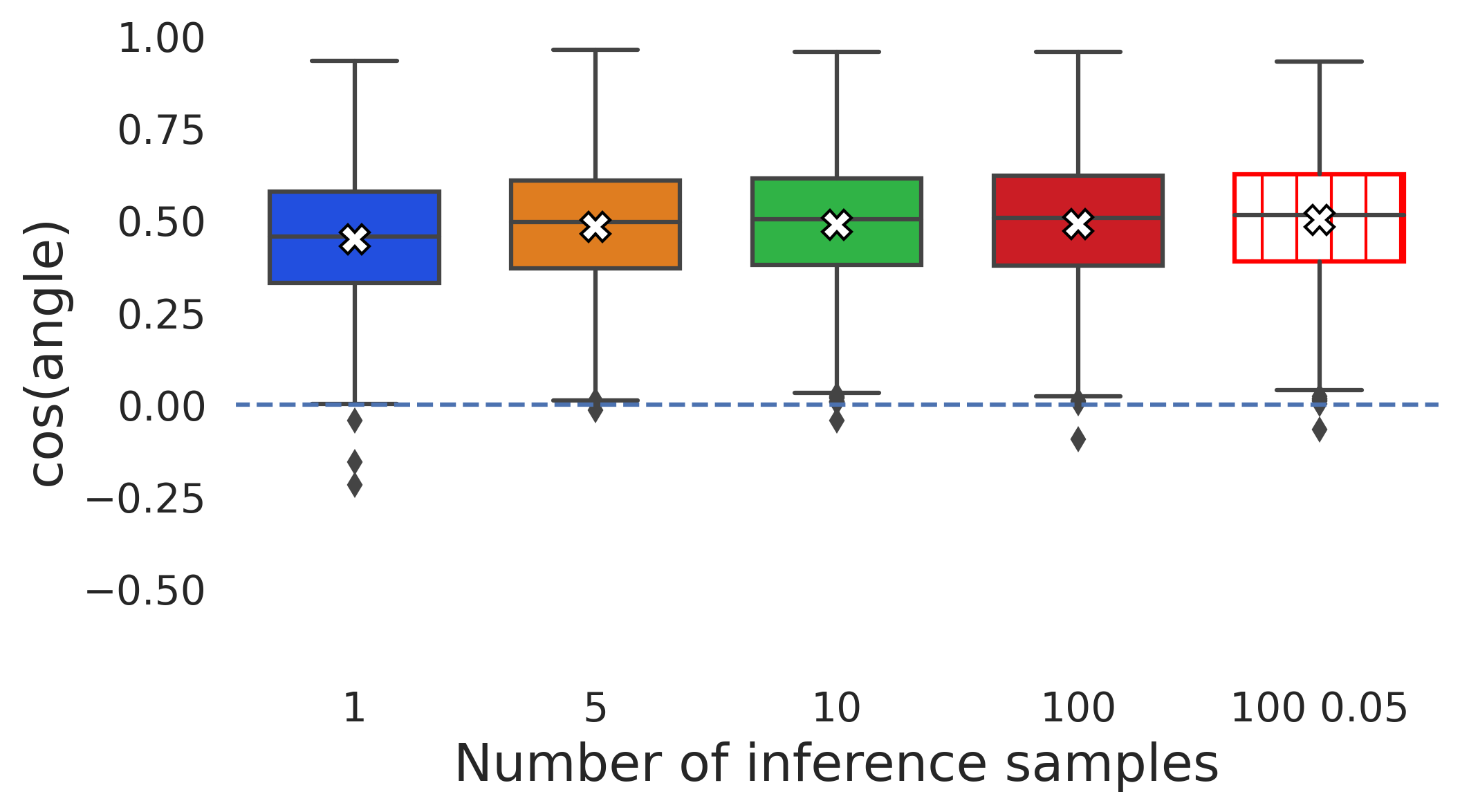} }%
\subfloat[droprate 0.6]{\includegraphics[width=0.3\textwidth]{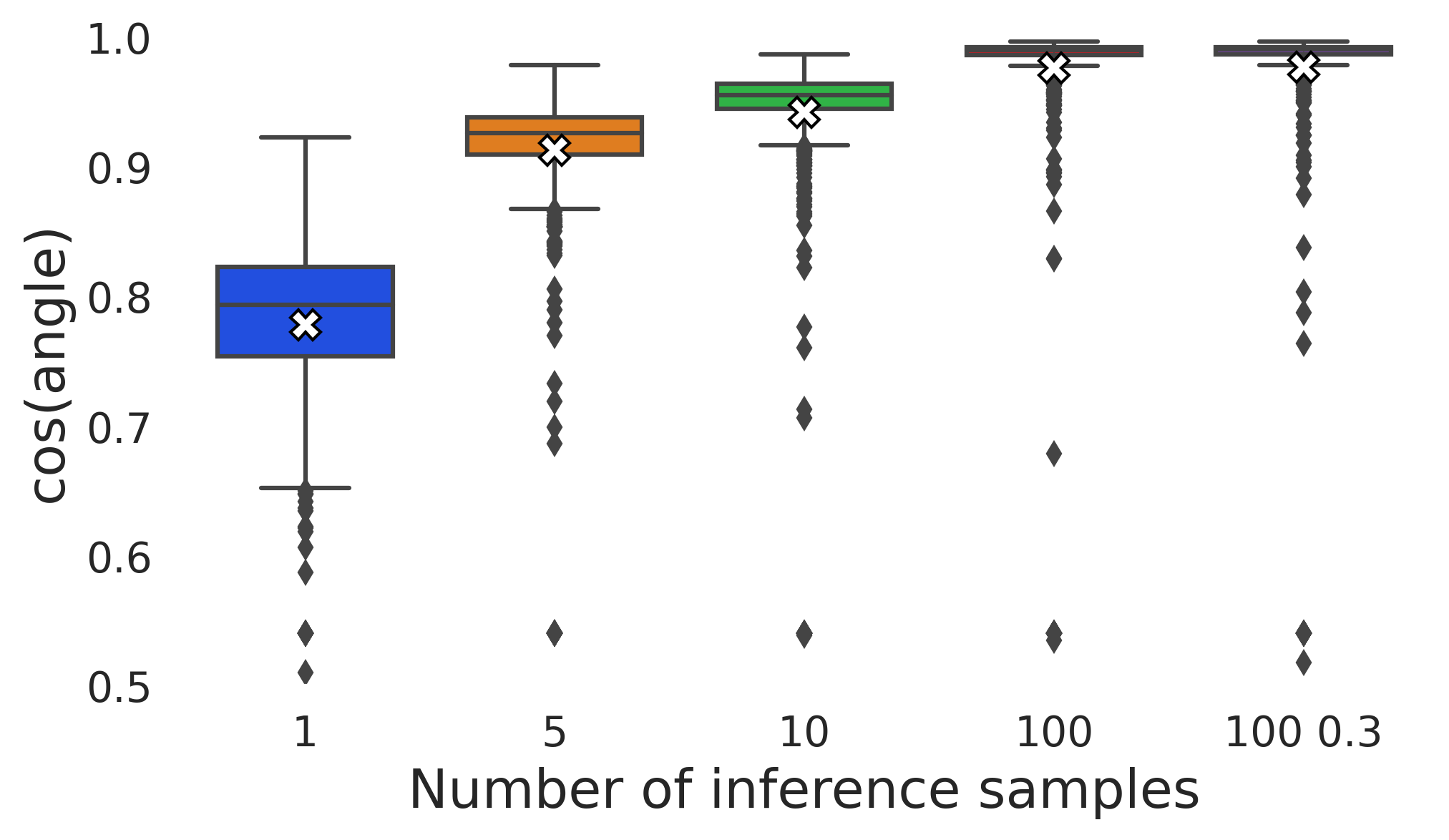} }
\qquad
\subfloat[BNN]{\includegraphics[width=0.3\textwidth]{images/cos_inference_BNN_1.png} }%
\subfloat[SIN 0.05]{\includegraphics[width=0.3\textwidth]{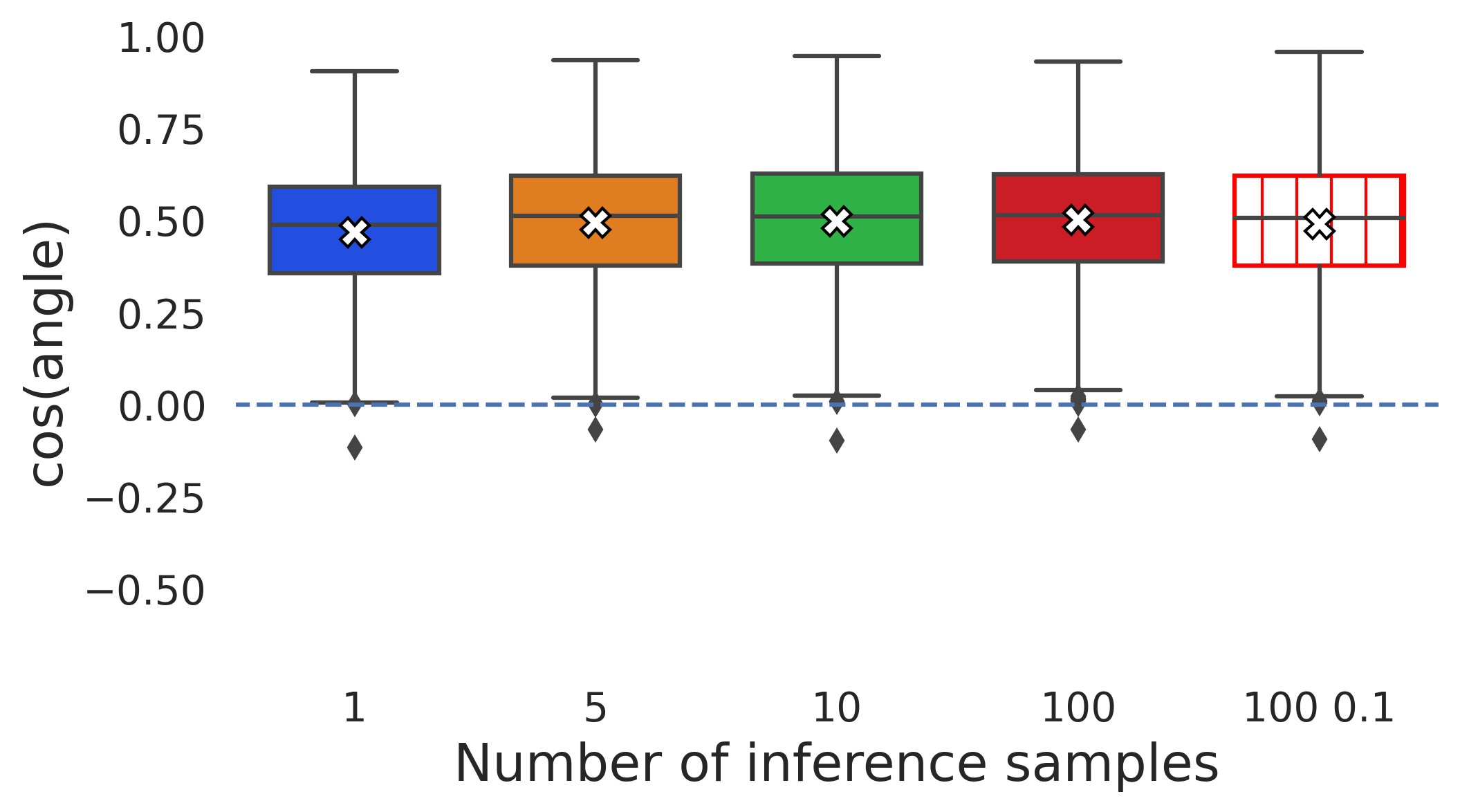} }%
\subfloat[droprate 0.3]{\includegraphics[width=0.3\textwidth]{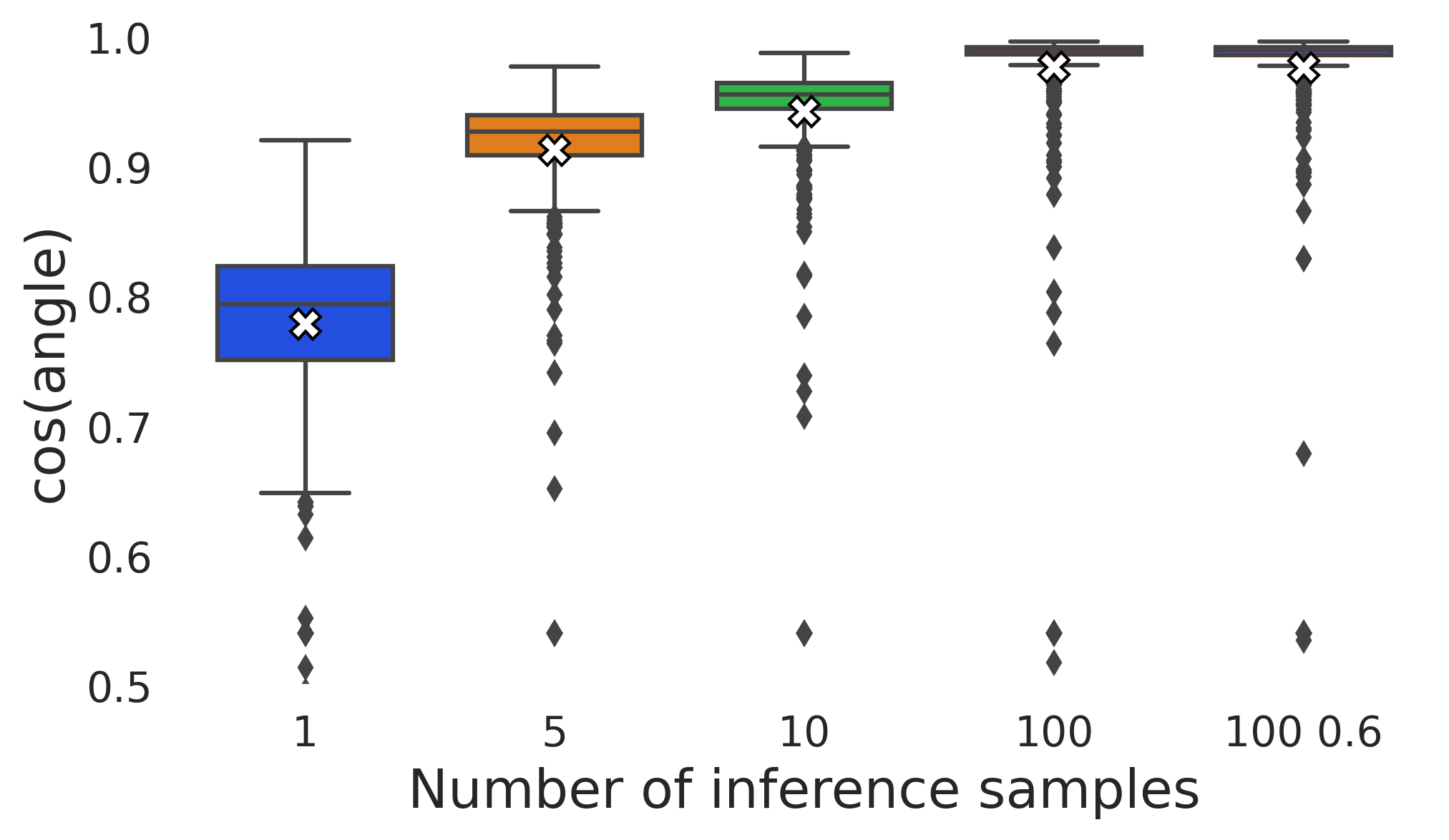} }
\caption{Cosine of the angle for models trained on FashionMNIST ( a), b), d), e) ) and trained on CIFAR10 ( c) and f) ) for different amounts of samples used during inference. Used attack direction $\delta$ was calculated based on 100 sample of FGM under $\ell_2$-norm constraint with $\eta=1.5$ and $0.3$ respectively.}%
\label{fig:app_angle_inference}%
\end{figure}

\newpage

\begin{figure}[h]
\centering
\subfloat[IM]{\includegraphics[width=0.3\textwidth]{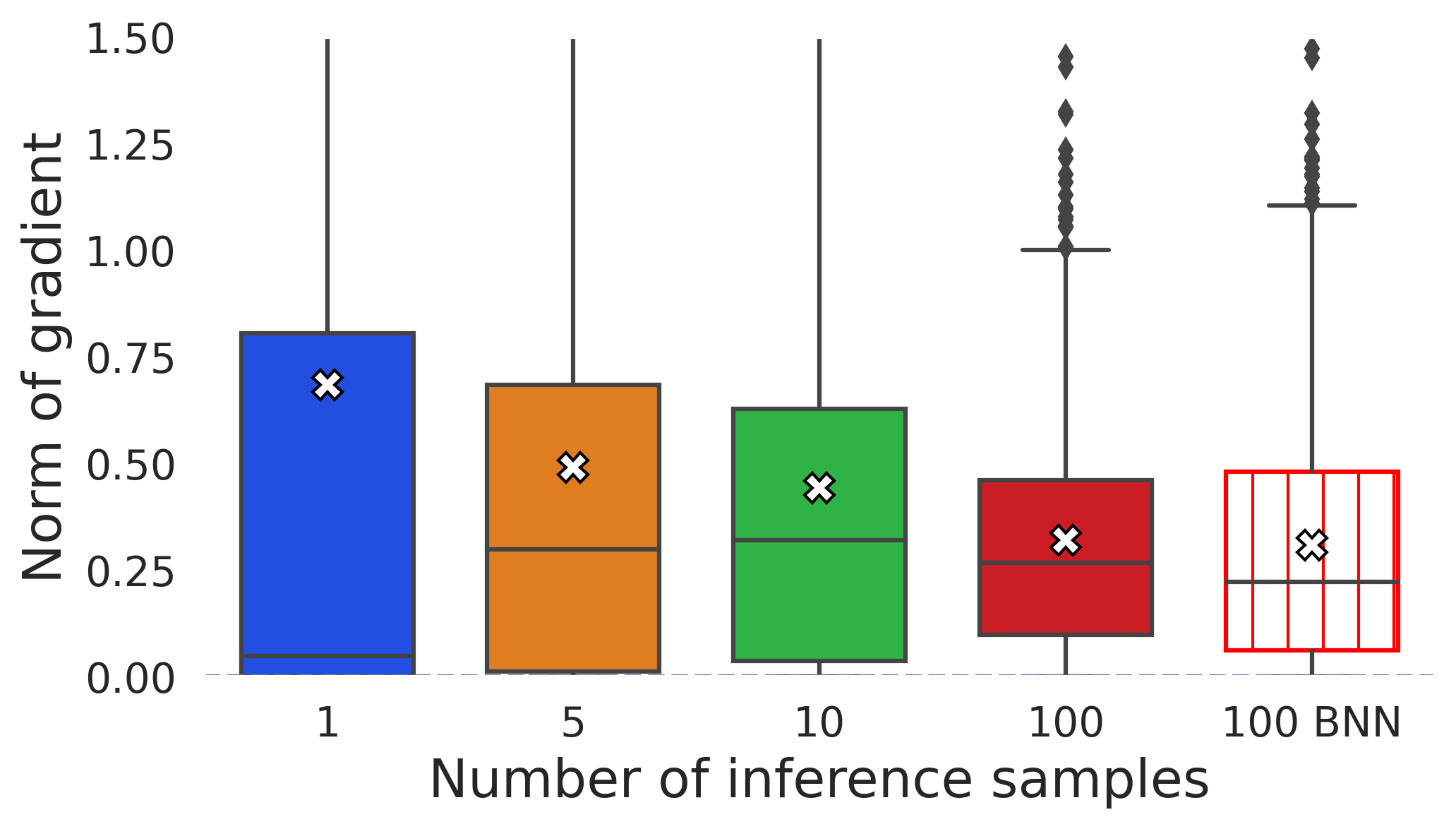} }%
\subfloat[SIN 0.1]{\includegraphics[width=0.3\textwidth]{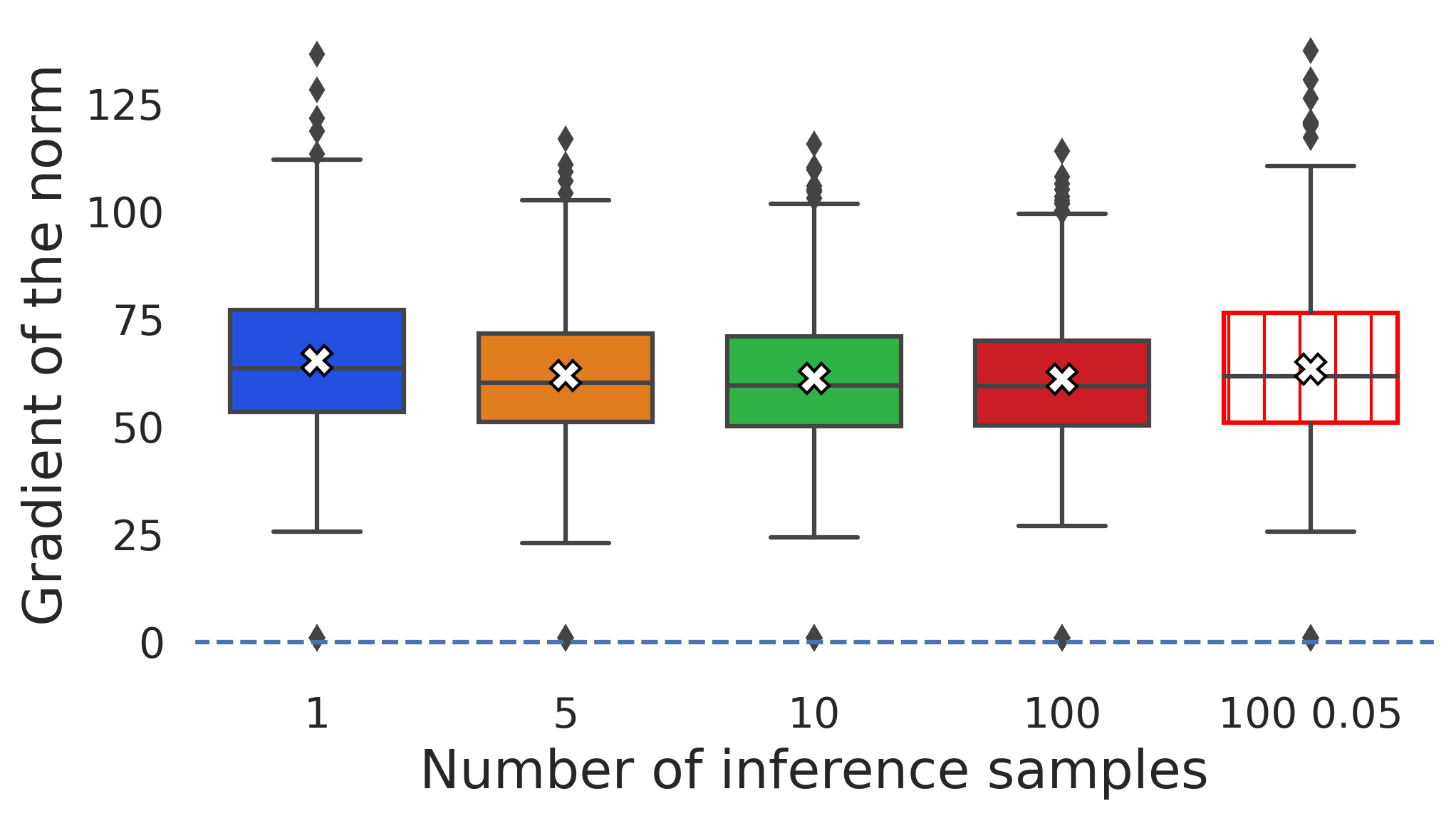} }%
\subfloat[droprate 0.6]{\includegraphics[width=0.3\textwidth]{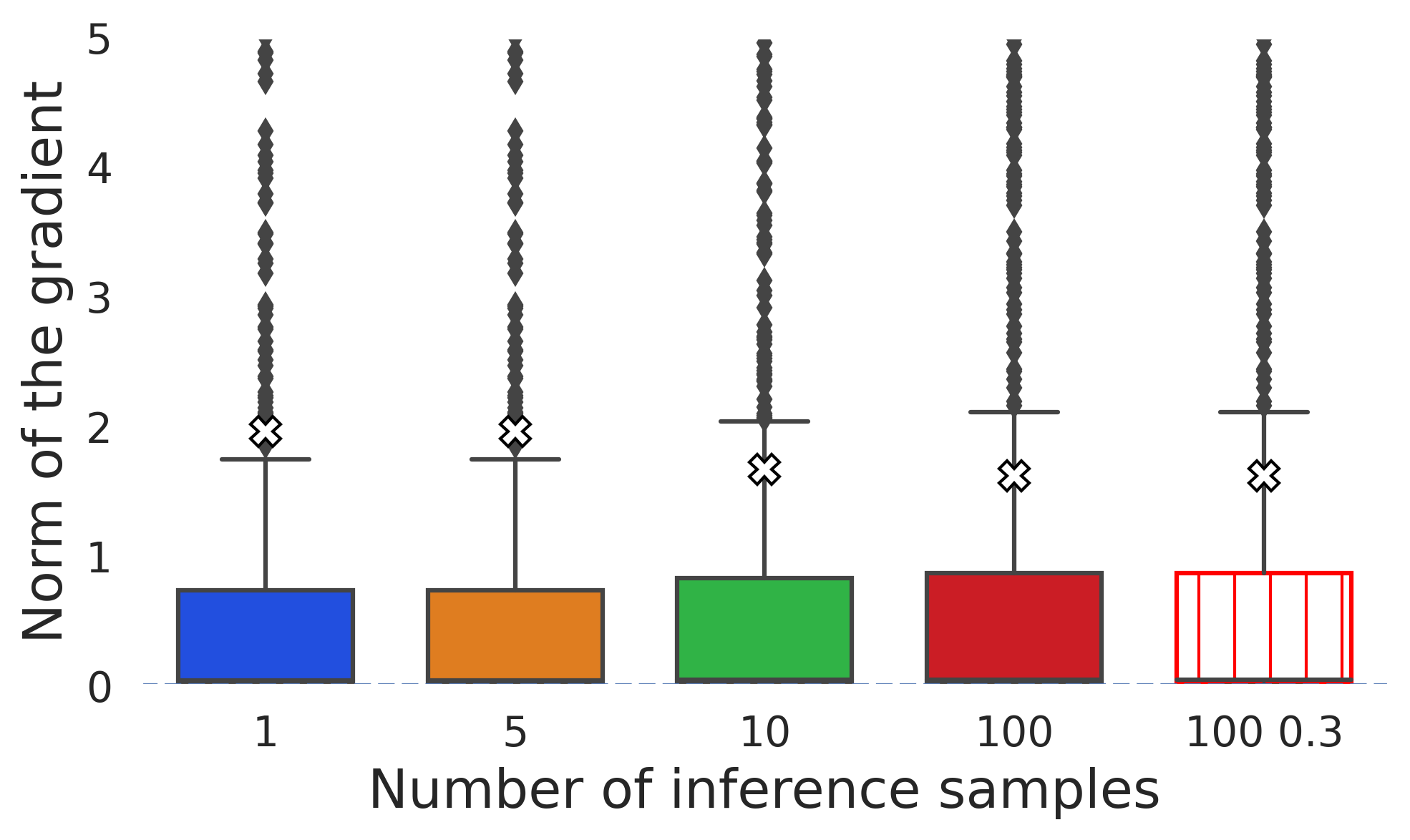} }
\qquad
\subfloat[BNN]{\includegraphics[width=0.3\textwidth]{images/norm_inference_BNN_1.png} }%
\subfloat[SIN 0.05]{\includegraphics[width=0.3\textwidth]{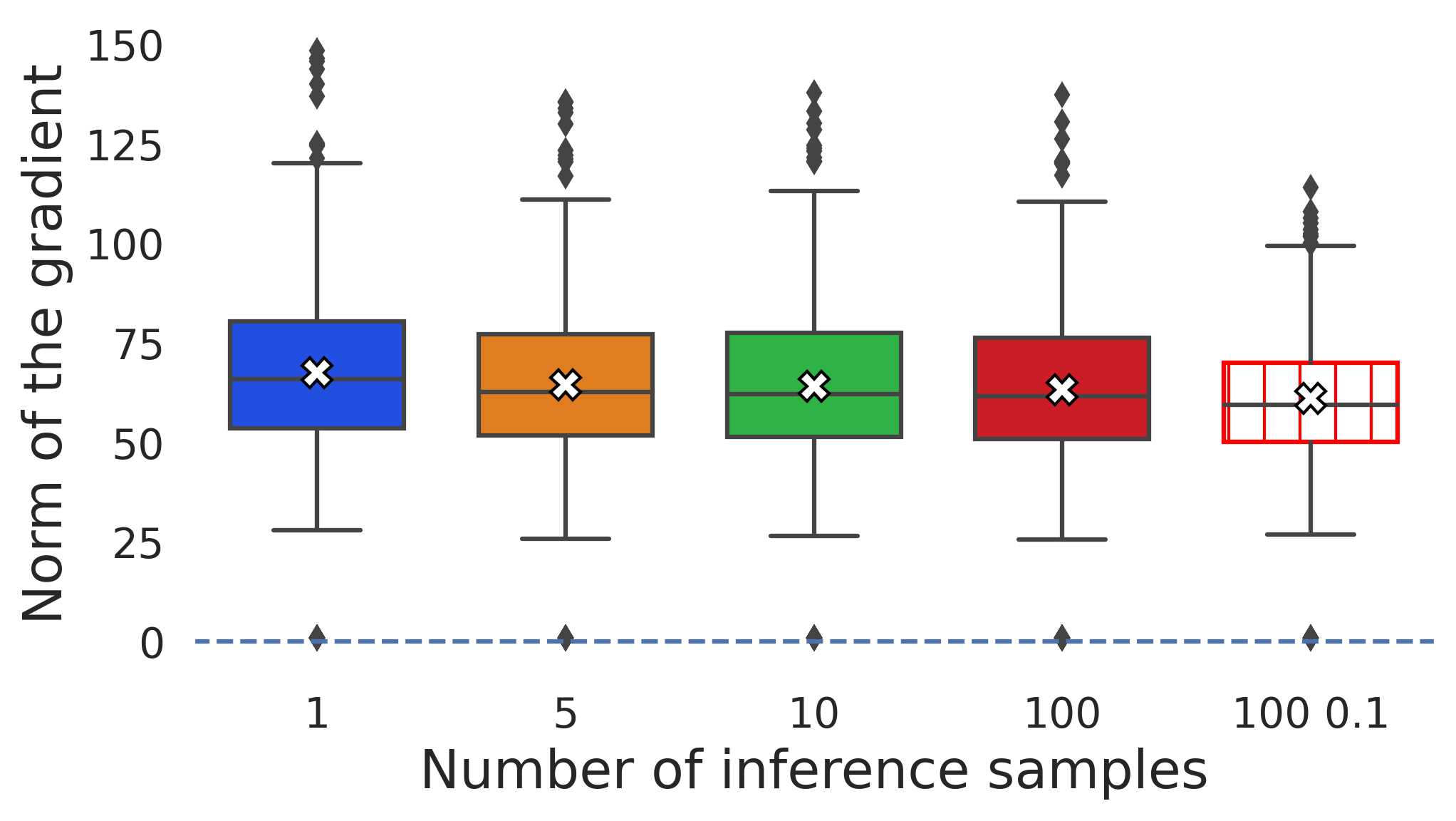} }%
\subfloat[droprate 0.3]{\includegraphics[width=0.3\textwidth]{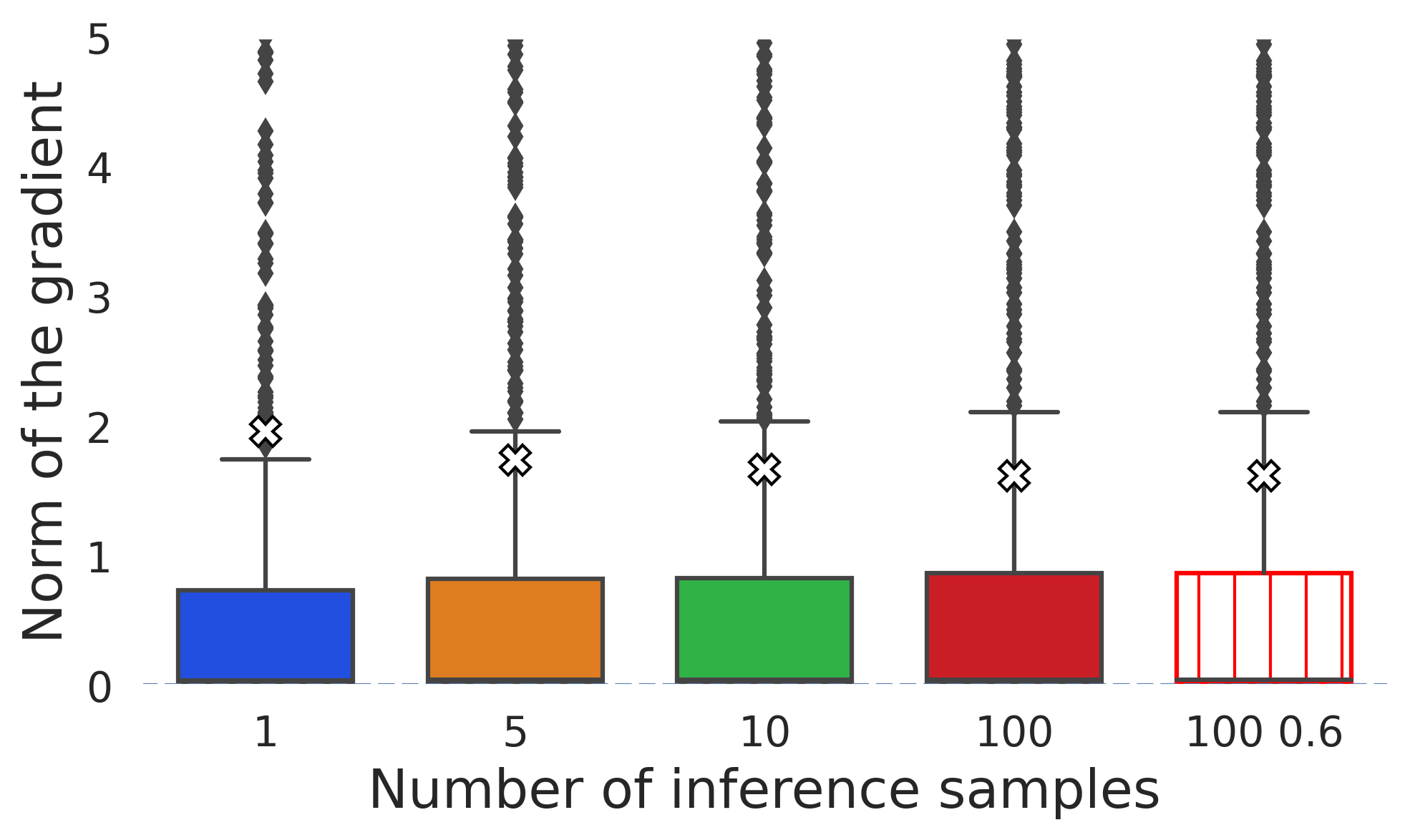} }
\caption{Norm of the gradient for models trained on FashionMNIST ( a), b), d), e) ) and trained on CIFAR10 ( c) and f) ) for different amounts of samples used during inference.}%
\label{fig:app_norm_inference}%
\end{figure}

\begin{figure}[h]
\centering
\subfloat[IM]{\includegraphics[width=0.3\textwidth]{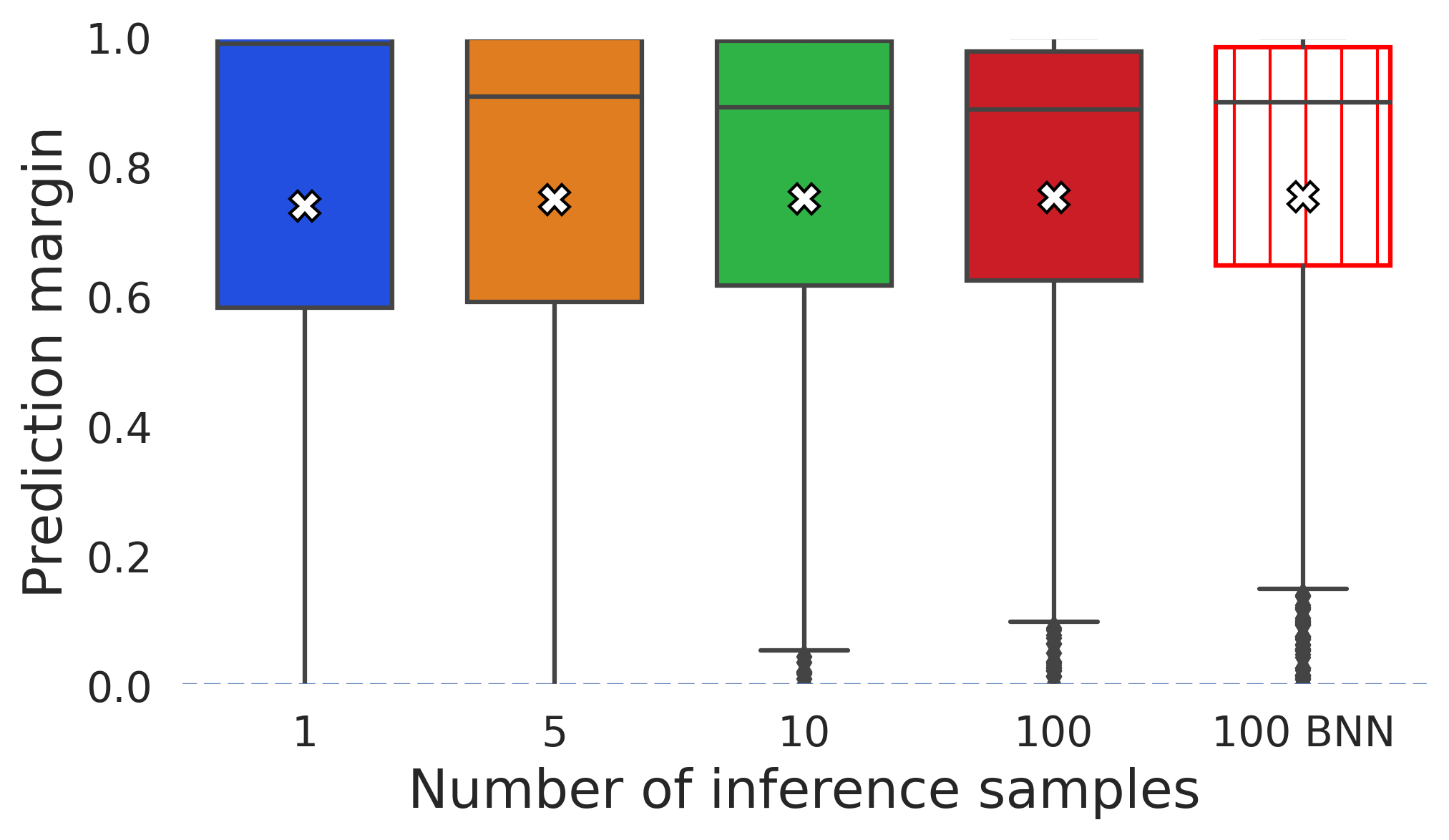} }%
\subfloat[SIN 0.1]{\includegraphics[width=0.3\textwidth]{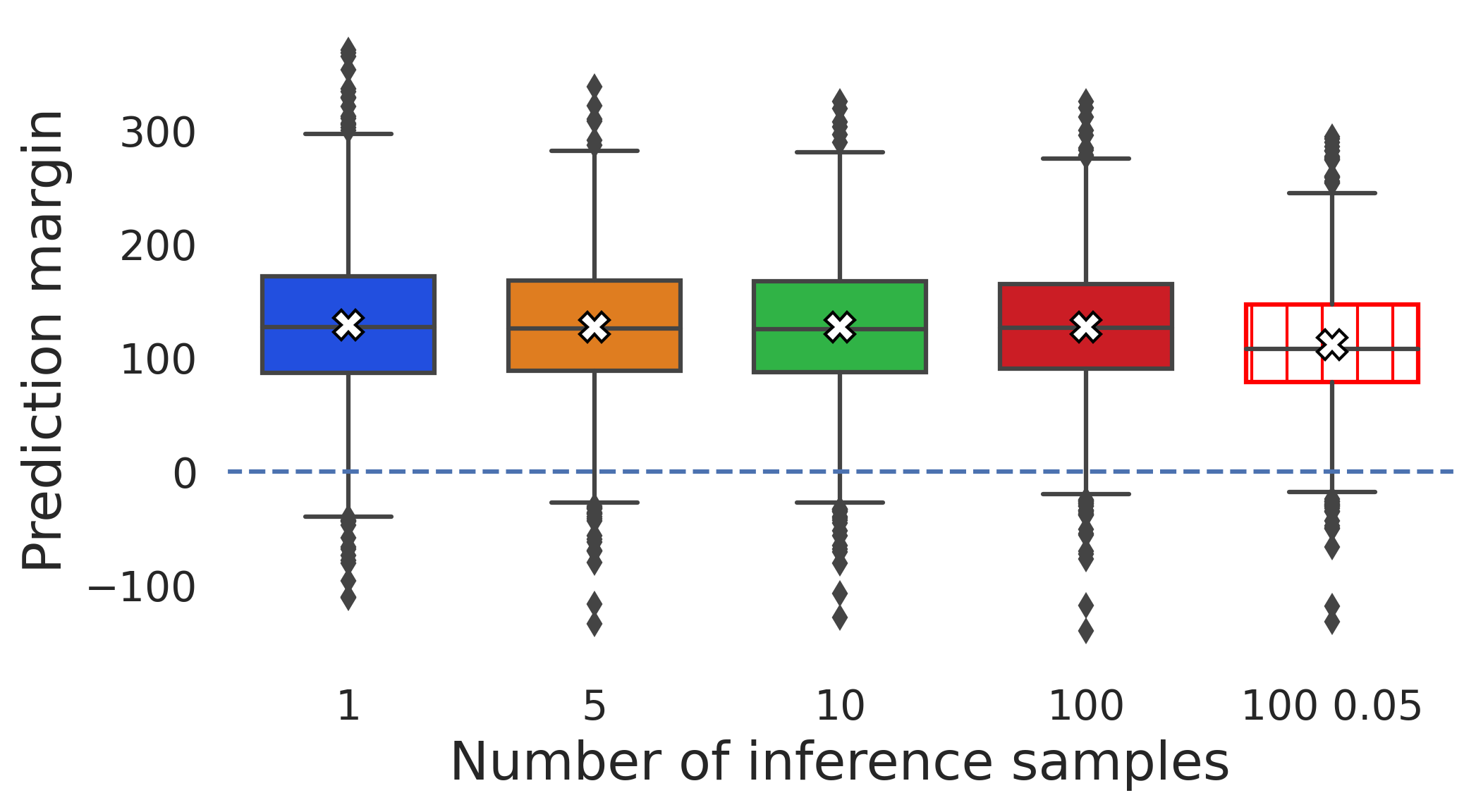} }%
\subfloat[droprate 0.6]{\includegraphics[width=0.3\textwidth]{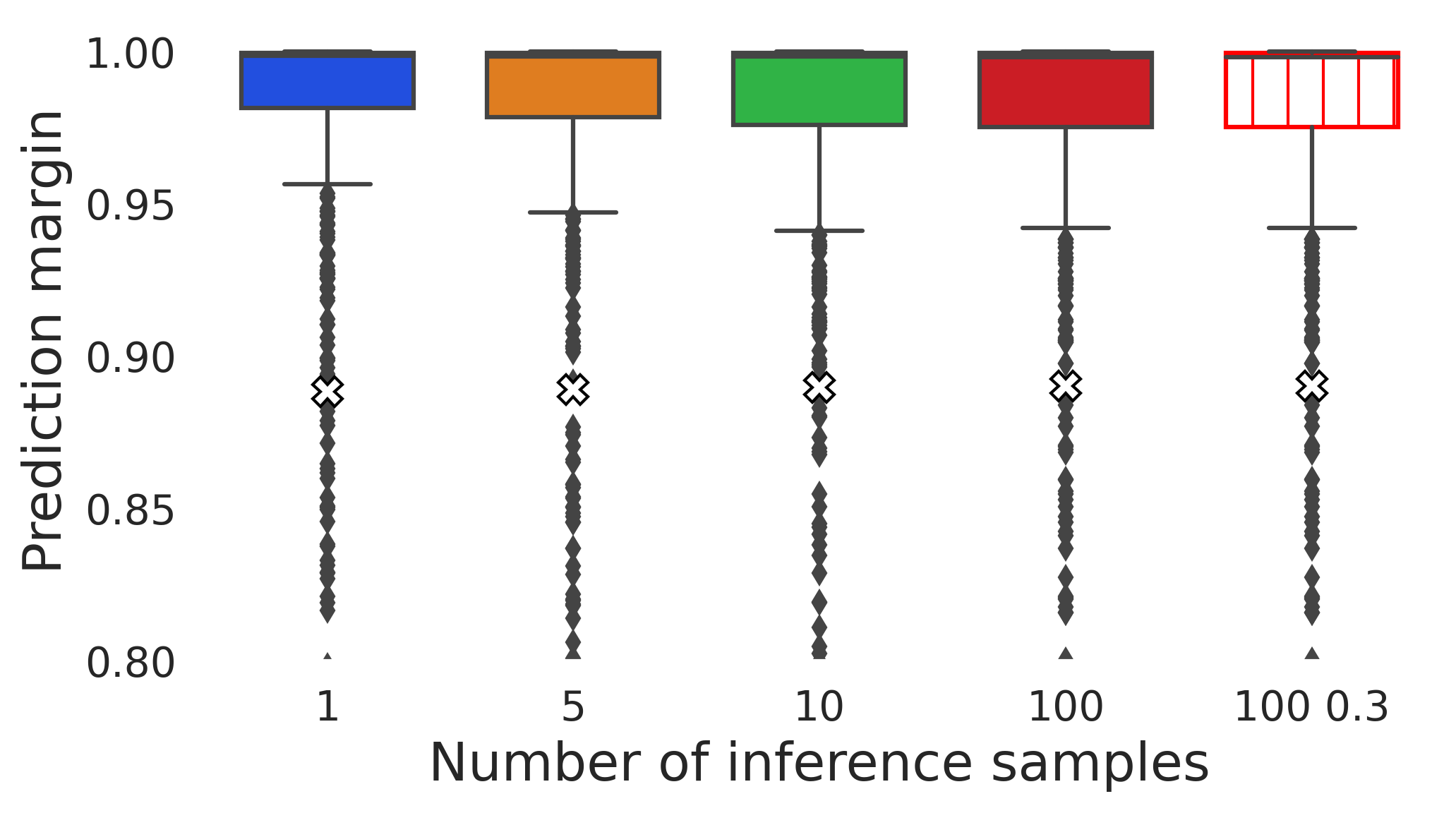} }
\qquad
\subfloat[BNN]{\includegraphics[width=0.3\textwidth]{images/margin_BNN.png} }%
\subfloat[SIN 0.05]{\includegraphics[width=0.3\textwidth]{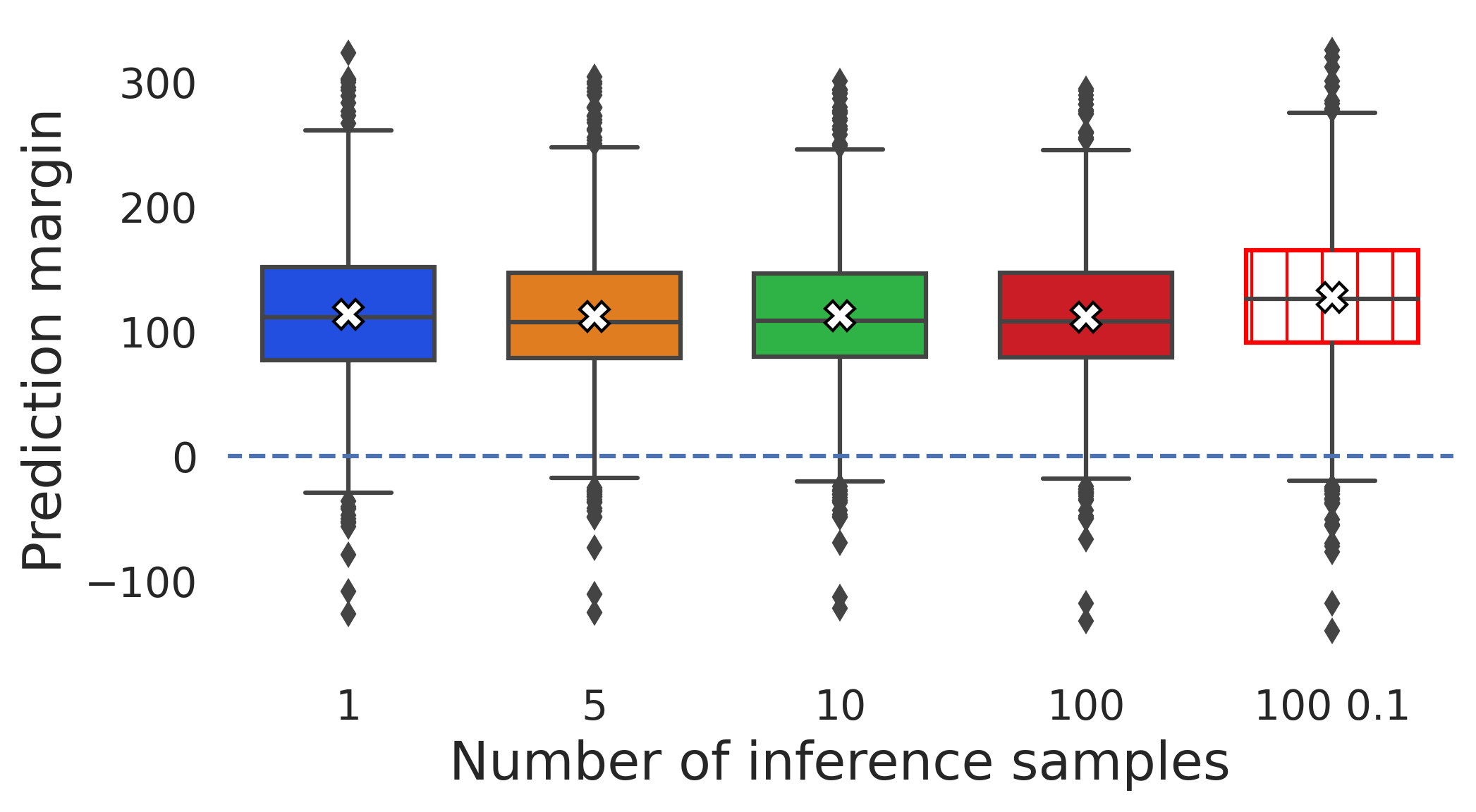} }%
\subfloat[droprate 0.3]{\includegraphics[width=0.3\textwidth]{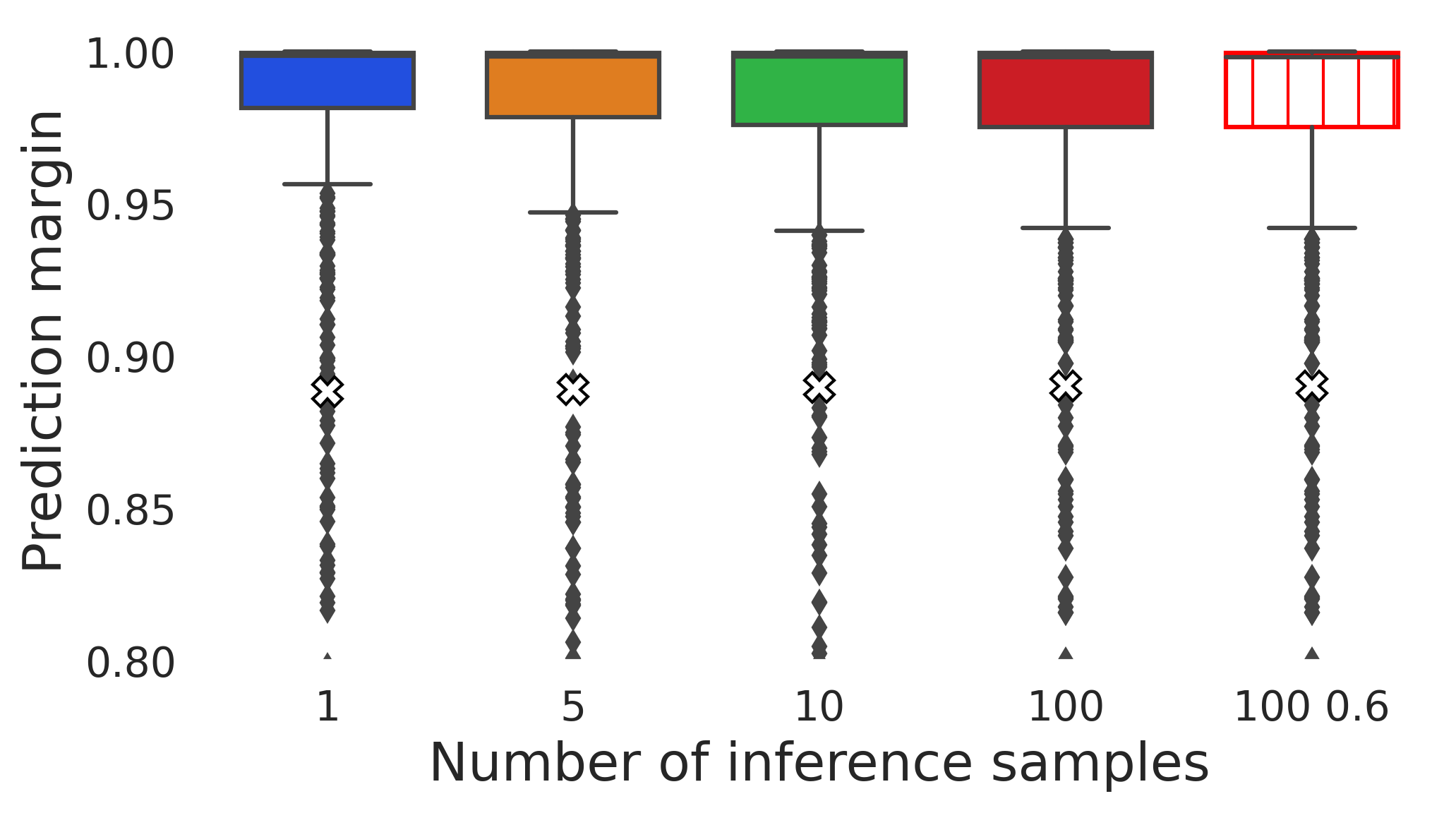} }
\caption{Prediction margin for the benign inputs for models trained on FashionMNIST ( a), b), d), e) ) and trained on CIFAR10 ( c) and f) ) for different amounts of samples used during inference. }%
\label{fig:app_margin_inference}%
\end{figure}

\newpage
\subsection{Experiments for CIFAR100}
\label{App:subsec:cifar100}
For the experiments on CIFAR100 we used yet another method to create stochastic neural networks, namely by applying a Laplace approximation~\citep{mackay92} to an already trained network. This is archived by adapting a Gaussian distribution over the network's parameters such that the mean is given by the maximum a posterior estimate and the covariance are calculated to match the local loss curvature. 
Specifically, we used the GitHub library from~\citet{laplace2021} on top of the adversarial trained wide ResNet70-16 with clean accuracy $69.15$ provided by~\citet{gowal2020uncovering}.
Because of the high amount of parameters in this network we used a last-layer diagonal Gaussian approximation for fitting our posterior distribution from which we sample the $\theta_i$'s for deriving an approximate expected prediction. As in the previous experiments we observe, that the adversarial accuracy decreases with more samples during attack, while the angle between the attack and negative gradient during inference decreases which leads to a cosine increase.

\begin{table}[h]
\label{tab:app_cifar100}
\caption{Adversarial accuracy decrease and $\cos(\alpha_c^{\mathcal{I},\mathcal{A}})$ increase on CIFAR100 with increased amount of samples during attack, where the attack was conducted with $\ell_{\infty}$ norm constraint and perturbation strength 8/225.}
\vskip 0.15in
\begin{center}
\begin{small}
\begin{sc}
\begin{tabular}{r|c|c}
\toprule
\# samples &	Adversarial accuracy &	Average $\cos(\alpha_c^{\mathcal{I},\mathcal{A}}) \pm $ std\\
\hline
1&	48.00&	0.2028 $\pm$  0.112 \\
5&	45.00&	0.2550 $\pm$  0.100 \\
10&	43.90&	0.2680 $\pm$   0.095 \\
\bottomrule
\end{tabular}
\end{sc}
\end{small}
\end{center}
\vskip -0.1in
\end{table}

\end{document}